\documentclass[10pt,twocolumn,letterpaper]{article}

\usepackage{cvpr}
\usepackage[accsupp]{axessibility}
\usepackage{xcolor}
\usepackage{bm}
\usepackage{amsmath}
\usepackage{amsfonts}
\usepackage{amssymb}
\usepackage{mathrsfs}
\usepackage{amsthm}
\usepackage{graphicx}
\usepackage{multirow}
\usepackage{makecell}
\usepackage{tikz}
\usepackage{pgfplots}
\pgfplotsset{compat=1.18}
\usetikzlibrary{fit,positioning}
\definecolor{myPurple}{RGB}{128,0,128}

\newcommand{\vb}{\mathbf{v}}
\newcommand{\wb}{\mathbf{w}}
\newcommand{\xb}{\mathbf{x}}

\newcommand{\zb}{\mathbf{z}}

\newcommand{\Bb}{\mathbf{B}}

\newcommand{\Fb}{\mathbf{F}}
\newcommand{\Gb}{\mathbf{G}}

\newcommand{\Mb}{\mathbf{M}}

\newcommand{\Ob}{\mathbf{O}}

\newcommand{\Xb}{\mathbf{X}}

\newcommand{\cC}{\mathcal{C}}

\newcommand{\cG}{\mathcal{G}}

\newcommand{\cO}{\mathcal{O}}

\newcommand{\cU}{\mathcal{U}}

\newcommand{\cX}{\mathcal{X}}

\newcommand{\norm}[1]{\left\lVert#1\right\rVert}

\newcommand{\argmax}{\mathop{\mathrm{argmax}}}

\newtheorem{theorem}{Theorem}
\newtheorem{definition}{Definition}

\newtheorem*{remark}{Remark}

\allowdisplaybreaks
\makeatletter

\definecolor{cvprblue}{rgb}{0.21,0.49,0.74}
\usepackage[pagebackref,breaklinks,colorlinks,allcolors=cvprblue]{hyperref}

\def\paperID{} % *** Enter the Paper ID here
\def\confName{CVPR}
\def\confYear{2026}

\title{
	Reparameterized Tensor Ring Functional Decomposition for Multi-Dimensional Data Recovery
}

\author{
Yangyang Xu\textsuperscript{1} \quad
Junbo Ke\textsuperscript{1} \quad
You-Wei Wen\textsuperscript{1} \quad
Chao Wang\textsuperscript{2,}\thanks{Corresponding author: \href{mailto:wangc6@sustech.edu.cn}{wangc6@sustech.edu.cn}.} \\
\textsuperscript{1}Key Laboratory of Computing
and Stochastic Mathematics (Ministry of Education),\\ School of Mathematics and Statistics, Hunan Normal University\\
\textsuperscript{2}Department of Statistics and Data Science, Southern University of Science and Technology\\
}

\begin{document}
\maketitle

\begin{abstract}
	Tensor Ring (TR) decomposition is a powerful tool for high-order data modeling, but is inherently restricted to discrete forms defined on fixed meshgrids. In this work, we propose a TR functional decomposition for both meshgrid and non-meshgrid data, where factors are parameterized by Implicit Neural Representations (INRs). However, optimizing this continuous framework to capture fine-scale details is intrinsically difficult. Through a frequency-domain analysis, we demonstrate that the spectral structure of TR factors determines the frequency composition of the reconstructed tensor and limits the high-frequency modeling capacity. To mitigate this, we propose a reparameterized TR functional decomposition, in which each TR factor is a structured combination of a learnable latent tensor and a fixed basis. This reparameterization is theoretically shown to improve the training dynamics of TR factor learning. We further derive a principled initialization scheme for the fixed basis and prove the Lipschitz continuity of our proposed model.
	Extensive experiments on image inpainting, denoising, super-resolution, and point cloud recovery demonstrate that our method achieves consistently superior performance over existing approaches. Code is available at \url{https://github.com/YangyangXu2002/RepTRFD}.
\end{abstract}

\section{Introduction}

Low-rank tensor representations have become a fundamental tool for modeling multi-dimensional data across a wide range of computer vision and signal processing tasks, such as image and video processing~\cite{lu2021transforms,luo2022hlrtf}, remote sensing~\cite{karami2012compression,chen2019hyperspectral}, medical imaging~\cite{lv2021discriminant}, and data compression~\cite{kolda2009tensor,cichocki2013tensor}. By exploiting the underlying low-dimensional structures, classical tensor decompositions such as CANDECOMP/PARAFAC (CP)~\cite{carroll1970analysis,kiers2000towards}, Tucker~\cite{tucker1966some,de2000introduction}, Tensor Train (TT)~\cite{oseledets2011tensor}, and Tensor Ring (TR)~\cite{zhao2016tensor} provide compact representations and facilitate recovery from incomplete or corrupted data.

\begin{figure}[tbp]
	\centering
	\includegraphics[width=\linewidth]{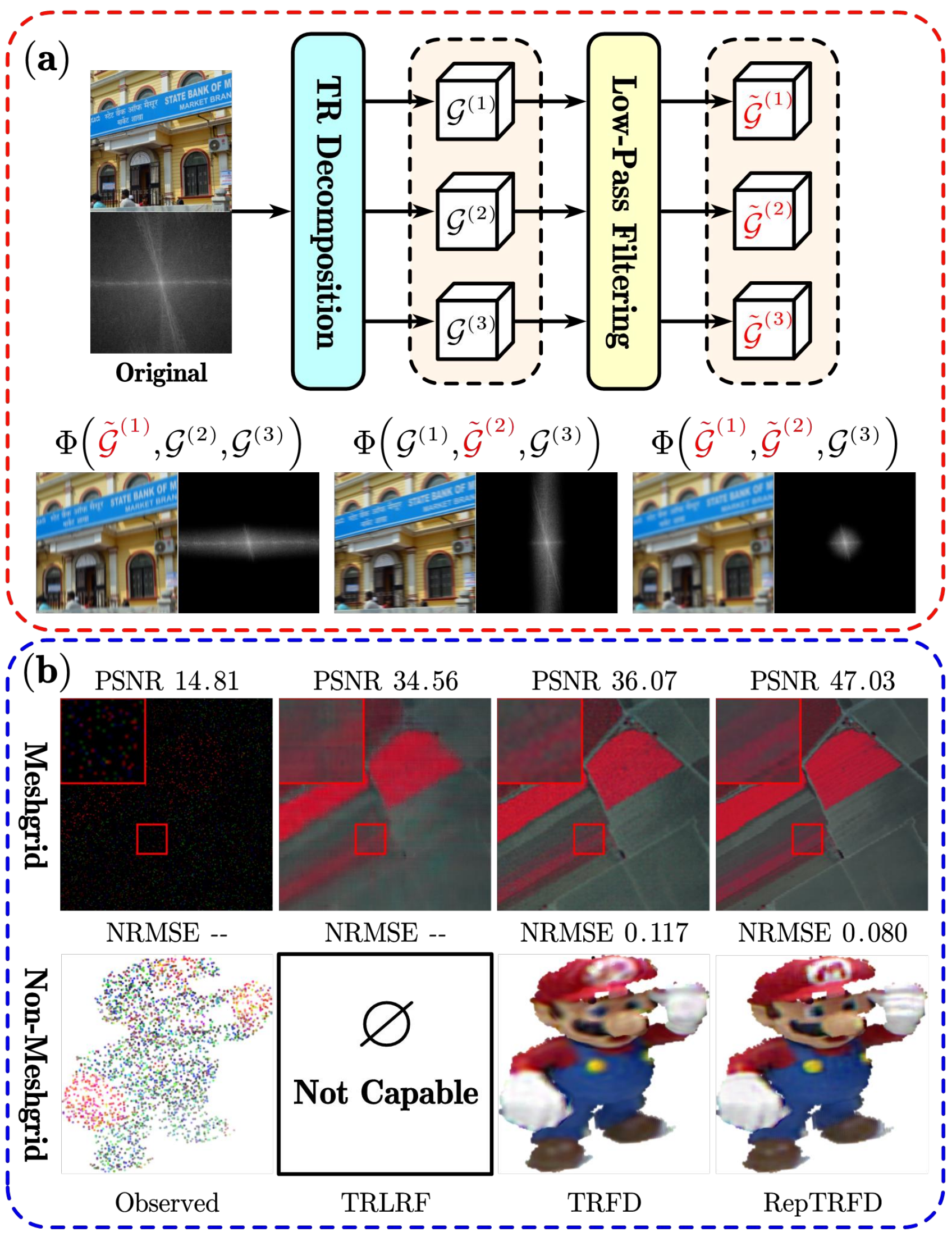}
	\vspace{-1.5em}
	\caption{(a) Frequency analysis of TR factors. Low-pass filtering $\{\cG^{(k)}\}_{k=1}^3$ along mode-2 yields the low-frequency counterparts $\{\tilde{\cG}^{(k)}\}_{k=1}^3$. Reconstructing via the contraction $\Phi(\cdot)$ using these factors shows noticeable attenuation along the corresponding modes. (b) Qualitative comparison of TRLRF~\cite{yuan2019tensor}, TRFD, and RepTRFD for meshgrid and non-meshgrid data recovery. Compared with TRLRF, our methods (TRFD and RepTRFD)  effectively handle non-meshgrid data and demonstrate marked improvements over the baseline.}
	\label{fig:introduction}
	\vspace{-1em}
\end{figure}

Among these, TR decomposition has been widely studied for its ability to efficiently represent high-order tensors with compact structures~\cite{wang2017efficient,yuan2019tensor,qiu2022noisy,liu2025block}. However, traditional TR formulations are inherently discrete and assume that data are defined on fixed grids, which limits their applicability to continuous signals and resolution-independent modeling. To address the limitation in discrete tensor decomposition, a general framework known as low-rank tensor functional representation has been proposed~\cite{luo2023low}. In this paradigm, tensor factors are modeled as Implicit Neural Representations (INRs)~\cite{sitzmann2020implicit} that map continuous coordinates to latent factor values, effectively placing tensor decomposition in a continuous functional space. While this functional framework has been successfully extended to various decomposition formats~\cite{wang2024functional,vemuri2026f,li2025deep} and improved with continuous regularizations~\cite{luo2024revisiting,luo2025neurtv}, extending TR decomposition to this continuous setting remains unexplored. 

Inspired by existing functional tensor methods, we propose a TR Functional Decomposition (TRFD) for multidimensional image recovery. It is worth noting that establishing this decomposition is non-trivial due to the intrinsic properties of TR factors, and direct application of existing functional tensor methods often produces reconstructions dominated by low-frequency components. 
To further understand this limitation, we analyze TR decomposition from a frequency perspective. As illustrated in Fig.~\ref{fig:introduction}, replacing the original TR factors $\{\cG^{(k)}\}_{k=1}^3$ with their low-pass filtered versions $\{\tilde{\cG}^{(k)}\}_{k=1}^3$ yields reconstructions that exhibit noticeable attenuation along the corresponding modes. This finding reveals that the frequency characteristics of TR factors are directly inherited by the reconstructed tensor. 

However, due to the inherent spectral bias of INRs~\cite{rahaman2019spectral}, learning TR factors that capture appropriate high-frequency components remains highly challenging. Distinct from existing approaches, we take a novel perspective by revisiting the factor learning process through the lens of training dynamics.
Specifically, we seek to transform the parameter space of the TR factors, so that the optimization dynamics can more effectively explore high-frequency directions. To achieve this, we propose Reparameterized TRFD (RepTRFD), in which each TR factor is expressed as a structured combination of a learnable latent tensor and a fixed basis. This reparameterization is theoretically shown to enhance the training dynamics, allowing the TR structure to fully exploit both meshgrid and continuous data. As illustrated in Fig.~\ref{fig:introduction}, it significantly improves the reconstruction of high-frequency details. By incorporating this, we offer a novel perspective for advancing the study of tensor functional representations.
In summary, our main contributions are as follows:
\begin{itemize}
	\item We extend the TR decomposition to the continuous domain and provide a frequency-domain analysis of the challenge in learning TR factors that capture high-frequency components.
	\item We propose a reparameterization strategy for TR factors, where each factor is expressed as a structured combination of a learnable tensor and a fixed basis, which facilitates the learning of high-frequency details.
	\item We theoretically show that this reparameterization enhances the training dynamics, and further derive a principled initialization scheme while providing the Lipschitz continuity guarantees. Extensive experiments demonstrate the superiority of our approach.
\end{itemize}

% To overcome these limitations, resent studies represented by LRTFR~\cite{luo2023low} have explored implicit neural representations (INRs) as a powerful alternative for modeling tensor factors. Such methods reformulate tensor decomposition in a continuous functional space, where each tensor entry is modeled as the output of a coordinate-based neural network.
% This allows tensors to be represented beyond discrete meshgrids, enabling continuous and high-resolution reconstruction from arbitrary coordinates.
% Furthermore, regularization terms originally defined on discrete data, such as total variation (TV) or non-local regularizers, can also be naturally extended to the continuous domain.

\section{Related Work}
\noindent
{\bf Implicit Neural Representations. }INRs have recently emerged as a powerful paradigm for representing continuous signals by mapping coordinates to values~\cite{sitzmann2020implicit}.
Subsequent works have sought to enhance the representational power and efficiency of INRs through improved activation functions~\cite{ramasinghe2022beyond,saragadam2023wire,liu2024finer}, encoding schemes~\cite{tancik2020fourier,mildenhall2021nerf}, and training dynamics~\cite{cai2024batch,shi2024improved}. Despite their expressive power, learning global mappings over entire signals still exhibits limited computational efficiency, and the lack of structural priors often limits performance in inverse problems.

\noindent
{\bf Tensor Functional Representations. }Recent works have extended low-rank tensor representations to the continuous domain. Unlike INRs that map coordinates to the entire signal, these methods learn coordinate-to-factor mappings, enabling continuous modeling with lower computational cost and inherent low-rank priors. Representative approaches include the low-rank tensor function representation (LRTFR)~\cite{luo2023low} based on Tucker decomposition, a functional transform-based factorization~\cite{wang2024functional} that improves expressiveness, a parameter-efficient deep rank-one factorization (DRO-TFF)~\cite{li2025deep} that reduces storage and computation overhead, and unified functional extensions of CP and TT decompositions~\cite{vemuri2026f}. Moreover, these functional tensor formulations can be synergistically combined with continuous regularization methods, such as neural total variation (NeurTV)~\cite{luo2025neurtv} and continuous representation-based nonlocal (CRNL)~\cite{luo2024revisiting} priors, to further enhance reconstruction quality and structural consistency.

\noindent
{\bf Reparameterization and Training Dynamics. }Reparameterization techniques have been widely adopted in deep learning to improve training dynamics. Early work by Salimans et al.~\cite{salimans2016weight} first stabilized optimization by decoupling the magnitude and direction of weights through weight normalization. Following this, Ding et al.~\cite{ding2021repvgg} further extended it to structural reparameterization with parallel branches. Mostafa~\cite{mostafa2019parameter} explored another perspective by decoupling weight optimization from sparse connectivity learning.
More recently, Shi et al.~\cite{shi2024improved} factorized network weights into a learnable matrix and a fixed Fourier basis, enhancing learning behavior in INRs.
In addition, several works have explored alternative approaches to improving the training dynamics of INRs and achieved promising results~\cite{cai2024batch,li2025learning,shi2025inductive}. In this work, we take the first step toward analyzing low-rank tensor functional representations through training dynamics, highlighting the potential of reparameterization for more effective factor learning.

\section{Methodology}

\subsection{Notations and Preliminaries}

To enhance readability, we denote scalars by lowercase letters (e.g., $x$), vectors by bold lowercase (e.g., $\xb$), matrices by bold uppercase (e.g., $\Xb$), and tensors by calligraphic letters (e.g., $\cX$). Given a $d$-th order tensor $\cX \in \mathbb{R}^{n_1 \times n_2\times \cdots \times n_d}$, its $\mathbf{v}$-th entry is written as $\cX_{\mathbf{v}}$ where $\mathbf{v}=[v_1,v_2,\dots,v_d]^\mathrm{T}$. We also define the mode-$k$ slice at index $v_k$ as $\cX_{:v_k:}\in \mathbb{R}^{n_1 \times \cdots \times n_{k-1} \times n_{k+1} \times \cdots \times n_d}$. In addition, $\|\cX\|_\infty$  is the infinity norm of $\cX$ by calculating the maximum magnitude of the tensor's entries.  
% For example, we use $\cG_{:v:}$ to denote the lateral slice of a third-order tensor $\cG$  at index $v$. 
We next introduce two fundamental definitions used throughout this paper.

\begin{definition}[Tensor Ring Decomposition~\cite{zhao2016tensor}]
	Consider a $d$-th order tensor $\cX \in \mathbb{R}^{n_1 \times n_2\times \cdots \times n_d}$, TR decomposition represents $\cX$ using third-order cores $\cG^{(k)} \in \mathbb{R}^{r_{k} \times n_k \times r_{k+1}}$ for $k=1,2,\dots,d$, with the circular constraint $r_{d+1}=r_1$.
	The tuple $[r_1, r_2, \dots, r_d]$ is referred to as the TR rank, which controls the representation capacity.
	Each element of $\cX$ is given by
	\begin{equation}
		% $
		\cX_{\vb}
		= \operatorname{trace}\left(
		\cG^{(1)}_{:v_1:}
		\cG^{(2)}_{:v_2:}
		\cdots
		\cG^{(d)}_{:v_d:}
		\right),
		% $
	\end{equation}
	where $\operatorname{trace}(\cdot)$ denotes the matrix trace.
	For simplicity, the decomposition is presented as 
	\begin{equation}
		\cX = \Phi(\cG^{(1)}, \cG^{(2)}, \dots, \cG^{(d)}),
	\end{equation}
	where $\Phi(\cdot)$ denotes the TR contraction operation.
\end{definition}

\begin{definition}[Mode-$k$ Discrete Fourier Transform~\cite{kolda2009tensor,lu2019tensor}]\label{def:dft}
	Given a $d$-th order tensor $\cX \in \mathbb{R}^{n_1 \times n_2 \times \cdots \times n_d}$, its mode-$k$ Discrete Fourier Transform (DFT) is defined as
	\begin{equation}
		\mathscr{F}_k[\cX]=\cX \times_k \Fb_{k},
	\end{equation}
	where $\Fb_{k} \in \mathbb{C}^{n_k \times n_k}$ is the DFT matrix, and $\times_k$ is the $k$-mode product between tensor and matrix. 
\end{definition}

\subsection{Tensor Ring Functional Decomposition}

Traditional TR decomposition assumes that each factor $\cG^{(k)}$ is a discrete tensor defined on a fixed grid. To extend the discrete TR model into a continuous representation, we reinterpret each TR factor as a functional tensor parameterized by INRs. Since the TR factors are inherently linked through cyclic contractions, we introduce a shared frequency embedding to enhance their consistency across modes.

\begin{figure}[htbp]
	\centering
	\includegraphics[width=\linewidth]{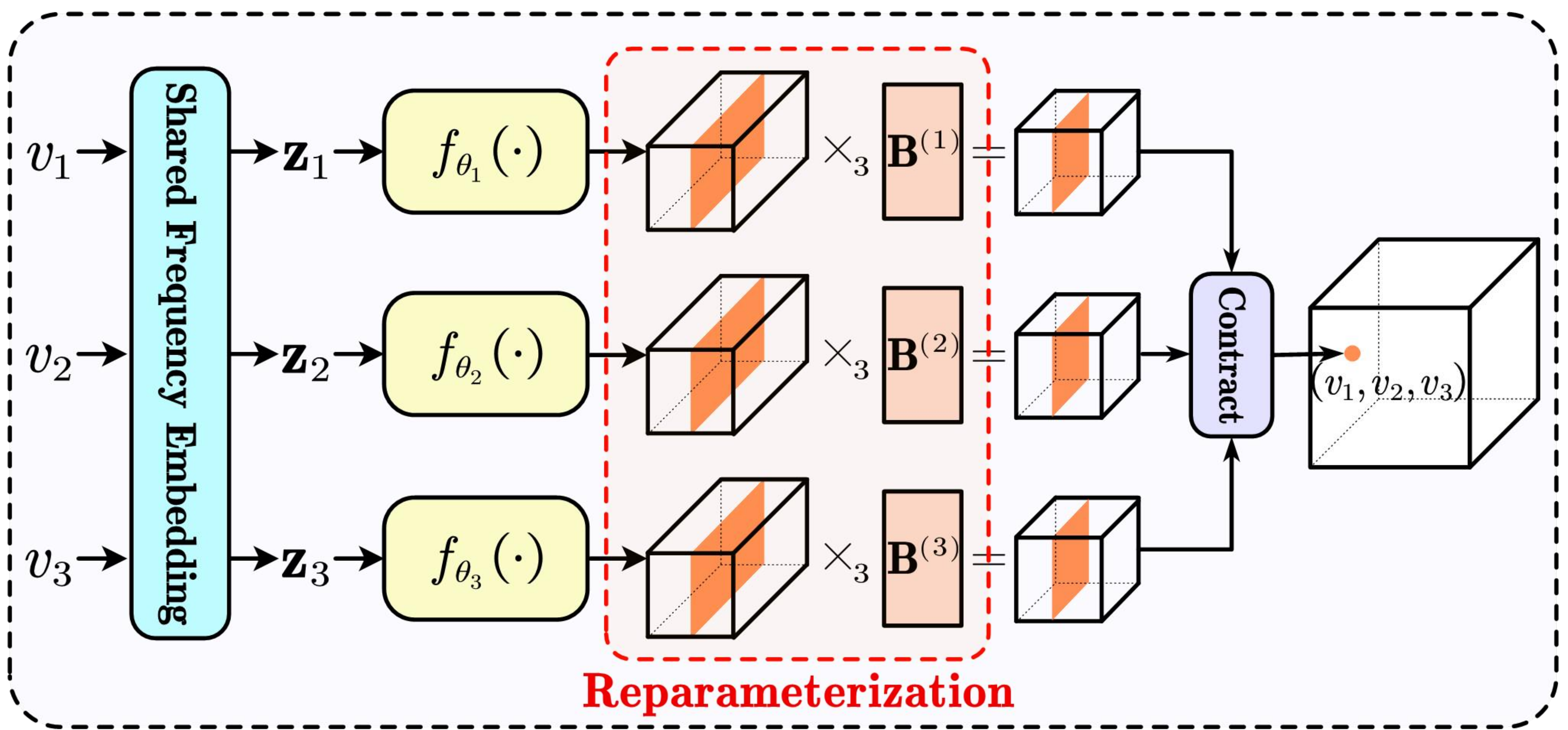}
	\vspace{-1.3em}
	\caption{Illustration of proposed TRFD and RepTRFD. TRFD directly learns a continuous mapping from coordinates to TR factors. In contrast, RepTRFD reparameterizes each TR factor as a structured combination of a learnable latent tensor and a fixed basis, which improves the training dynamics and facilitates more effective learning of high-frequency components.}
	\label{fig:RepTR_shared}
	\vspace{-1em}
\end{figure}

As illustrated in Fig.~\ref{fig:RepTR_shared}, given a coordinate vector $\vb = [v_1, v_2, \dots, v_d]^\mathrm{T}$, the shared embedding for the $k$-th dimension is implemented by a single sinusoidal layer:
\begin{equation}\label{eq:freq_vec}
	\zb_k = \sin(\omega_0 (\wb v_k + \mathbf{b}))\in \mathbb{R}^{h}, \quad k=1,\dots,d,
\end{equation}
where $\omega_0$ is a frequency scaling parameter, $\wb\in \mathbb{R}^{h}$ and $\mathbf{b}\in\mathbb{R}^h$ are learnable weight and bias, and $h$ denotes the hidden dimension of the sinusoidal layer. To generate each TR factor, we employ a separate branch network for each mode. Each branch is a multi-layer perceptron (MLP) that maps the corresponding embedding $\zb_k$ to a vector of length $r_k r_{k+1}$, which is then reshaped into the desired TR slice:
\begin{equation}\label{eq:factor}
	\cG^{(k)}_{:v_k:} = f_{\theta_k}(\zb_k) = g_{\theta_k}(v_k)\in \mathbb{R}^{r_k \times r_{k+1}}, v_k\in \{1,\dots, n_k\},
\end{equation}
where $f_{\theta_k}(\cdot)$ denotes the MLP associated with the $k$-th tensor mode, producing the mode-specific factor slice from the shared embedding, and $g_{\theta_k}(\cdot)$ is the function composed with MLP and shared frequency embedding. This approach naturally extends the TR decomposition to a continuous representation, which we refer to as the Tensor Ring Functional Decomposition (TRFD).  
% Nevertheless, representing TR factors solely through the above framework falls short of achieving optimal performance in practice. This motivates us to further explore the potential of TR structures.

\subsection{Reparameterization for Factor Learning}

% Although the TR functional decomposition offers a flexible approach to modeling multi-dimensional data, its practical performance is often limited. To understand this limitation, it is necessary to examine the learning dynamics of its underlying components. Since each TR factor is produced by an INR, the entire model inevitably inherits the well-known spectral bias of such networks, which tend to capture low-frequency components first while struggling with high-frequency details. Consequently, the ability of individual TR factors to represent high-frequency components may be restricted, which in turn could influence the spectral characteristics of the reconstructed tensor along the corresponding modes. This observation motivates a careful study of how the frequency content of each factor propagates through the cyclic contractions of the TR decomposition.
% Although the TR functional decomposition provides a flexible framework for modeling multi-dimensional data, its practical performance is often limited. Each TR factor is generated by an INR with inherent spectral bias, which favors low-frequency components but struggles to capture high-frequency details.  As a result, the high-frequency representation of individual factors may be restricted, affecting the spectral characteristics of the reconstructed tensor. This motivates a closer study of how the frequency content of each factor propagates through the cyclic contractions in the TR decomposition.

\noindent
% \subsubsection{
	{\bf Frequency Analysis of TR Factors. }While the TR functional decomposition provides a flexible and compact framework for representing high-dimensional data, its practical effectiveness is often constrained by the inherent spectral bias of the INRs used to parameterize the factors. INRs tend to prioritize low-frequency components while underrepresenting high-frequency details, leading to limited expressiveness within the individual TR factors. Crucially, this spectral limitation does not remain confined to the factor level, because it propagates through the cyclic contractions of the TR structure and influences the frequency characteristics of the reconstructed tensor $\cX$. To better understand this phenomenon, we formally analyze how the frequency content of the TR factors propagates through the reconstruction process. The following theorem presents the corresponding characterization.
	
	% \tb{While the TR functional decomposition offers a flexible framework, its practical performance is often constrained by the inherent spectral bias of the INRs used to parameterize the factors. INRs naturally favor low-frequency components and struggle to capture high-frequency details. This bias restricts the high-frequency representation within the individual TR factors $\{\cG^{(k)}\}$.}
	
	% \tb{Crucially, this limitation propagates to the final output. Due to the cyclic contractions in the TR decomposition, the spectral characteristics of each factor directly influence the corresponding dimension of the reconstructed tensor $\cX$. We formalize this insight by analyzing how frequency content propagates through the reconstruction process. }
	
	% Although the TR functional decomposition provides a flexible framework for modeling multi-dimensional data, its practical performance is often limited. Each TR factor is generated by an INR with inherent spectral bias, which favors low-frequency components but struggles to capture high-frequency details.  As a result, the high-frequency representation of individual factors may be restricted, affecting the spectral characteristics of the reconstructed tensor. This motivates a closer study of how the frequency content of each factor propagates through the cyclic contractions in the TR decomposition.
	% To formalize this insight, we analyze how the frequency content of TR factors influences the reconstructed tensor. The following theorem characterizes the propagation of spectral components through the TR reconstruction process.
	\begin{theorem}
		\label{thm:spectral_preservation}
		Let $\cX = \Phi(\cG^{(1)}, \dots, \cG^{(d)})$ be a TR decomposition. 
		Suppose that the mode-2 frequency components of $\cG^{(k)}$ beyond a threshold $\Omega_k$ are negligible for all $k$, i.e.,
		\begin{equation*}
			\|\big(\mathscr{F}_2[\cG^{(k)}]\big)_{:\omega_k:}\|_\infty \le \epsilon, 
			\ \forall |\omega_k| > \Omega_k, k= 1, 2, \dots, d,
		\end{equation*}
		where $\epsilon > 0$ is a small constant. Then the reconstructed tensor $\cX$ also exhibits attenuated high-frequency content along mode $k$:
		\begin{equation*}
			\|\big(\mathscr{F}_k[\cX]\big)_{:\omega_k:}\|_\infty 
			\le c_{k} \, \epsilon, \ \forall |\omega_k| > \Omega_k,  k= 1, 2,\dots, d,
		\end{equation*}
		where $c_{k}$ is a constant depending only on the magnitudes of the remaining cores $\{\cG^{(j)}\}_{j\ne k}$.
	\end{theorem}
	
	All proofs of the theoretical results presented in this section are provided in the supplementary material. Theorem~\ref{thm:spectral_preservation} indicates that the spectral content of each TR factor directly affects the corresponding dimension of the reconstructed tensor. Consequently, achieving accurate reconstruction of high-frequency tensors is contingent on the TR factors themselves containing sufficiently high-frequency components. However, due to the inherent spectral bias of standard INRs, learning such factors is challenging. This motivates the design of strategies that facilitate the acquisition of high-frequency information.
	
	% The proof of Theorem~\ref{thm:spectral_preservation} is provided in supplementary material. This theorem indicates that the spectral content of each TR factor directly affects the corresponding dimension of the reconstructed tensor. In particular, if a TR factor is dominated by low-frequency components, the reconstruction will inevitably suffer from high-frequency loss along the corresponding tensor dimension. Consequently, achieving accurate reconstruction of high-frequency tensors requires the TR factors themselves to contain appropriate high-frequency components. Due to the inherent spectral bias of standard INRs, however, learning such factors may be challenging, motivating the design of strategies that facilitate the acquisition of high-frequency information.
	
	% \tb{The proof of Theorem~\ref{thm:spectral_preservation} is provided in the supplementary material. This theorem establishes that high-frequency attenuation in a TR factor $\cG^{(k)}$ (specifically in its mode-2) directly results in a corresponding high-frequency loss in the $k$-th mode of the reconstructed tensor $\cX$. Consequently, achieving accurate reconstruction of high-frequency tensors is contingent on the TR factors themselves containing sufficient high-frequency components. The inherent spectral bias of standard INRs renders this challenging, thus motivating the design of strategies that facilitate the acquisition of high-frequency information.}
	
	\noindent
	% \subsubsection{
		{\bf Reparameterized Tensor Ring. } 
		% \tb{Unlike existing approaches, such a formulation improves the conditioning of the optimization}
		% To address this challenge, we take a novel perspective: rather than solely focusing on network architectures or loss design, we turn to the  
		% % directly address the optimization difficulty in
		% factor learning by leveraging the structural decomposition of TR factors. 
		To address this challenge, 
		% we aim to represent the TR factors from the perspective of parameter space and training dynamics. This provides a novel viewpoint that has not been explored in existing studies on tensor functional representations. Our goal is to 
		we aim to transform the parameter space of the TR factors to improve the training dynamics, thereby facilitating more efficient learning of the TR factors. Specifically, by introducing a fixed basis $\Bb^{(k)}\in \mathbb{R}^{r_{k+1}\times R_{k+1}}$, we reparameterize each TR factor as a combination of a learnable latent tensor and this basis, which is expressed as
		\begin{equation}
			\cG^{(k)} = \cC^{(k)} \times_3 \Bb^{(k)}, \quad k = 1, \ldots, d. 
		\end{equation}
		This decomposition separates the learnable component from the fixed basis, allowing each TR factor to be expressed as a structured combination of latent elements. Such a formulation improves the conditioning of the optimization and facilitates more effective learning of high-frequency components, as shown in the following theorem.
		% This decomposition separates the learnable component from the fixed basis, allowing each TR factor to be expressed as a structured combination of latent elements. Such a formulation improves the conditioning of the optimization and facilitates more effective learning of high-frequency components. The following theorem formalizes how this reparameterization enhances the training dynamics by amplifying gradients associated with high-frequency modes. 
		\begin{theorem}\label{thm:rep}
			Consider a TR factor $\cG$ reparameterized as $\cG = \cC \times_3 \Bb$, where $\cC \in \mathbb{R}^{r_1 \times n \times R_2}$ is a trainable tensor and $\Bb \in \mathbb{R}^{r_2 \times R_2}$ is a fixed
			% for some 
			basis. Let $\mathsf{L}(\omega)$ be the loss associated with frequency $\omega$. For any $\omega_{\rm high} > \omega_{\rm low} > 0$, given any $\epsilon \geq 0$, and fixed indices $p, q$ with $1 \le p \le r_1$ and $1 \le q \le n$, 
			% we have 
			there exists a matrix $\Bb$ such that for all $s=1,\dots,R_2$,
			\begin{equation*}
				\resizebox{\linewidth}{!}{$\displaystyle\left|
					\tfrac{\partial \mathsf{L}(\omega_{\rm high})}{\partial \cC_{p q s}}
					\Big/
					\tfrac{\partial \mathsf{L}(\omega_{\rm low})}{\partial \cC_{p q s}}
					\right|\ge
					\max_{j=1,\dots,r_2}
					\left|
					\tfrac{\partial \mathsf{L}(\omega_{\rm high}))}{\partial \cG_{p q j}}
					\Big/
					\tfrac{\partial \mathsf{L}(\omega_{\rm low})}{\partial \cG_{p q j}}
					\right|
					- \epsilon.$}
			\end{equation*}
		\end{theorem} 
		Theorem~\ref{thm:rep} demonstrates that our reparameterization amplifies the gradient response to high-frequency components relative to low-frequency ones. This implies that the optimization becomes more responsive to fine-scale variations, effectively improving the convergence behavior when learning TR factors with rich spectral content. To realize these benefits in practice, we need to specify how to determine the latent tensor $\cC^{(k)}$ and fixed basis $\Bb^{(k)}$.

		% For each mode $k$, the latent tensor $\cC^{(k)}$ is generated slice by slice through a neural mapping. Specifically, for 
		Similar to TRFD, we represent the latent tensor $\cC^{(k)}$ as generated slice by slice through a neural mapping. Specifically, given a coordinate $v_k$, the corresponding slice $\cC^{(k)}_{:v_k:}$ is produced by applying a shared frequency embedding and an  MLP as:
		% to the coordinate embedding $\zb_j$ defined in Eq.~\eqref{eq:freq_vec}:
		\begin{equation}
			\cC^{(k)}_{:v_k:} = g_{\theta_k}(v_k) \in \mathbb{R}^{r_k \times R_{k+1}}, v_k\in\{1,\dots, n_k\}. 
		\end{equation}
		The full latent tensor $\cC^{(k)}\in \mathbb{R}^{r_k\times n_k \times R_{k+1}}$ is formed by stacking all slices along the $k$-th mode, where $R_{k+1}=\beta r_{k+1}$ with $\beta\ge 1$. For brevity, the overall mapping from the coordinate indices $\vb_k=[1,\dots,n_k]\in \mathbb{R}^{n_k}$ to the latent tensor is denoted as
		\begin{equation}
			\cC^{(k)} = g_{\phi_k}(\vb_k) \in \mathbb{R}^{r_k \times n_k \times R_{k+1}},
		\end{equation}
		emphasizing that the latent tensor is generated through a learnable neural mapping.

		% On the other hand, the choice of the fixed basis $\Bb^{(k)}$ is also critical for ensuring stable training.
		While Theorem~\ref{thm:rep} guarantees the existence of a basis that can amplify high-frequency gradients, an arbitrary choice may lead to poor scaling or variance inconsistency during training. To this end, we adopt a Xavier-style initialization~\cite{glorot2010understanding} for the entries of $\Bb^{(k)}$, as formalized in the following theorem.

		\begin{theorem}\label{thm:basis_init}
			% Let $\Bb^{(k)} \in \mathbb{R}^{r_{k+1} \times R_{k+1}}$ be the fixed basis used in the reparameterized TR factor. 
			Suppose the entries of $\Bb^{(k)}$ in the reparameterized TR factor  are sampled independently from a uniform distribution as 
			% , the entries should be drawn from
			\begin{equation*}
				\Bb^{(k)}_{ij} \sim \mathcal{U}\left(-\sqrt{\tfrac{6}{r_{k+1}+R_{k+1}}}, \ \sqrt{\tfrac{6}{r_{k+1}+R_{k+1}}}\right), \forall i,j,k, 
			\end{equation*}
			then the variances in both the forward and backward passes are preserved.  
		\end{theorem}
		
		% \begin{figure}[htbp]
			% \centering
			% \renewcommand{\arraystretch}{0.5}
			% \setlength\tabcolsep{1pt}
			% \resizebox{\linewidth}{!}{
				% \begin{tabular}{cc}
					% \includegraphics[height=40mm]{figures/PSNR_comparison.png} & \includegraphics[height=40mm]{figures/Spectral_Centroid_comparison.png}
					% \end{tabular}}
			% \caption{.}
			% \label{fig:rep_vs_worep}
			% \end{figure}
		
		To summarize the above developments,  we refer to the proposed framework as the Reparameterized Tensor Ring Functional Decomposition (RepTRFD), with its overall architecture illustrated in Fig.~\ref{fig:RepTR_shared}.
		While the previous analyses focus on gradient dynamics and initialization, it is also essential to understand the overall stability of the reparameterized formulation as a functional mapping. To ensure that the reparameterized formulation remains well-behaved and avoids excessive sensitivity to input perturbations, we establish its Lipschitz continuity property in the following theorem.

		\begin{theorem}\label{thm:lipschitz}
			Let $g_{\phi}(\vb): \mathbb{R}^d \to \mathbb{R}$ be defined as 
			$$g_{\phi}(\vb) = \Phi(g_{\phi_1}(v_1)\times_3 \Bb^{(1)}, \dots, g_{\phi_d}(v_d)\times_3 \Bb^{(d)}),$$
			where $\vb = [v_1, \dots, v_d]^{\mathrm{T}} \in \mathbb{R}^d$ denotes the coordinate vector, and each $\Bb^{(k)} \in \mathbb{R}^{r_{k+1} \times R_{k+1}}$ is a fixed basis.
			Assume for each mode $k = 1, \dots, d$, the following conditions hold:
			\begin{itemize}
				\item $g_{\phi_k}(\cdot)$ is an $L_k$-layer MLP with activation function $\sigma(\cdot)$ that is $\kappa$-Lipschitz continuous;
				\item the spectral norm of each weight matrix in $g_{\phi_k}(\cdot)$ is bounded by $\eta$;
				\item the output is bounded: $\sup_{v_k} \|g_{\phi_k}(v_k)\|_F \le C_k < \infty$.
			\end{itemize}
			Then, $g_{\phi}(\vb)$ is globally Lipschitz continuous, i.e.,
			\begin{equation*}
				|g_{\phi}(\vb) - g_{\phi}(\vb')| \le \delta \|\vb - \vb'\|_2,
			\end{equation*}
			where the global Lipschitz constant $\delta = \sqrt{\sum_{k=1}^d \delta_k^2}$, and $\delta_k = (\kappa \eta)^{L_k} \cdot \|\mathbf{B}_k\|_2 \cdot \left( \prod_{j \ne k} C_j \right)$.
		\end{theorem}
		
		% \tb{In summary, the reparameterization enhances the training dynamics, provides a practical basis for initialization for variance consistency, and ensures that the overall mapping is globally Lipschitz continuous. These properties jointly enable stable and effective learning of continuous TR factors with rich spectral content.
			% }
		% \subsection{Multi-Dimensional Data Recovery Models}
		
		% In this section,
		Given the RepTRFD representation $g_{\phi}(\{\vb_k\}_{k=1}^d) = \Phi(g_{\phi_1}(\vb_1)\times_3 \Bb^{(1)}, \dots, g_{\phi_d}(\vb_d)\times_3 \Bb^{(d)})$, we introduce a general formulation on the loss designs used for various low-level vision and 3D recovery tasks. Each task is formulated as an optimization problem consisting of a data fidelity term $\mathsf{L}_{\text{data}}$ and an optional regularization term $\mathsf{L}_{\text{reg}}$, expressed as:
		\begin{equation}
			\min_{\phi} \mathsf{L}_{\text{data}}(g_{\phi}(\{\vb_k\}_{k=1}^d);\cO) + \mathsf{L}_{\text{reg}}(g_{\phi}(\{\vb_k\}_{k=1}^d)),
		\end{equation}
		where $\vb$ denotes the coordinate input, 
		$\phi$ includes all the learnable parameters, and $\cO$ denotes the given observations. Depending on the specific recovery task, the regularization term can be omitted or designed to impose appropriate priors on the spatial or spectral structure~\cite{aggarwal2016hyperspectral}.  For more details, please refer to the supplementary materials. 

		\section{Experimental Results}
		
		In this section, we evaluate the proposed method on four representative tasks, including inpainting, denoising, super-resolution, and point cloud recovery.
		Training is performed using the Adam optimizer~\cite{kingma2014adam}, and all experiments are run on a workstation equipped with an Intel Core i9-12900K CPU, 62 GB RAM, and an NVIDIA RTX 4090 GPU.
		
		\subsection{Parameter Setting}
		
		The main parameters of our model include the TR rank $[r_1,\dots, r_d]$, the latent tensor expansion factor $\beta$, and the sinusoidal frequency $\omega_0$. For all modes $k=1,\dots,d$, we set $r_k \in \{16,20\}$ and $R_k=\beta r_k$ with $\beta \in \{3,5,10\}$. The frequency parameter $\omega_0 $ is set to 90, 120, or 240, and each branch has one or two layers with 256 hidden units. The model is optimized with a fixed learning rate of $3\times10^{-4}$. Input coordinates are normalized to the range $[-1,1]$.
		% To further enhance local smoothness and structural consistency, total variation (TV) and spatial-spectral total variation (SSTV) regularizations~\cite{aggarwal2016hyperspectral} are incorporated. 
		Detailed parameter settings for each task are provided in the supplementary material.
		% The main parameters of our model include the TR rank, network depth, and the frequency parameter $\omega_0$ of the sinusoidal activation.
		% For the denoising experiments, the TR rank is set to $[16,16,16]$, while for all other tasks it is fixed at $[20,20,20]$.
		% The frequency parameter $\omega_0$ is set to 90 for color images and videos, and 120 for all other datasets. Each dimension-specific branch network contains one or two linear layers depending on the task.
		% Across all experiments, the learning rate is fixed at $3\times10^{-4}$, and the number of hidden units is 256. The regularization parameters are tuned for each task as follows.
		% For inpainting, we set $(\gamma_1, \gamma_2) = (5\times10^{-5}, 5\times10^{-4})$ for color images and videos, and $(\gamma_1, \gamma_2) = (5\times10^{-6}, 5\times10^{-5})$ for MSI and HSI data.
		% For denoising, $(\gamma_1, \gamma_2)$ are both set to $1\times10^{-4}$ for MSI and $1\times10^{-5}$ for HSI.
		% For super-resolution, the parameter is fixed as $\gamma_1 = 5\times10^{-5}$.
		
		\subsection{Comparisons with State-of-the-Arts}
		
		\noindent
		{\bf Image and Video Inpainting Results. }We first evaluate our method on the inpainting task under various sampling ratios (SRs) using color images, multispectral images (MSIs), hyperspectral images (HSIs), and videos.  
		Datasets are taken from USC-SIPI\footnote{\url{https://sipi.usc.edu/database/}}, CAVE\footnote{\url{https://cave.cs.columbia.edu/repository/Multispectral}}, Remote Sensing Scenes\footnote{\url{https://www.ehu.eus/ccwintco/index.php/Hyperspectral_Remote_Sensing_Scenes}}, and YUV sequences\footnote{\url{http://trace.eas.asu.edu/yuv/}}.  
		All images are resized or cropped for consistent input sizes, and 100 consecutive frames are used for video evaluation. We compare our approach with representative tensor completion models such as TRLRF~\cite{yuan2019tensor} and FCTN~\cite{zheng2021fully}, as well as recent neural tensor representations including HLRTF~\cite{luo2022hlrtf}, LRTFR~\cite{luo2023low}, DRO-TFF~\cite{li2025deep}, and NeurTV~\cite{luo2025neurtv}. All methods are tested under identical sampling masks, and performance is assessed using peak signal-to-noise ratio (PSNR) and structural similarity index (SSIM) \cite{wang2004image}.
		
		\begin{figure*}[tb]
			\renewcommand{\arraystretch}{0.5}
			\setlength\tabcolsep{0.5pt}
			\centering
			\resizebox{\linewidth}{!}{
				\begin{tabular}{ccccccccc}
					\scriptsize{PSNR/SSIM} & \scriptsize{21.50/0.529} & \scriptsize{22.08/0.600} & \scriptsize{26.77/0.811} & \scriptsize{27.21/0.822} & \scriptsize{27.64/0.853} &  \scriptsize{28.21/0.905} & \scriptsize{\textbf{30.45}/\textbf{0.932}} & \scriptsize{\textbf{Inf/1.000}} \\
					\includegraphics[width=18.5mm]{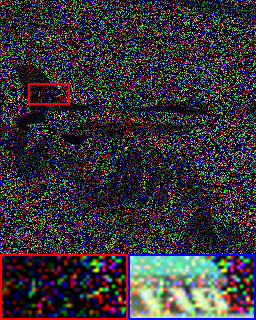} & \includegraphics[width=18.5mm]{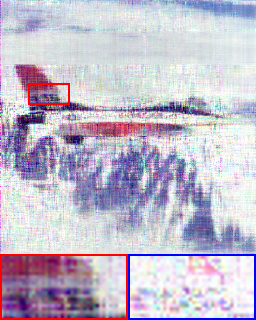} &
					\includegraphics[width=18.5mm]{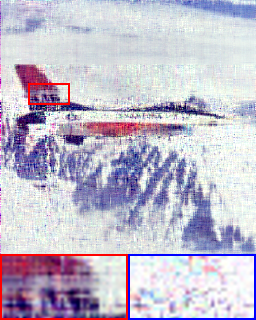} &
					\includegraphics[width=18.5mm]{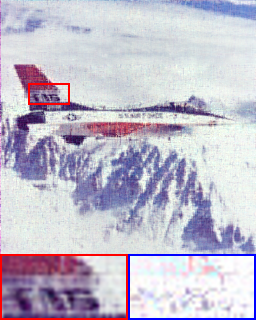} &
					\includegraphics[width=18.5mm]{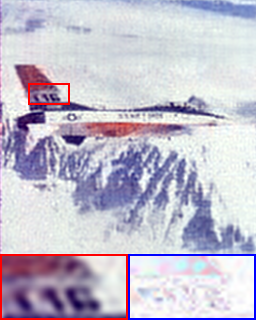} &
					\includegraphics[width=18.5mm]{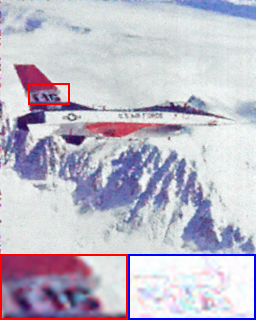} &
					\includegraphics[width=18.5mm]{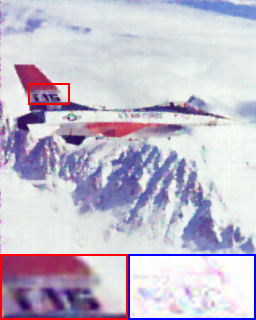} &
					\includegraphics[width=18.5mm]{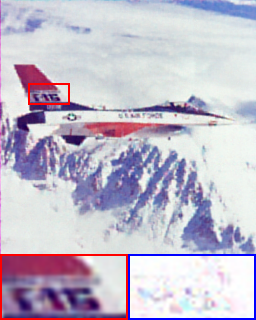} &
					\includegraphics[width=18.5mm]{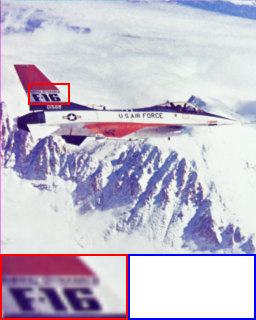} \\
					\scriptsize{PSNR/SSIM} & \scriptsize{26.76/0.812} & \scriptsize{27.74/0.837} & \scriptsize{29.64/0.914} & \scriptsize{29.62/0.893} & \scriptsize{30.66/0.937} &  \scriptsize{30.05/0.902} & \scriptsize{\textbf{32.55}/\textbf{0.957}} & \scriptsize{\textbf{Inf/1.000}} \\
					\includegraphics[width=18.5mm]{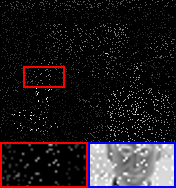} & \includegraphics[width=18.5mm]{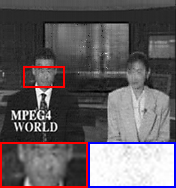} &
					\includegraphics[width=18.5mm]{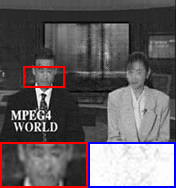} &
					\includegraphics[width=18.5mm]{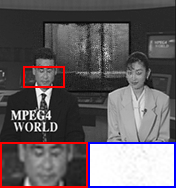} &
					\includegraphics[width=18.5mm]{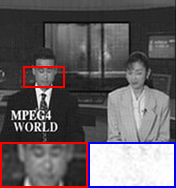} &
					\includegraphics[width=18.5mm]{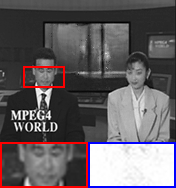} &
					\includegraphics[width=18.5mm]{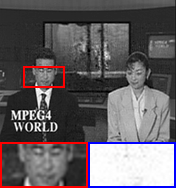} &
					\includegraphics[width=18.5mm]{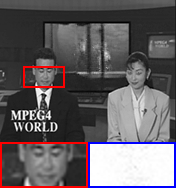} &
					\includegraphics[width=18.5mm]{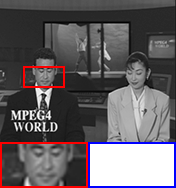} \\
					\scriptsize{Observed} & \scriptsize TRLRF & \scriptsize FCTN & \scriptsize HLRTF & \scriptsize LRTFR & \scriptsize DRO-TFF & \scriptsize NeurTV & \scriptsize Ours & \scriptsize{Ground truth}
			\end{tabular}}
			\vspace{-0.7em}
			\caption{Visual inpainting results on color image \textit{Airplane} ($\text{SR}=0.2$) and video \textit{News} ($\text{SR}=0.1$).}
			\label{fig:inpainting}
			\vspace{-1em}
		\end{figure*}
		
		\begin{table}[tbp]
			\centering
			\caption{Quantitative inpainting results on all datasets.}
			\vspace{-1em}
			\small
			\setlength{\tabcolsep}{3pt}
			\resizebox{\linewidth}{!}{
				\begin{tabular}{clccccccc}
					\toprule
					\multirow{2}{*}{Data} & \multirow{2}{*}{Method} & \multicolumn{2}{c}{$\text{SR}=0.1$} & \multicolumn{2}{c}{$\text{SR}=0.2$} & \multicolumn{2}{c}{$\text{SR}=0.3$} & \multirow{2}{*}{Time (s)} \\
					\cmidrule(lr){3-4} \cmidrule(lr){5-6} \cmidrule(lr){7-8}
					& & PSNR & SSIM & PSNR & SSIM & PSNR & SSIM &  \\
					\midrule
					\multirow{7}{*}{\makecell[c]{\textit{Airplane}\\
							\textit{House}\\
							\textit{Pepper}\\
							\textit{Sailboat}\\
							$(256\times256\times3)$}} & 
					TRLRF & 17.52 & 0.326 & 21.57 & 0.546 & 24.85 & 0.698 & 5.12 \\
					& FCTN & 17.38 & 0.359 & 21.23 & 0.566 & 24.67 & 0.710 & 5.44 \\
					& HLRTF & 20.89 & 0.544 & 25.27 & 0.754 & 28.02 & 0.841 & 3.43 \\
					& LRTFR & 21.38 & 0.587 & 25.86 & 0.770 & 28.74 & 0.856 & 6.03 \\
					& DRO-TFF & 23.22 & 0.734 & 27.52 & 0.848 & 30.04 & 0.899 & 13.25 \\
					& NeurTV & \underline{24.16} & \underline{0.797} & \underline{27.81} & \underline{0.884} & \underline{30.28} & \underline{0.924} & 31.57 \\ 
					& Ours & \textbf{25.70} & \textbf{0.841} & \textbf{29.37} & \textbf{0.915} & \textbf{32.01} & \textbf{0.944} & 13.11 \\
					\midrule
					% \multicolumn{9}{c}{\normalsize MSI, HSI, and video datasets} \\
					\midrule
					\multirow{2}{*}{Data} & \multirow{2}{*}{Method} & \multicolumn{2}{c}{$\text{SR}=0.05$} & \multicolumn{2}{c}{$\text{SR}=0.1$} & \multicolumn{2}{c}{$\text{SR}=0.15$} & \multirow{2}{*}{Time (s)} \\
					\cmidrule(lr){3-4} \cmidrule(lr){5-6} \cmidrule(lr){7-8}
					& & PSNR & SSIM & PSNR & SSIM & PSNR & SSIM &  \\
					\midrule
					\multirow{7}{*}{\makecell[c]{\textit{Toys}\\ \textit{Flowers}\\$(256\times 256\times 31)$}} & 
					TRLRF & 28.95 & 0.748 & 33.40 & 0.869 & 36.17 & 0.922 & 20.16 \\
					& FCTN & 31.71 & 0.838 & 36.77 & 0.929 & 39.33 & 0.956 & 15.34 \\
					& HLRTF & 34.89 & 0.937 & 40.86 & 0.980 & 43.48 & 0.988 & 4.12 \\
					& LRTFR & 36.38 & 0.952 & 40.71 & 0.980 & 42.48 & 0.984 & 7.04 \\
					& DRO-TFF & \underline{38.45} & \underline{0.982} & \underline{42.28} & \underline{0.992} & \underline{45.00} & \underline{0.994} & 14.57 \\
					& NeurTV & 37.49 & 0.964 & 41.66 & 0.984 & 43.72 & 0.988 & 33.56 \\ 
					& Ours & \textbf{39.34} & \textbf{0.984} & \textbf{44.66} & \textbf{0.993} & \textbf{47.74} & \textbf{0.995} & 14.49 \\
					\hline
					\multirow{7}{*}{\makecell[c]{\textit{Washington DC}\\$(256\times 256\times 191)$\\ \textit{Botswana}\\$(256\times 256\times 145)$}} & 
					TRLRF & 29.64 & 0.776 & 32.24 & 0.899 & 34.72 & 0.938 & 98.45 \\
					& FCTN & 32.42 & 0.905 & 34.94 & 0.942 & 36.30 & 0.956 & 57.55 \\
					& HLRTF & 37.36 & 0.970 & 40.14 & 0.984 & 42.09 & 0.988 & 18.36 \\
					& LRTFR & 36.62 & 0.962 & 39.11 & 0.976 & 42.56 & 0.986 & 17.63 \\
					& DRO-TFF & \underline{38.59} & \underline{0.972} & 41.23 & 0.984 & 42.11 & 0.986 & 56.58 \\
					& NeurTV & 37.79 & \underline{0.972} & \underline{41.99} & \underline{0.987} & \underline{43.33} & \underline{0.989} & 45.51 \\ 
					& Ours & \textbf{40.75} & \textbf{0.984} & \textbf{44.22} & \textbf{0.990} & \textbf{45.83} & \textbf{0.992} & 56.45 \\
					\hline
					\multirow{7}{*}{\makecell[c]{\textit{News}\\ \textit{Carphone} \\$(144\times176\times100)$}} & 
					TRLRF & 24.59 & 0.716 & 26.58 & 0.790 & 27.95 & 0.837 & 21.30 \\
					& FCTN & 25.63 & 0.752 & 27.41 & 0.815 & 28.73 & 0.854 & 16.88 \\
					& HLRTF & 26.16 & 0.787 & 28.88 & 0.865 & 30.45 & 0.898 & 4.70 \\
					& LRTFR & 26.48 & 0.798 & 28.70 & 0.865 & 29.95 & 0.889 & 6.80 \\
					& DRO-TFF & \underline{28.79} & \underline{0.892} & \underline{30.13} & \underline{0.915} & \underline{31.48} & \underline{0.932} & 11.83 \\
					& NeurTV & 27.48 & 0.818 & 29.44 & 0.875 & 30.55 & 0.907 & 16.01 \\ 
					& Ours & \textbf{29.77} & \textbf{0.917} & \textbf{31.64} & \textbf{0.939} & \textbf{32.85} & \textbf{0.951} & 15.59\\
					\bottomrule
				\end{tabular}
			}
			\label{tab:inpainting}
			\vspace{-1.5em}
		\end{table}
		
		Fig.~\ref{fig:inpainting} shows the inpainting results on the color image \textit{Airplane} and the video \textit{News} at different SRs. For clearer comparison, zoomed-in regions and their corresponding residual maps are provided. The proposed method recovers sharper details and finer textures than existing approaches, achieving around 2 dB higher PSNR than the best competing method. Table~\ref{tab:inpainting} reports the quantitative inpainting results on all datasets. Our method consistently attains the highest PSNR and SSIM for all data types and SRs, while keeping computational cost at a moderate level.
		
		\noindent
		{\bf Multispectral Image Denoising Results.}
		We further evaluate the proposed method on denoising tasks with additive Gaussian noise of varying standard deviations (SDs). Six representative baselines are considered for comparison: LRTDTV~\cite{wang2017hyperspectral}, DeepTensor~\cite{saragadam2024deeptensor}, HLRTF~\cite{luo2022hlrtf}, LRTFR~\cite{luo2023low}, DRO-TFF~\cite{li2025deep}, and NeurTV~\cite{luo2025neurtv}. Fig.~\ref{fig:denoising} presents visual comparisons on the MSI \textit{Face} and HSI \textit{Washington DC} datasets under Gaussian noise with $\text{SD}=0.2$, while quantitative results are summarized in Table~\ref{tab:denoising_avg}. The proposed approach consistently attains the highest PSNR and SSIM, outperforming the strongest baseline by approximately 1~dB on average across all noise levels.
		
		\begin{figure*}[tb]
			\renewcommand{\arraystretch}{0.5}
			\setlength\tabcolsep{0.5pt}
			\centering
			\resizebox{\linewidth}{!}{
				\begin{tabular}{ccccccccc}
					\scriptsize{PSNR/SSIM} & \scriptsize{32.26/0.701} & \scriptsize{34.06/0.859} & \scriptsize{34.67/0.835} & \scriptsize{35.13/0.858} & \scriptsize{36.37/0.920} &  \scriptsize{36.78/0.927} & \scriptsize{\textbf{37.89}/\textbf{0.945}} & \scriptsize{\textbf{Inf/1.000}} \\
					\includegraphics[width=18.5mm]{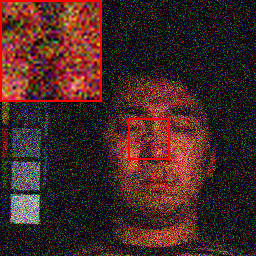} & \includegraphics[width=18.5mm]{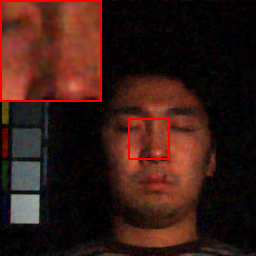} &
					\includegraphics[width=18.5mm]{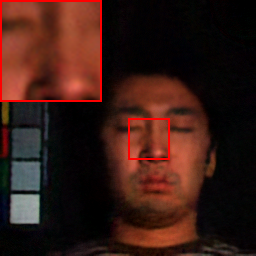} &
					\includegraphics[width=18.5mm]{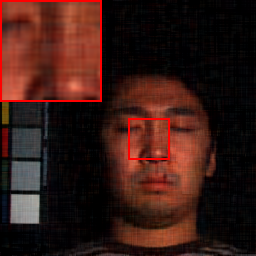} &
					\includegraphics[width=18.5mm]{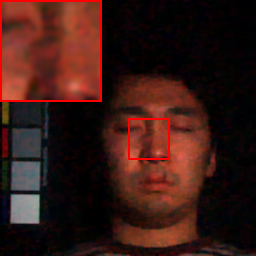} &
					\includegraphics[width=18.5mm]{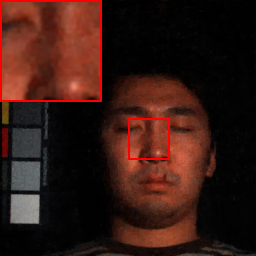} &
					\includegraphics[width=18.5mm]{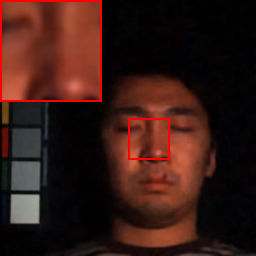} &
					\includegraphics[width=18.5mm]{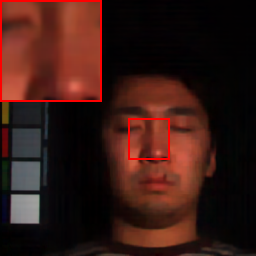} &
					\includegraphics[width=18.5mm]{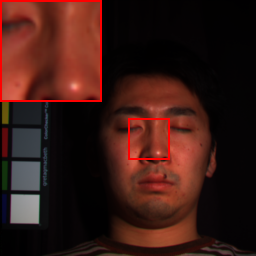} \\
					\scriptsize{PSNR/SSIM} & \scriptsize{31.71/0.869} & \scriptsize{32.32/0.891} & \scriptsize{32.75/0.913} & \scriptsize{33.02/0.914} & \scriptsize{34.26/0.919} &  \scriptsize{33.86/0.910} & \scriptsize{\textbf{35.08}/\textbf{0.945}} & \scriptsize{\textbf{Inf/1.000}} \\
					\includegraphics[width=18.5mm]{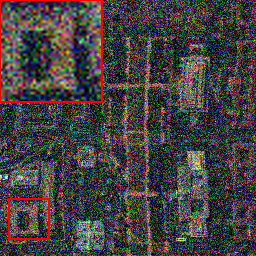} & \includegraphics[width=18.5mm]{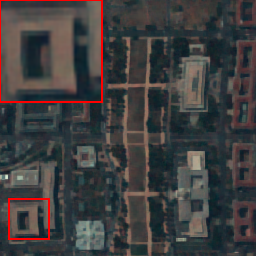} &
					\includegraphics[width=18.5mm]{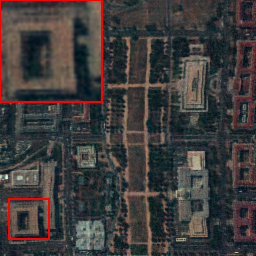} &
					\includegraphics[width=18.5mm]{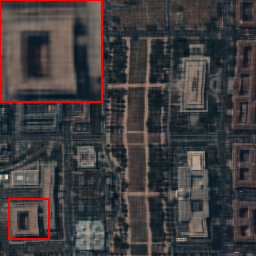} &
					\includegraphics[width=18.5mm]{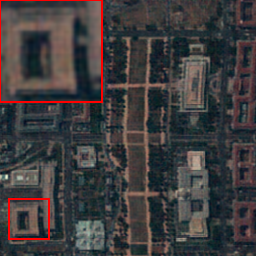} &
					\includegraphics[width=18.5mm]{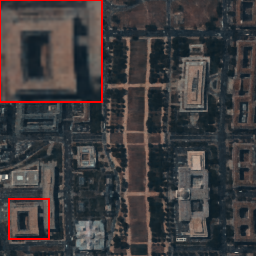} &
					\includegraphics[width=18.5mm]{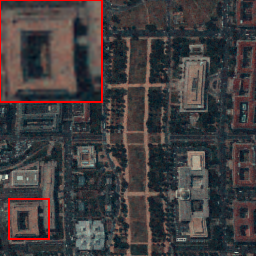} &
					\includegraphics[width=18.5mm]{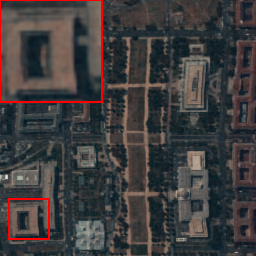} &
					\includegraphics[width=18.5mm]{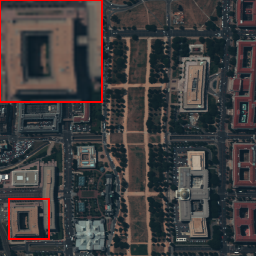} \\
					\scriptsize{Observed} & \scriptsize LRTDTV & \scriptsize DeepTensor & \scriptsize HLRTF & \scriptsize LRTFR & \scriptsize DRO-TFF & \scriptsize NeurTV & \scriptsize Ours & \scriptsize{Ground truth}
			\end{tabular}}
			\vspace{-0.7em}
			\caption{Visual denoising results on MSI \textit{Face} and HSI \textit{Washington DC} ($\text{SD}=0.2$).}
			\label{fig:denoising}
		\end{figure*}
		
		\begin{table}[tb]
			\centering
			\caption{Average quantitative results for MSI and HSI denoising under different Gaussian noise levels.}
			\vspace{-1em}
			\small
			\setlength{\tabcolsep}{3pt}
			\resizebox{\linewidth}{!}{
				\begin{tabular}{clccccccc}
					\toprule
					\multirow{2}{*}{Data} & \multirow{2}{*}{Method} & \multicolumn{2}{c}{$\text{SD}=0.1$} & \multicolumn{2}{c}{$\text{SD}=0.2$} & \multicolumn{2}{c}{$\text{SD}=0.3$} & \multirow{2}{*}{Time (s)} \\
					\cmidrule(lr){3-4} \cmidrule(lr){5-6} \cmidrule(lr){7-8}
					& & PSNR & SSIM & PSNR & SSIM & PSNR & SSIM &  \\
					\midrule
					\multirow{7}{*}{\makecell[c]{\textit{Toys}\\ \textit{Face}\\$(256\times 256\times 31)$}} 
					& LRTDTV & 35.85 & 0.895 & 31.27 & 0.718 & 28.45 & 0.554 & 5.97 \\
					& DeepTensor & 36.23 & 0.914 & 32.48 & 0.833 & 30.25 & 0.702 & 47.93 \\
					& HLRTF & 37.09 & 0.917 & 33.28 & 0.828 & 30.88 & 0.726 & 4.55 \\
					& LRTFR & 37.55 & 0.935 & 33.75 & 0.864 & 31.80 & 0.819 & 5.86 \\
					& DRO-TFF & \underline{38.67} & 0.959 & 34.74 & 0.911 & 32.91 & 0.881 & 7.61 \\
					& NeurTV & 38.61 & \underline{0.960} & \underline{34.98} & \underline{0.915} & \underline{33.37} & \underline{0.883} & 49.20 \\
					& Ours & \textbf{39.02} & \textbf{0.964} & \textbf{35.91} & \textbf{0.933} & \textbf{34.08} & \textbf{0.905} & 8.79 \\
					\hline
					\multirow{7}{*}{\makecell[c]{\textit{Washington DC}\\$(256\times 256\times 191)$\\ \textit{Salinas}\\$(217\times 217\times 224)$}} 
					& LRTDTV & 37.23 & 0.941 & 34.27 & 0.895 & 32.56 & 0.857 & 21.13 \\
					& DeepTensor & 38.13 & 0.954 & 34.69 & 0.906 & 32.81 & 0.856 & 53.67 \\
					& HLRTF & 38.23 & 0.958 & 35.12 & 0.921 & 33.45 & 0.893 & 17.21 \\
					& LRTFR & 38.14 & 0.958 & 35.50 & 0.928 & 33.76 & 0.901 & 16.32 \\
					& DRO-TFF & 38.52 & 0.954 & 36.35 & \underline{0.934} & \underline{34.95} & \underline{0.911} & 34.23 \\
					& NeurTV & \underline{39.36} & \underline{0.963} & \underline{36.75} & 0.932 & 34.78 & 0.899 & 259.53 \\
					& Ours & \textbf{40.37} & \textbf{0.972} & \textbf{37.63} & \textbf{0.950} & \textbf{35.98} & \textbf{0.928} & 35.76\\
					\bottomrule
				\end{tabular}
			}
			\label{tab:denoising_avg}
			\vspace{-1em}
		\end{table}
		
		\noindent
		{\bf Image Super Resolution Results. }We further evaluate our method on single-image super-resolution using samples from the DIV2K dataset~\cite{agustsson2017ntire}, where high-resolution images are cropped to $768\times768$. 
		% High-resolution images are cropped to $768\times768$, and lower-resolution inputs are generated with various scaling factors. 
		The proposed method is compared with several INR-based baselines, including PEMLP~\cite{mildenhall2021nerf}, SIREN~\cite{sitzmann2020implicit}, Gauss~\cite{ramasinghe2022beyond}, WIRE~\cite{saragadam2023wire}, FINER~\cite{liu2024finer}, and LRTFR~\cite{luo2023low}.
		Fig.~\ref{fig:super_resolution} presents visual results on the \textit{Lion} and \textit{Parrot} images at $\times4$ scaling, where our method recovers sharper edges and finer textures, effectively mitigating over-smoothing and aliasing artifacts. Quantitatively, it surpasses the strongest baseline by approximately 1~dB in PSNR on average. Detailed results on \textit{Lion} under different scaling factors are reported in Table~\ref{tab:super_resolution}. Overall, by leveraging a low-rank functional representation, the proposed model consistently achieves higher PSNR and SSIM across all scales while being substantially faster than INR-based baselines.
		
		\begin{figure*}[tb]
			\renewcommand{\arraystretch}{0.5}
			\setlength\tabcolsep{0.5pt}
			\centering
			\resizebox{\linewidth}{!}{
				\begin{tabular}{cccccccc}
					\scriptsize{PSNR 27.84} & \scriptsize{PSNR 29.48} & \scriptsize{PSNR 28.23} & \scriptsize{PSNR 30.21} & \scriptsize{PSNR 29.84} &  \scriptsize{PSNR 28.10} & \scriptsize{PSNR \textbf{31.01}} & \scriptsize{PSNR \textbf{Inf}} \\
					\includegraphics[width=18.5mm]{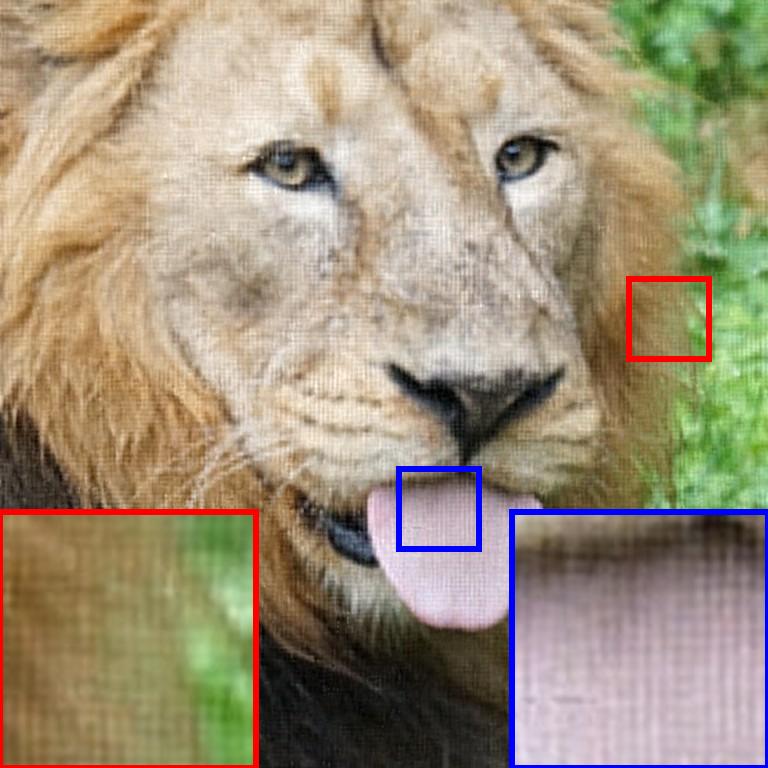} &
					\includegraphics[width=18.5mm]{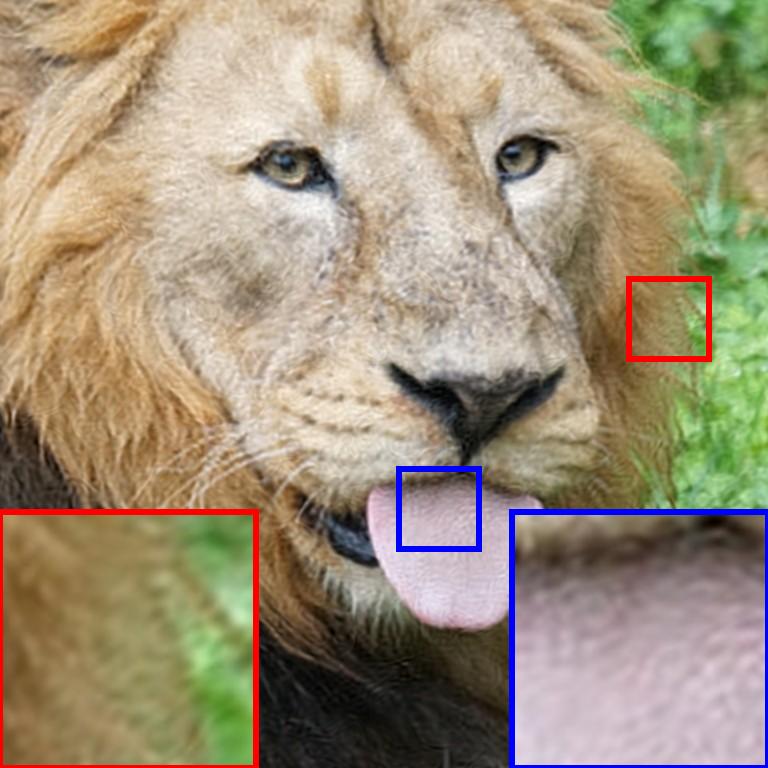} &
					\includegraphics[width=18.5mm]{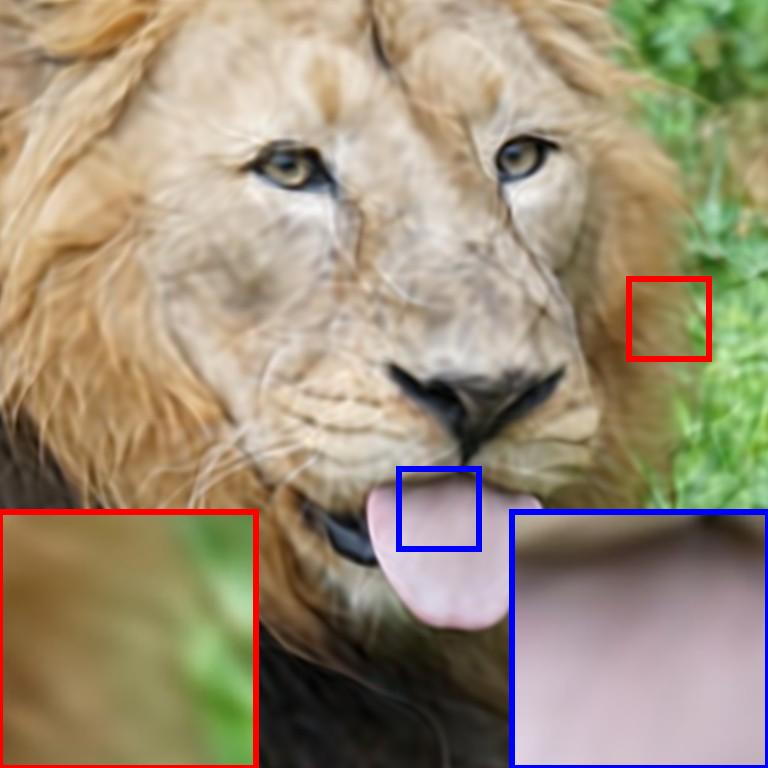} &
					\includegraphics[width=18.5mm]{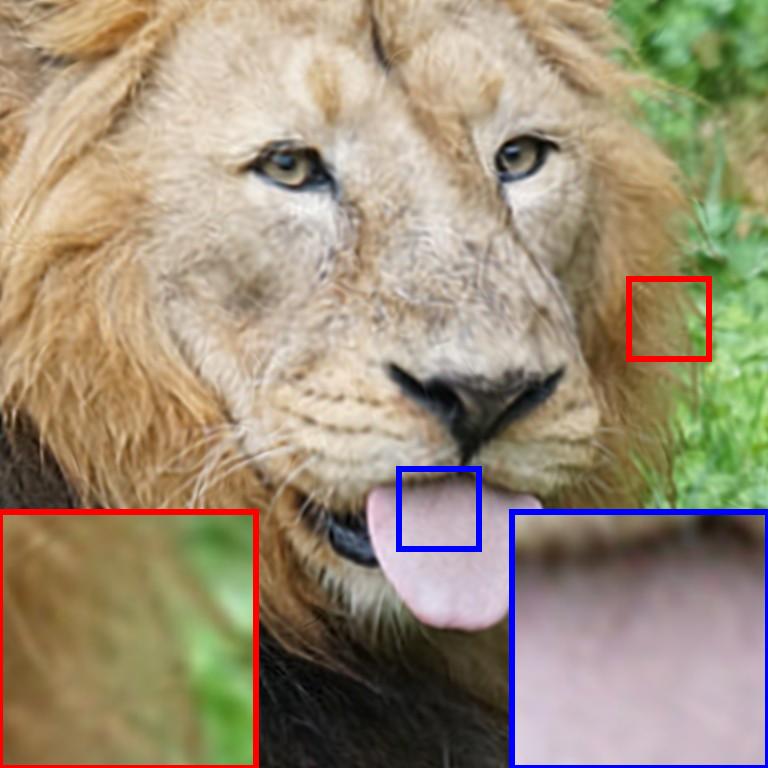} &
					\includegraphics[width=18.5mm]{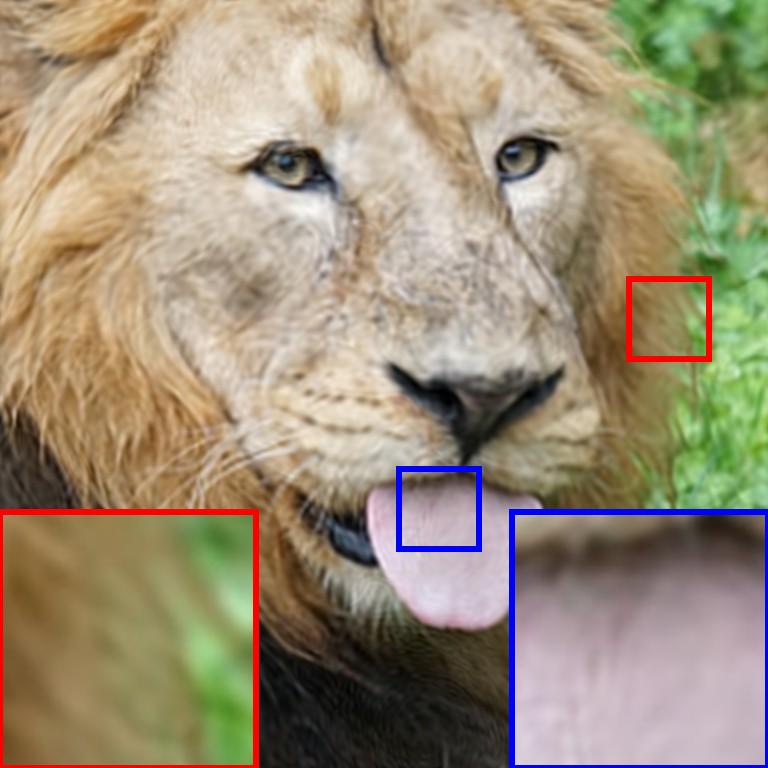} &
					\includegraphics[width=18.5mm]{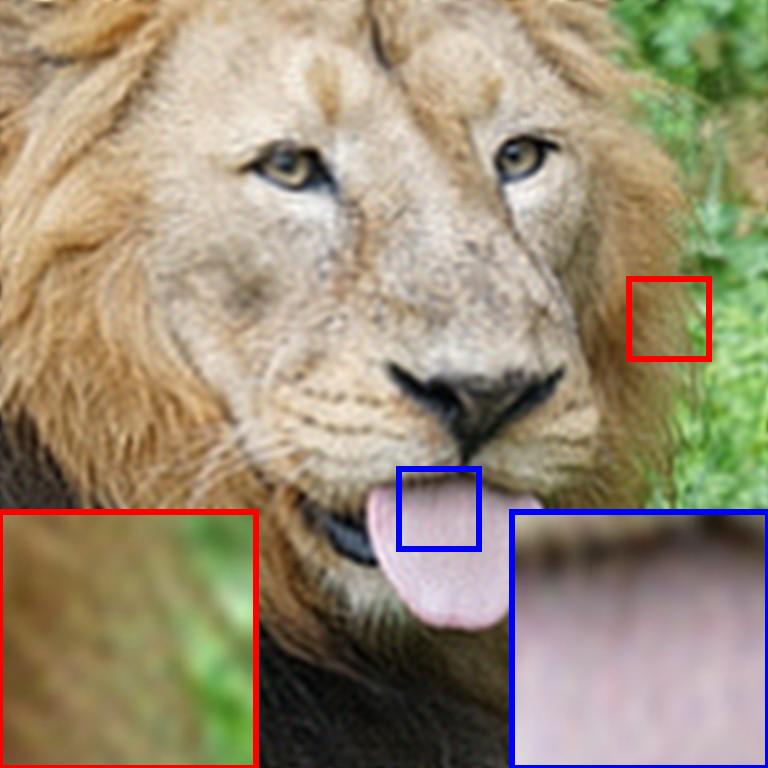} &
					\includegraphics[width=18.5mm]{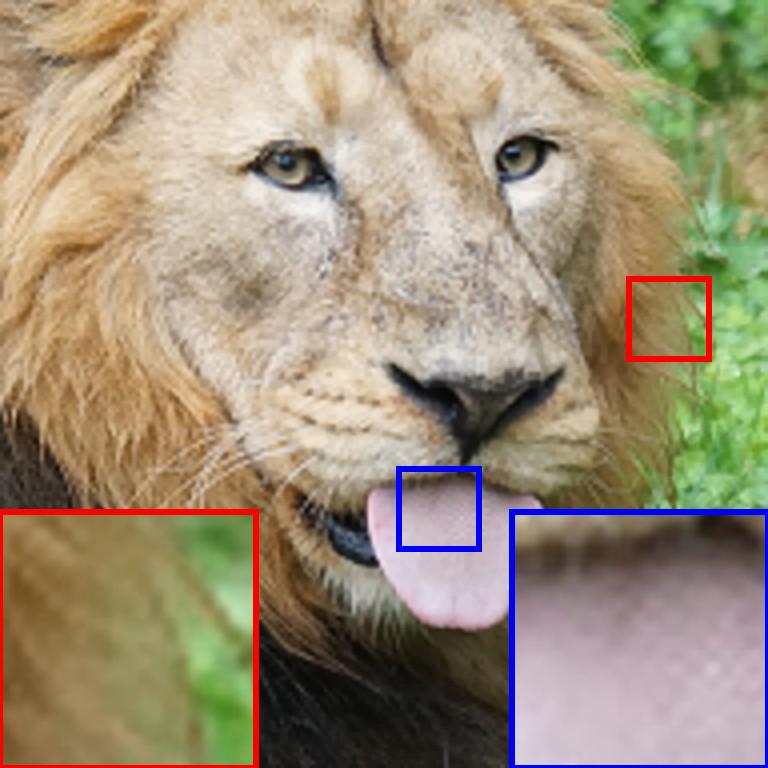} &
					\includegraphics[width=18.5mm]{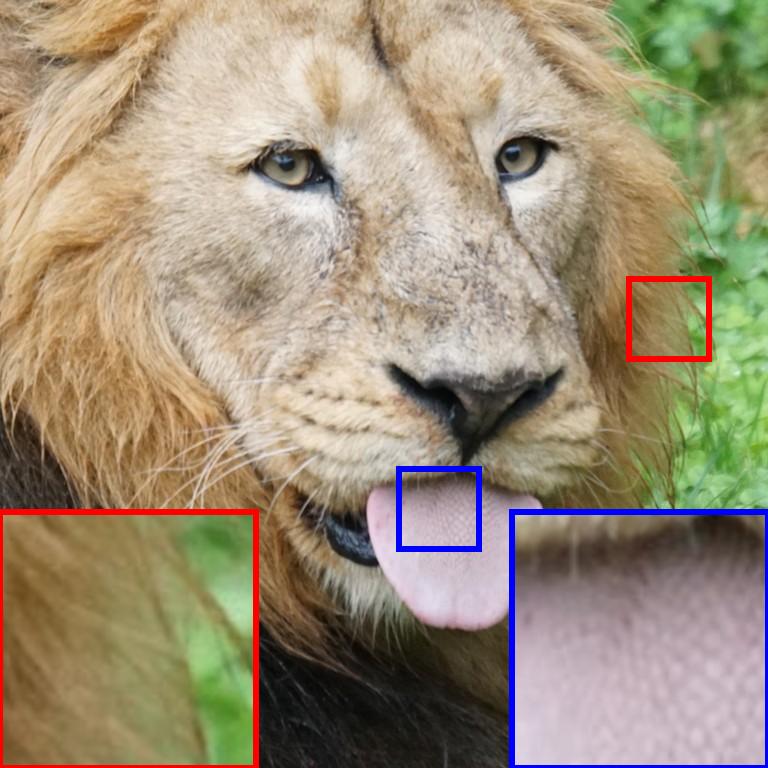} \\
					\scriptsize{PSNR 27.11} & \scriptsize{PSNR 28.16} & \scriptsize{PSNR 27.39} & \scriptsize{PSNR 29.27} & \scriptsize{PSNR 29.47} &  \scriptsize{PSNR 27.81} & \scriptsize{PSNR \textbf{30.39}} & \scriptsize{PSNR \textbf{Inf}} \\
					\includegraphics[width=18.5mm]{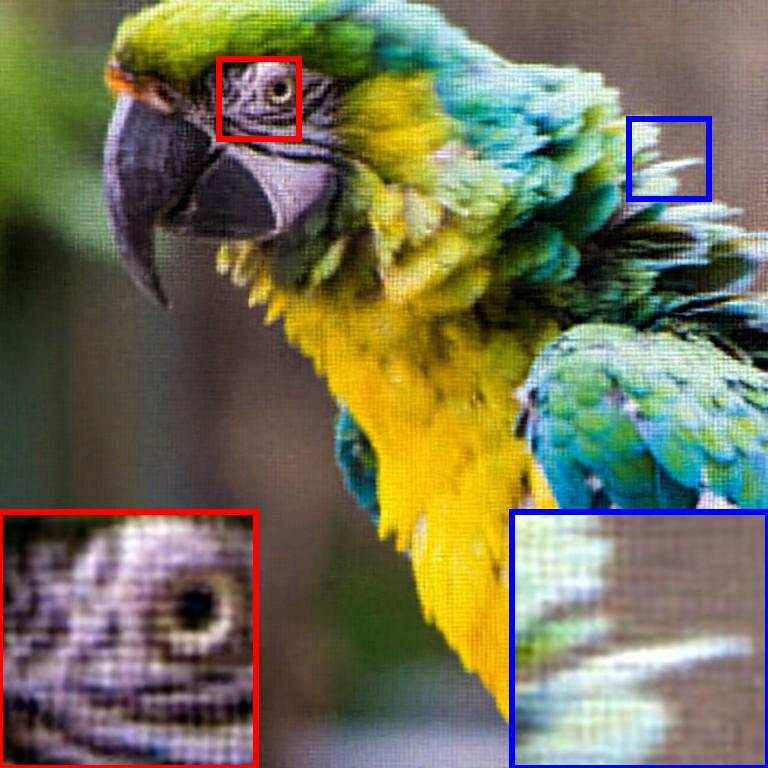} &
					\includegraphics[width=18.5mm]{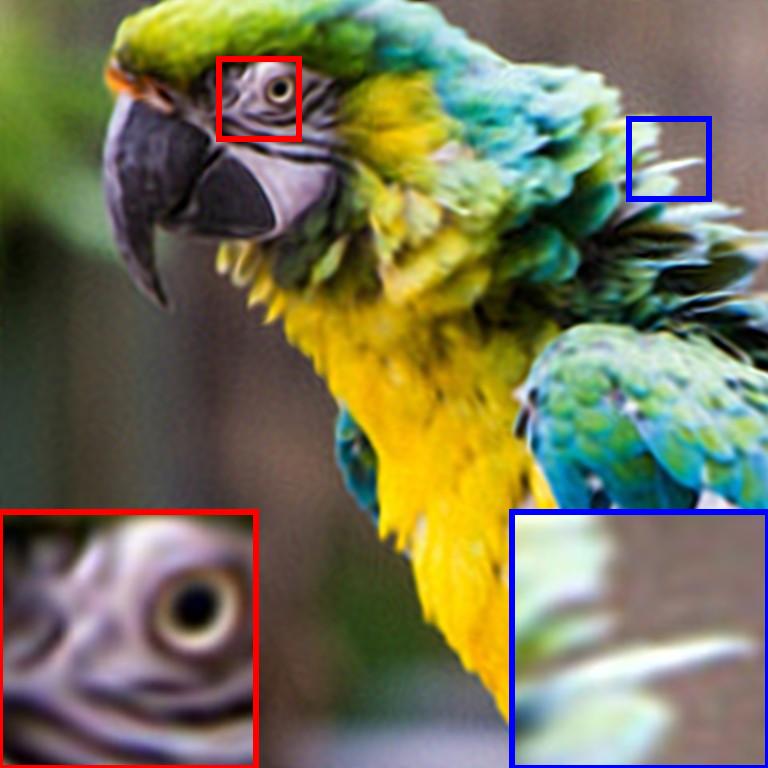} &
					\includegraphics[width=18.5mm]{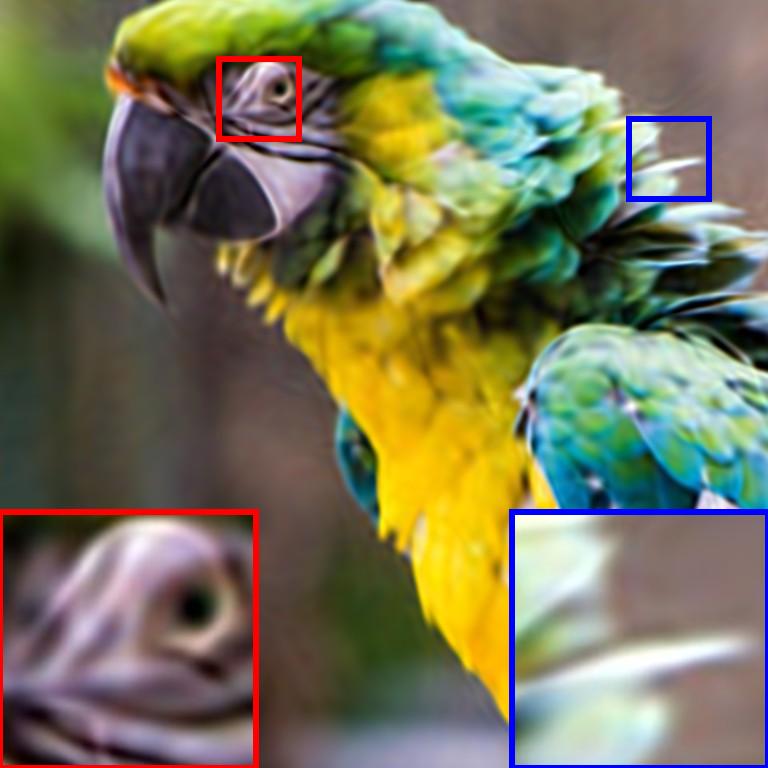} &
					\includegraphics[width=18.5mm]{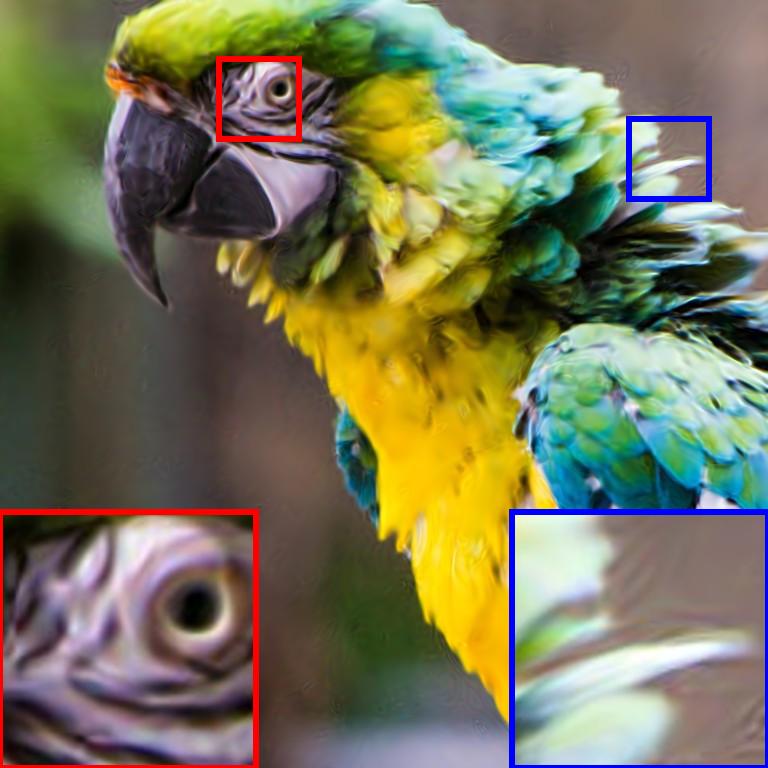} &
					\includegraphics[width=18.5mm]{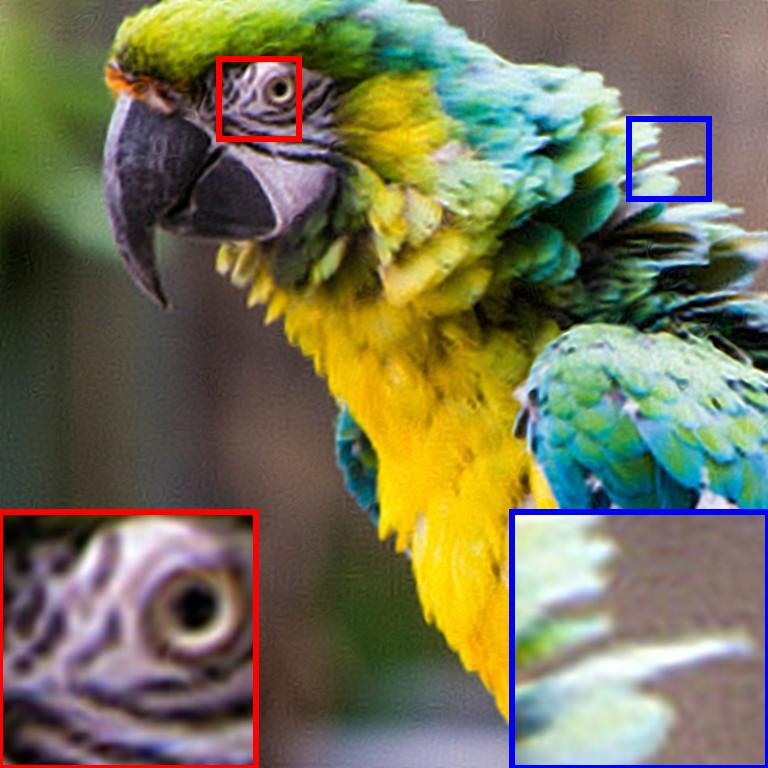} &
					\includegraphics[width=18.5mm]{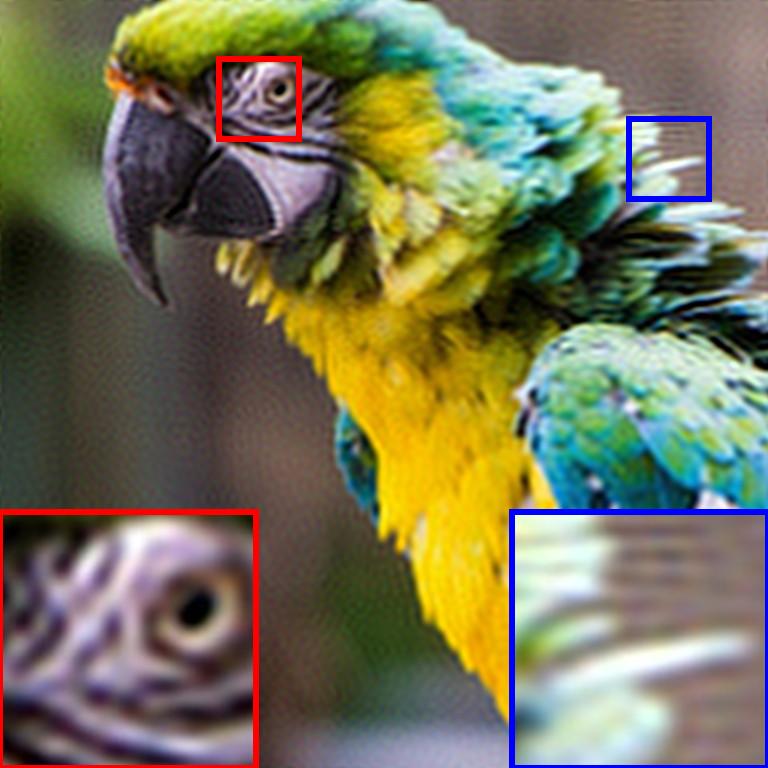} &
					\includegraphics[width=18.5mm]{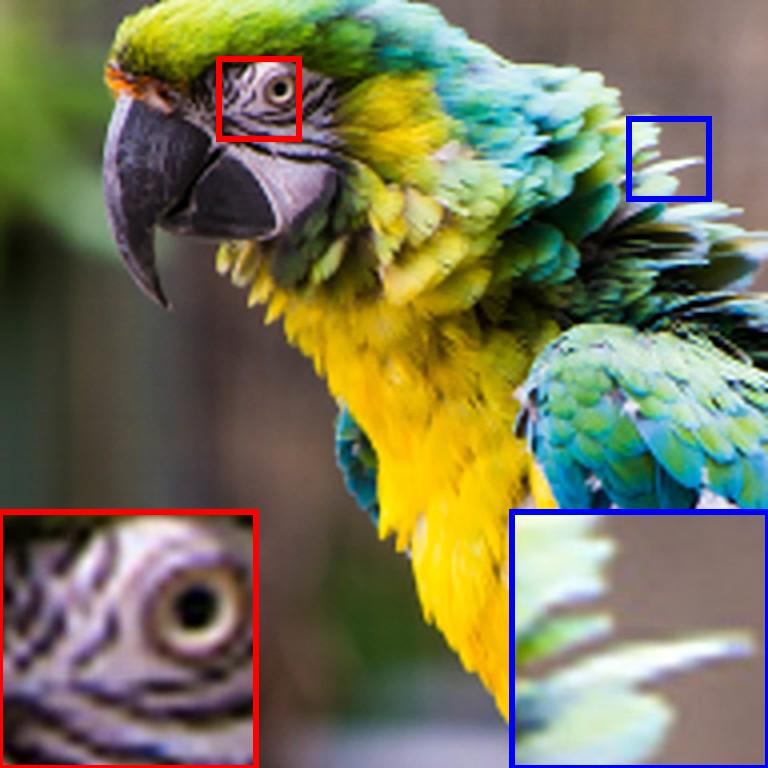} &
					\includegraphics[width=18.5mm]{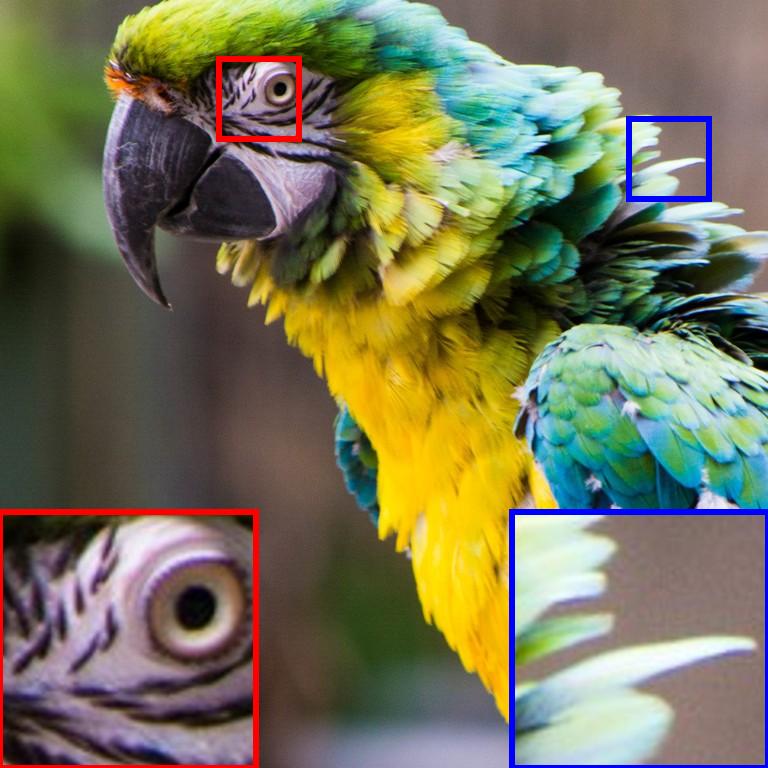} \\
					\scriptsize PEMLP & \scriptsize SIREN & \scriptsize Gauss & \scriptsize WIRE & \scriptsize FINER & \scriptsize LRTFR & \scriptsize Ours & \scriptsize{Ground truth}
			\end{tabular}}
			\vspace{-0.7em}
			\caption{Visual super-resolution results at $\times4$ scaling on the \textit{Lion} and \textit{Parrot} images.}
			\label{fig:super_resolution}
			\vspace{-1em}
		\end{figure*}
		
		\begin{table}[tb]
			\centering
			\caption{Super-resolution results on \textit{Lion} at different scales.}
			\vspace{-1em}
			\small
			\setlength{\tabcolsep}{6pt}
			\resizebox{\linewidth}{!}{
				\begin{tabular}{lccccccc}
					\toprule
					\multirow{2}{*}{Method} & \multicolumn{2}{c}{$\times 3$} & \multicolumn{2}{c}{$\times 4$} & \multicolumn{2}{c}{$\times 6$} & \multirow{2}{*}{Time (s)} \\
					\cmidrule(lr){2-3} \cmidrule(lr){4-5} \cmidrule(lr){6-7}
					& PSNR & SSIM & PSNR & SSIM & PSNR & SSIM &  \\
					\midrule
					PEMLP & 29.26 & 0.740 & 27.84 & 0.676 & 24.28 & 0.439 & 90.15 \\
					SIREN & 31.69 & 0.866 & 29.48 & 0.779 & 27.16 & 0.634 & 125.38 \\
					Gauss & 30.27 & 0.790 & 28.23 & 0.683 & 26.73 & 0.605 & 186.92 \\
					WIRE & \underline{32.43} & \underline{0.869} & \underline{30.21} & \underline{0.788} & \underline{27.88} & \underline{0.683} & 423.27 \\
					FINER & 32.09 & 0.858 & 29.84 & 0.778 & 27.61 & 0.666 & 165.08 \\
					LRTFR & 30.11 & 0.805 & 28.10 & 0.701 & 25.27 & 0.577 & 10.26 \\
					Ours & \textbf{33.41} & \textbf{0.907} & \textbf{31.01} & \textbf{0.836} & \textbf{28.34} & \textbf{0.711} & 12.65 \\
					\bottomrule
				\end{tabular}
			}
			\label{tab:super_resolution}
			\vspace{-1em}
		\end{table}
		
		\noindent
		{\bf Point Cloud Recovery Results.}
		This task aims to learn a continuous mapping for reconstructing complete point clouds from sparsely observed samples, where discrete tensor representations often fail to generalize.
		Following the setup in~\cite{luo2024revisiting,luo2025neurtv}, we employ the SHOT dataset~\cite{salti2014shot} and evaluate performance under different SRs.
		The performance is measured by normalized root mean square error (NRMSE).
		Fig.~\ref{fig:point_cloud} presents visual comparisons on the \textit{Duck} ($\text{SR}=0.15$) and \textit{Mario} ($\text{SR}=0.2$) point clouds, where the proposed method yields the most faithful reconstructions, accurately recovering fine geometric structures and smooth surfaces.
		Quantitative results summarized in Table~\ref{tab:point_cloud} further confirm that our approach consistently achieves the lowest reconstruction error across all datasets under $\text{SR}=0.2$, demonstrating its strong capability in continuous 3D signal recovery.
		
		\begin{figure*}[tb]
			\renewcommand{\arraystretch}{1}
			\setlength\tabcolsep{0.5pt}
			\centering
			\resizebox{\linewidth}{!}{
				\begin{tabular}{ccccccccc}
					\scriptsize{NRMSE} & \scriptsize{NRMSE 0.073} & \scriptsize{NRMSE 0.071} & \scriptsize{NRMSE 0.072} & \scriptsize{NRMSE 0.072} & \scriptsize{NRMSE 0.070} &  \scriptsize{NRMSE 0.071} & \scriptsize{NRMSE \textbf{0.061}} & \scriptsize{NRMSE \textbf{0.000}} \\
					\includegraphics[width=18.5mm]{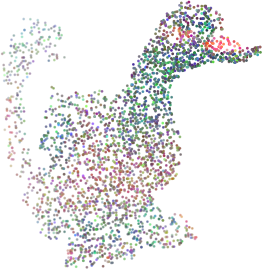} & \includegraphics[width=18.5mm]{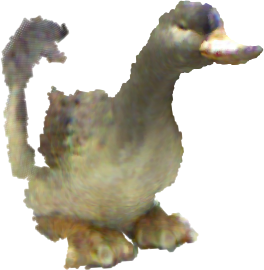} &
					\includegraphics[width=18.5mm]{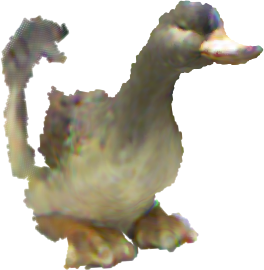} &
					\includegraphics[width=18.5mm]{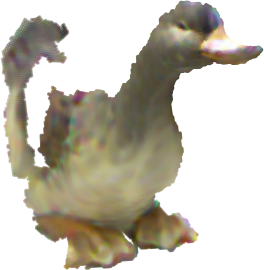} &
					\includegraphics[width=18.5mm]{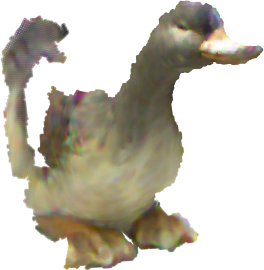} &
					\includegraphics[width=18.5mm]{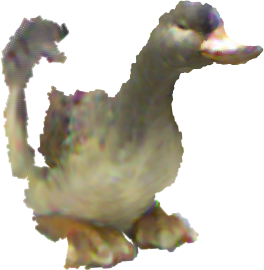} &
					\includegraphics[width=18.5mm]{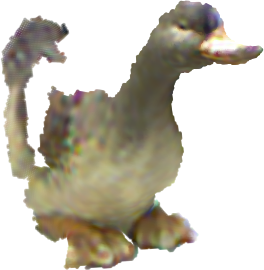} &
					\includegraphics[width=18.5mm]{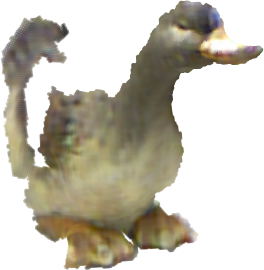} &
					\includegraphics[width=18.5mm]{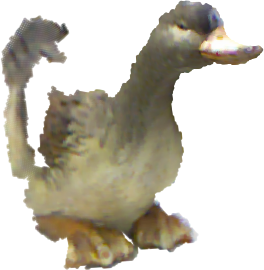} \\
					\scriptsize{NRMSE} & \scriptsize{NRMSE 0.091} & \scriptsize{NRMSE 0.089} & \scriptsize{NRMSE 0.092} & \scriptsize{NRMSE 0.086} & \scriptsize{NRMSE 0.088} &  \scriptsize{NRMSE 0.095} & \scriptsize{NRMSE \textbf{0.080}} & \scriptsize{NRMSE \textbf{0.000}} \\
					\includegraphics[width=18.5mm]{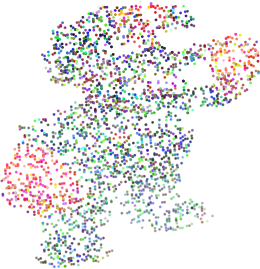} & \includegraphics[width=18.5mm]{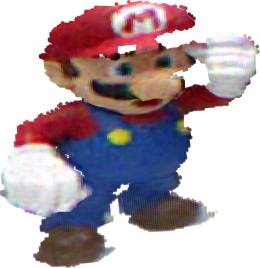} &
					\includegraphics[width=18.5mm]{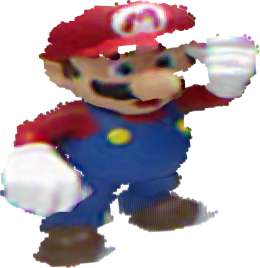} &
					\includegraphics[width=18.5mm]{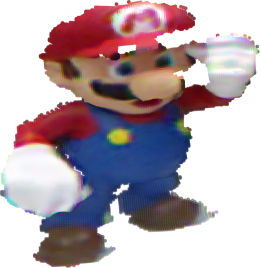} &
					\includegraphics[width=18.5mm]{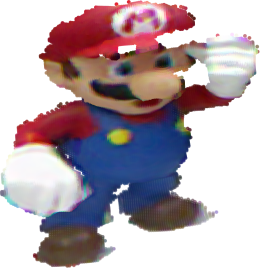} &
					\includegraphics[width=18.5mm]{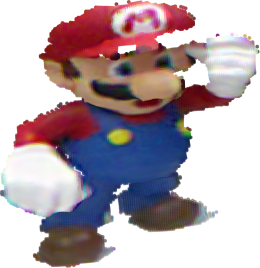} &
					\includegraphics[width=18.5mm]{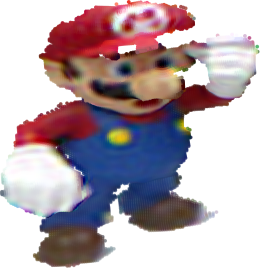} &
					\includegraphics[width=18.5mm]{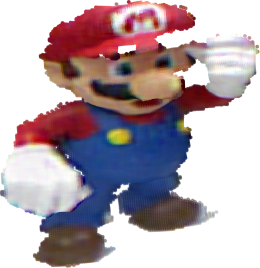} &
					\includegraphics[width=18.5mm]{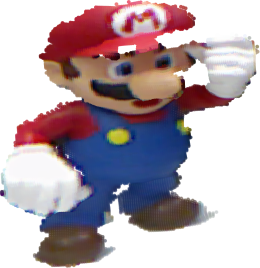} \\
					\scriptsize{Observed} & \scriptsize PEMLP & \scriptsize SIREN & \scriptsize Gauss & \scriptsize WIRE & \scriptsize FINER & \scriptsize LRTFR & \scriptsize Ours & \scriptsize{Ground truth}
			\end{tabular}}
			\vspace{-0.7em}
			\caption{Visual comparisons on point cloud recovery for \textit{Duck} ($\text{SR} = 0.15$) and \textit{Mario} ($\text{SR} = 0.2$).}
			\label{fig:point_cloud}
			\vspace{-0.7em}
		\end{figure*}
		
		\begin{table}[tb]
			\centering
			\caption{NRMSE results for point cloud recovery on the SHOT dataset under $\text{SR} = 0.2$.}
			\vspace{-1em}
			\small
			\setlength{\tabcolsep}{7pt}
			\resizebox{\linewidth}{!}{  
				\begin{tabular}{lcccccc}
					\toprule
					Method & \textit{Doll} & \textit{Duck} & \textit{Frog} & \textit{Mario} & \textit{PeterRabbit} & \textit{Squirrel}  \\
					\midrule
					PEMLP  & 0.112 & 0.064 & 0.058 & 0.091 & 0.081 & 0.084 \\
					SIREN  & 0.111 & 0.062 & 0.057 & 0.089 & 0.074 & 0.085 \\
					Gauss  & 0.123 & 0.064 & 0.060 & 0.092 & 0.078 & 0.085 \\
					WIRE   & \underline{0.106} & 0.060 & \underline{0.053} & \underline{0.086} & \textbf{0.068} & \underline{0.083} \\
					FINER  & 0.110 & \underline{0.059} & 0.054 & 0.088 & \underline{0.073} & \underline{0.083} \\
					LRTFR  & 0.114 & 0.061 & 0.061 & 0.095 & 0.082 & 0.094 \\
					Ours   & \textbf{0.093} & \textbf{0.053} & \textbf{0.050} & \textbf{0.080} & \textbf{0.068} & \textbf{0.073} \\
					\bottomrule
				\end{tabular}
			}
			\label{tab:point_cloud}
			\vspace{-1.5em}
		\end{table}

		\subsection{Discussions}
		
		To further investigate the contribution of each component within our framework, we perform a series of ablation studies under controlled conditions. All experiments are conducted with identical parameter settings and training protocols to ensure a fair comparison.
		
		\noindent
		{\bf Effect of Reparameterization.}
		We evaluate the effect of the proposed reparameterization on inpainting tasks using the color image \textit{Airplane} with $\text{SR}=0.3$ and the MSI \textit{Toy} with $\text{SR}=0.1$. Comparisons are made between models without reparameterization (w/o Rep.) and with reparameterization (w/ Rep.). For the reparameterized models, the latent dimension is set as $R_k = \beta r_k$, where $r_k$ is the original TR rank and $\beta$ takes values 1, 3, 5, and 10. Fig.~\ref{fig:ablation_rep} shows the PSNR curves, indicating that reparameterization enhances reconstruction quality and stabilizes training, with larger $\beta$ leading to faster convergence. Table~\ref{tab:ablation_rep} reports quantitative results across several datasets in terms of PSNR and runtime, confirming that the proposed reparameterization substantially enhances reconstruction accuracy with only a marginal increase in computational cost.
		
		\begin{figure}[tbp]
			\centering
			\resizebox{\linewidth}{!}{
				\begin{tabular}{cc}
					\includegraphics[width=0.5\linewidth]{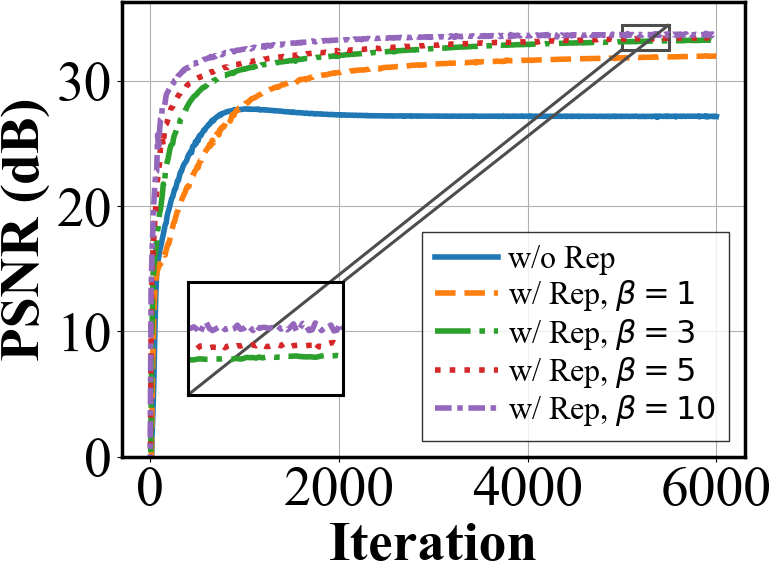} & \includegraphics[width=0.5\linewidth]{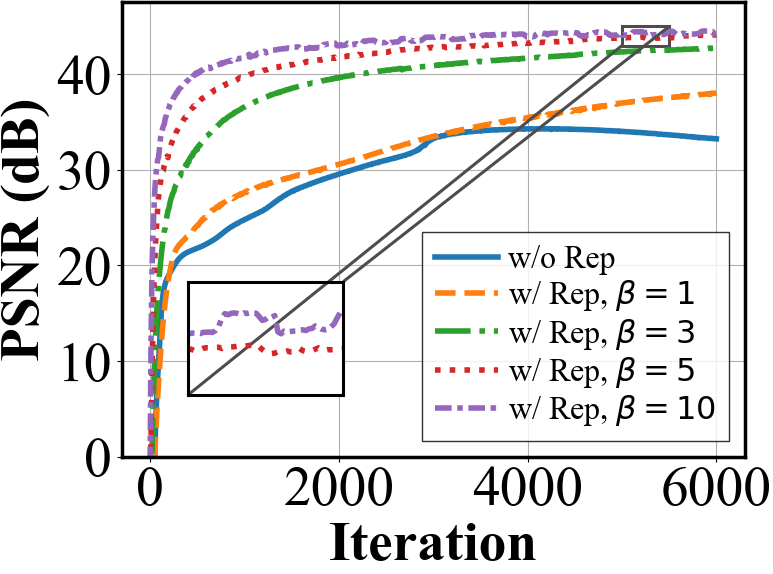} 
			\end{tabular}}
			\vspace{-1em}
			\caption{PSNR curves on the color image \textit{Airplane} ($\text{SR}=0.3$) (left) and MSI \textit{Toy} ($\text{SR}=0.1$) (right), comparing models without reparameterization (w/o Rep.), with reparameterization (w/ Rep.), and different latent sizes $R_k = \beta r_k$ for $\beta = 1, 3, 5, 10$.}
			\label{fig:ablation_rep}
			% \vspace{-0.5em}
		\end{figure}
		
		\begin{table}[tbp]
			\centering
			\caption{Effect of reparameterization on PSNR and runtime for different datasets with $\text{SR}=0.2$.}
			\vspace{-1em}
			\small
			\setlength{\tabcolsep}{3pt}
			\resizebox{\linewidth}{!}{
				\begin{tabular}{lcc|cc|cc|cc}
					\toprule
					\multirow{2}{*}{Metric} & \multicolumn{2}{c}{\textit{Airplane}} & \multicolumn{2}{c}{\textit{Toy}} & \multicolumn{2}{c}{\textit{Botswana}} & \multicolumn{2}{c}{\textit{News}} \\
					\cmidrule(lr){2-3} \cmidrule(lr){4-5} \cmidrule(lr){6-7} \cmidrule(lr){8-9}
					& w/o Rep. & w/ Rep. & w/o Rep. & w/ Rep.& w/o Rep. & w/ Rep. & w/o Rep. & w/ Rep. \\
					\midrule
					PSNR & 27.41 & \textbf{30.45} & 29.41 & \textbf{48.67} & 29.21 & \textbf{45.27} & 26.48 & \textbf{34.90} \\
					Time (s) & 12.10 & 13.11 & 13.23 & 14.49 & 40.81 & 42.26 & 14.15 & 15.59 \\
					\bottomrule
				\end{tabular}
			}
			\label{tab:ablation_rep}
			\vspace{-0.5em}
		\end{table}
		
		\begin{table}[tbp]
			\centering
			\caption{PSNR results for HSI inpainting with different basis initialization scale at $\text{SR}=0.2$.}
			\vspace{-1em}
			% \vspace{-3mm}
			\small
			\setlength{\tabcolsep}{6pt}
			\resizebox{\linewidth}{!}{
				\begin{tabular}{lcccccc}
					\toprule
					\multirow{2}{*}{Data} & \multicolumn{6}{c}{Initialization Scale $a$} \\
					\cmidrule(lr){2-7}
					& 0.01 & 0.05 & 0.165 (ours) & 0.3 & 0.5 & 1 \\
					\midrule
					\textit{Botswana} & 33.86 & 43.74 & \textbf{45.27} & 44.81 & 43.15 & 38.37 \\
					\textit{Washington DC} & 33.14 & 46.08 & \textbf{47.96} & 47.44 & 45.90 & 38.20 \\
					\bottomrule
				\end{tabular}
				% \begin{tabular}{lccccc}
					% \toprule
					% Data & $a=0.01$ & $a=0.05$ & $a=0.165$ (ours) & $a=0.5$ & $a=1$  \\
					% \midrule
					% \textit{Botswana} & 39.62 & 43.53 & \textbf{45.27} & 43.15 & 36.30 \\
					% \textit{Washington DC} & 39.88 & 45.68 & \textbf{47.96} & 45.90 & 40.26 \\
					% \bottomrule
					% \end{tabular}
			}
			\label{tab:ablation_basis}
			\vspace{-1em}
		\end{table}
		
		% \noindent
		% {\bf Effect of Basis Initialization. }To test the impact of the basis initialization range on reconstruction quality, we conduct HSI inpainting experiments on the \textit{Washington DC} and \textit{Botswana} datasets with $\text{SR}=0.2$. According to Theorem~\ref{thm:basis_init}, the entries of the fixed basis $\Bb^{(k)}$ are sampled independently from a uniform distribution $\cU(-a,a)$, where the theoretically derived value of $a$ is $\sqrt{6/(r_{k+1}+R_{k+1})}\approx 0.165$. To evaluate sensitivity, we vary $a\in\{0.01, 0.05,0.3, 0.5,1\}$. As shown in Table~\ref{tab:ablation_basis}, extremely small or large values lead to poor reconstructions, while the proposed scheme achieves the best performance, highlighting the importance of proper basis initialization.
		\noindent
		{\bf Sensitivity of Basis Initialization. }To evaluate the sensitivity of basis initialization, we conduct HSI inpainting experiments on the \textit{Washington DC} and \textit{Botswana} datasets with $\text{SR}=0.2$. According to Theorem~\ref{thm:basis_init}, the entries of the fixed basis $\Bb^{(k)}$ are sampled independently from a uniform distribution $\cU(-a,a)$, where the theoretically derived value of $a$ is $\sqrt{6/(r_{k+1}+R_{k+1})}\approx 0.165$. As shown in Table~\ref{tab:ablation_basis}, extremely small or large values lead to poor reconstructions, while the proposed scheme achieves the best performance, highlighting the importance of proper basis initialization.
		
		\noindent
		{\bf Effect of Shared Frequency Embedding.}
		We evaluate the impact of the shared frequency embedding on denoising tasks using the MSIs \textit{Toy} and \textit{Face} with $\text{SD}=0.2$. In the baseline, each coordinate is processed independently, whereas in our proposed design, all coordinates are passed through a shared frequency embedding network. Fig.~\ref{fig:ablation_share} presents the PSNR curves over iterations, showing that the shared embedding produces more stable training and mitigates overfitting, resulting in more robust reconstructions.
		
		\begin{figure}[tbp]
			\centering
			\resizebox{\linewidth}{!}{
				\begin{tabular}{cc}
					\includegraphics[width=0.5\linewidth]{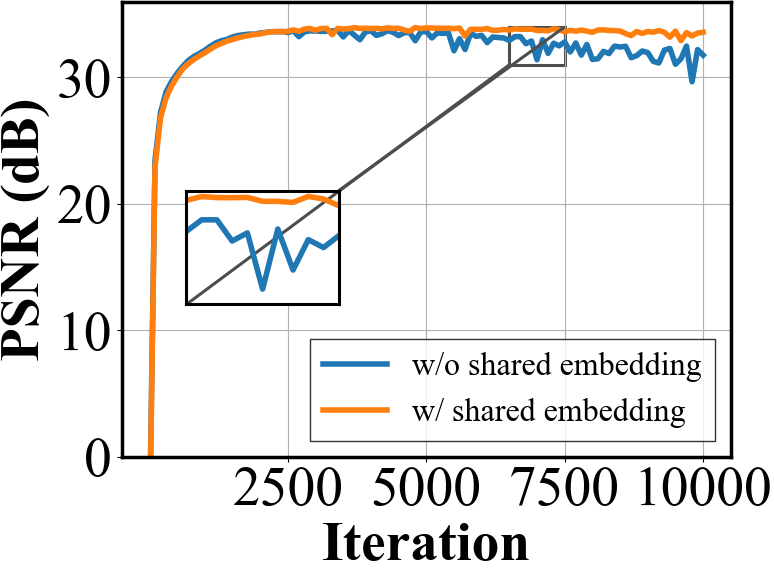} & \includegraphics[width=0.5\linewidth]{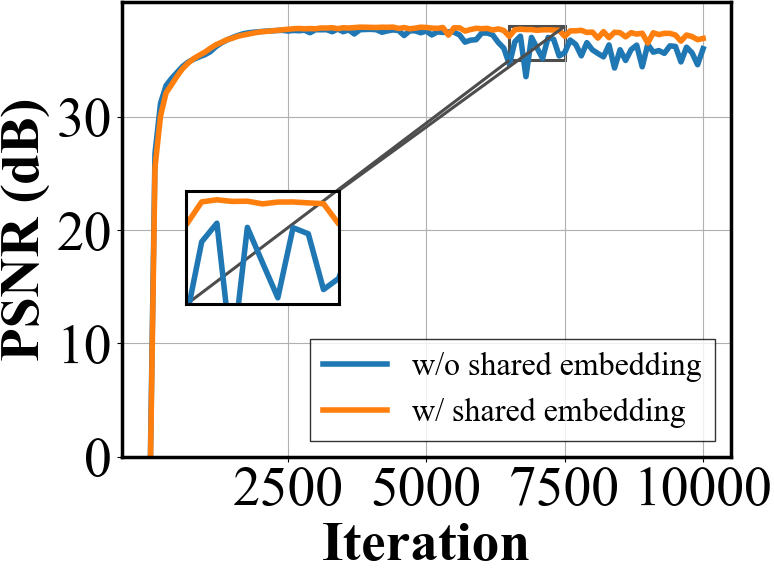} 
			\end{tabular}}
			\vspace{-1em}
			\caption{Effect of shared frequency embedding on denoising performance for the MSIs \textit{Toy} (left) and \textit{Face} (right) with $\text{SD}=0.2$.}
			\label{fig:ablation_share}
			\vspace{-1em}
		\end{figure}
		
		\noindent{\bf Model Complexity. }We vary $\beta_k$ to obtain models with different parameter counts and compare them with LRTFR under matched complexity. 
		For LRTFR, the MLP width is adjusted to ensure comparable parameters and FLOPs. Fig.~\ref{fig:psnr_flops_params} reports PSNR versus model complexity, showing that RepTRFD consistently outperforms LRTFR under similar parameter and FLOP budgets.
		% Fig.~\ref{fig:psnr_flops_params}(b) further shows the training loss curves of TRFD, LRTFR, and RepTRFD under matched parameter settings. 
		% RepTRFD converges faster and reaches a lower loss, demonstrating improved optimization behavior.
		
		\begin{figure}[tbp]
			\centering
			\resizebox{\linewidth}{!}{
				\begin{tabular}{c}
					\includegraphics[width=\linewidth]{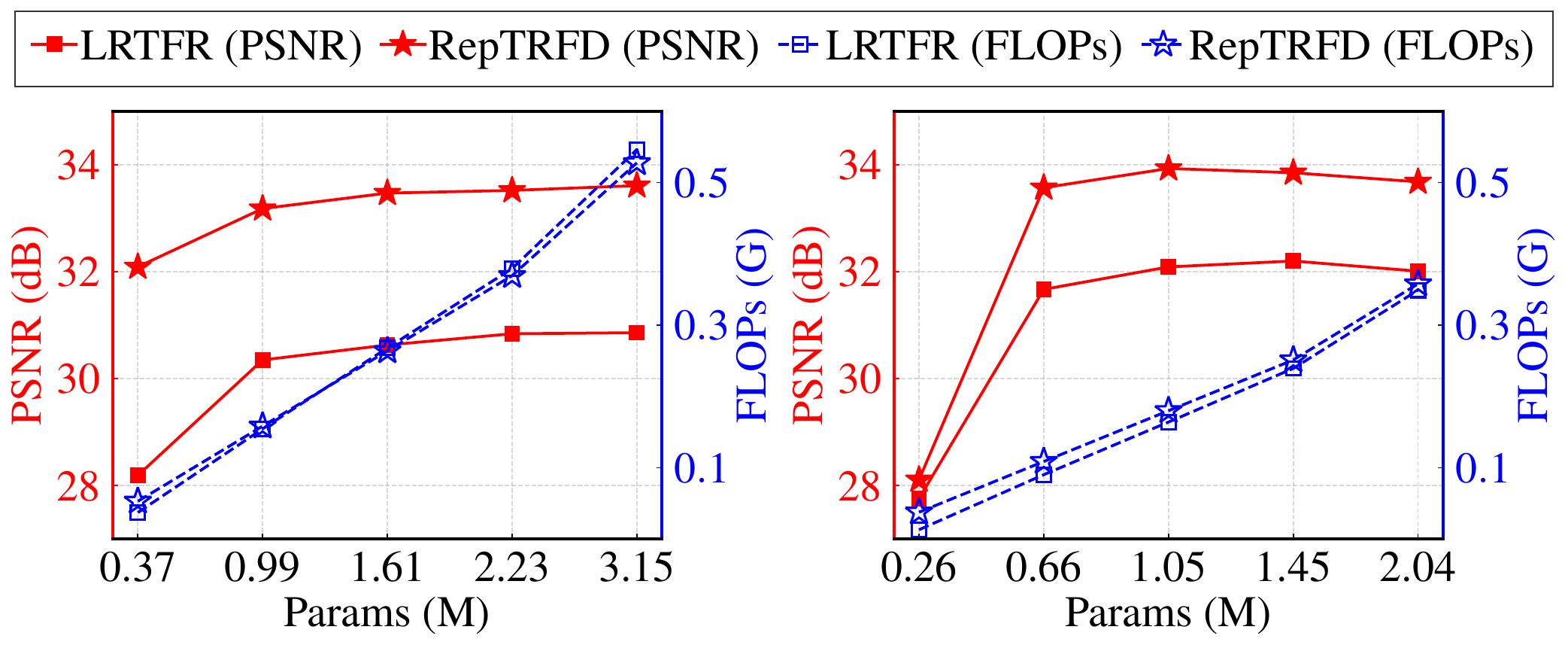} 
			\end{tabular}}
			\vspace{-1em}
			\caption{PSNR and FLOPs versus parameter count on \textit{Airplane} ($\text{SR}=0.3$) (left) and \textit{Toy} ($\text{SD}=0.2$) (right).}
			\label{fig:psnr_flops_params}
			\vspace{-0.5em}
		\end{figure}
		
		\noindent{\bf Scalability to Higher-Order Tensors. }To evaluate the scalability of RepTRFD to higher-order data, 
		we conduct inpainting experiments on color videos. Fig.~\ref{fig:color_video} presents representative results on the \textit{Container} and \textit{Salesman} sequences. 
		RepTRFD achieves consistently higher PSNR and SSIM values while producing visually sharper details. 
		
		\begin{figure}[tbp]
			\centering
			\renewcommand{\arraystretch}{0.5}
			\setlength\tabcolsep{0.5pt}
			\resizebox{\linewidth}{!}{
				\begin{tabular}{ccccc}
					\scriptsize{PSNR/SSIM} & \scriptsize{34.94/0.953} & \scriptsize{34.48/0.950} & \scriptsize{35.55/0.954} & \scriptsize{\textbf{36.39}/\textbf{0.968}}  \\
					\includegraphics[width=18.5mm]{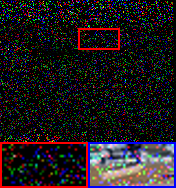} & 
					\includegraphics[width=18.5mm]{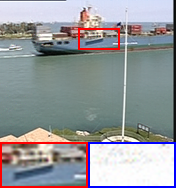}
					& 
					\includegraphics[width=18.5mm]{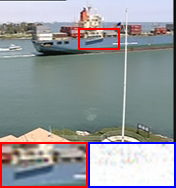} & \includegraphics[width=18.5mm]{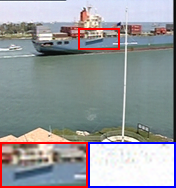} & \includegraphics[width=18.5mm]{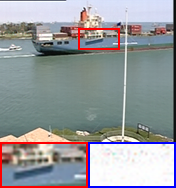} \\
					\scriptsize{PSNR/SSIM} & \scriptsize{34.49/0.973} & \scriptsize{34.09/0.967} & \scriptsize{35.09/0.974} & \scriptsize{\textbf{35.99}/\textbf{0.978}} \\
					\includegraphics[width=18.5mm]{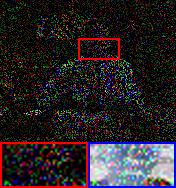} & 
					\includegraphics[width=18.5mm]{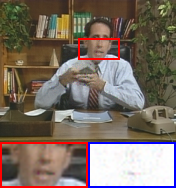}
					& 
					\includegraphics[width=18.5mm]{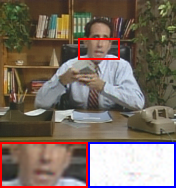} & \includegraphics[width=18.5mm]{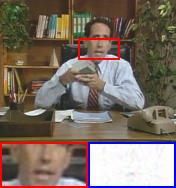}& \includegraphics[width=18.5mm]{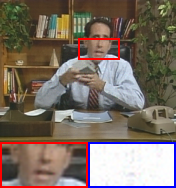} \\
					\scriptsize{Observed} & \scriptsize{HLRTF} & \scriptsize{LRTFR} & \scriptsize{DRO-TFF} & \scriptsize{Ours}
			\end{tabular}}
			\vspace{-0.5em}
			\caption{Visual inpainting results on color videos \textit{Container} ($\text{SR}=0.1$) and \textit{Salesman} ($\text{SR}=0.2$).}
			\label{fig:color_video}
			\vspace{-0.5em}
		\end{figure}

		\section{Conclusion}
		
		In this work, we presented RepTRFD for multi-dimensional data recovery. Specifically, by analyzing the TR's frequency behavior, we identified the critical role of high-frequency components and proposed a structured reparameterization of TR factors that significantly facilitates their learning. We further established Lipschitz continuity and derived a variance-preserving initialization for the basis matrices, ensuring stable and effective training. Extensive experiments in various tasks validated our approach. 
		Future work includes extending the framework for Tucker decomposition and block-term tensor decomposition.
		
		\section*{Acknowledgment}
		This work was partially supported by the National Natural Science Foundation of China under Grants  12571564, 12361089, and the Guangdong Basic and Applied Research Foundation 2024A1515012347. 
		% In this work, we presented RepTRFD for multi-dimensional data recovery, providing both theoretical insights and practical improvements. Specifically, by analyzing the TR's frequency behavior, we identified the critical role of high-frequency components and proposed a structured reparameterization of TR factors that significantly facilitates their learning. We further established Lipschitz continuity and derived a variance-preserving initialization for the basis matrices, ensuring stable and effective training. Extensive experiments on image inpainting, denoising, super-resolution, and point cloud recovery validate the effectiveness of our approach. 
		% Our future work includes exploring reparameterized models for other tensor decompositions, such as Tucker decomposition and block term decomposition.
		
		% We presented a reparameterized TR decomposition for multi-dimensional data recovery. By analyzing the TR's frequency behavior, we identified the importance of high-frequency components and developed a structured reparameterization that facilitated factor learning. Extensive experiments in various tasks validated our approach. Future work includes extending the framework for Tucker and block-term decompositions.
		
		{
			\small
			\bibliographystyle{ieeenat_fullname}
			\bibliography{ref}
		}

		\clearpage
		\setcounter{page}{1}
		\setcounter{theorem}{0}
		\maketitlesupplementary
		\appendix
		
		\renewcommand{\thefigure}{S\arabic{figure}}
		\renewcommand{\thetable}{S\arabic{table}}
		\setcounter{figure}{0}
		\setcounter{table}{0}
		
		This supplementary material provides additional technical details and experimental results.
		In Section~\ref{supp_sec:proof}, we present the complete proofs of all theorems in the main paper.
		Section~\ref{supp_sec:settings} describes the model configurations and parameter settings used for the four tasks.
		Section~\ref{supp_sec:experiments} reports extended experiments, including ablation studies and practical applications.

		\section{Detaied Proofs of Theorems}\label{supp_sec:proof}
		
		\begin{theorem}
			\label{thm:spectral_preservation_supp}
			Let $\cX = \Phi(\cG^{(1)}, \dots, \cG^{(d)})$ be a TR decomposition. 
			Suppose that the mode-2 frequency components of $\cG^{(k)}$ beyond a threshold $\Omega_k$ are negligible for all $k$, i.e.,
			\begin{equation*}
				\|\big(\mathscr{F}_2[\cG^{(k)}]\big)_{:\omega_k:}\|_\infty \le \epsilon, 
				\ \forall |\omega_k| > \Omega_k, k= 1, 2, \dots, d,
			\end{equation*}
			where $\epsilon > 0$ is a small constant. Then the reconstructed tensor $\cX$ also exhibits attenuated high-frequency content along mode $k$:
			\begin{equation*}
				\|\big(\mathscr{F}_k[\cX]\big)_{:\omega_k:}\|_\infty 
				\le c_{k} \, \epsilon, \ \forall |\omega_k| > \Omega_k,  k= 1, 2,\dots, d,
			\end{equation*}
			where $c_{k}$ is a constant depending only on the magnitudes of the remaining cores $\{\cG^{(j)}\}_{j\ne k}$.
		\end{theorem}
		
		\begin{proof}
			In TR decomposition, each element of $\cX$ is expressed as
			\[
			\cX_{v_1,\dots,v_d} = \operatorname{trace}\!\left(
			\cG^{(1)}_{:v_1:}\cG^{(2)}_{:v_2:}\cdots
			\cG^{(k)}_{:v_k:}\cdots
			\cG^{(d)}_{:v_d:}
			\right).
			\]
			By the cyclic property of the trace, we can isolate the $k$-th factor:
			\[
			\cX_{v_1,\dots,v_d}
			= \operatorname{trace}\!\left(
			\mathbf{M}_{(\neg k)} \cdot \cG^{(k)}_{:v_k:}
			\right),
			\]
			where 
			\[
			\mathbf{M}_{(\neg k)} =
			\cG^{(k+1)}_{:v_{k+1}:}\cdots
			\cG^{(d)}_{:v_d:}
			\cG^{(1)}_{:v_1:}\cdots
			\cG^{(k-1)}_{:v_{k-1}:} \in \mathbb{R}^{r_{k+1}\times r_k}.
			\]
			For fixed indices $\{v_j\}_{j\ne k}$, the matrix $\mathbf{M}_{(\neg k)}$ is constant with respect to $v_k$. Expanding the trace gives
			\begin{align*}
				\cX_{v_1,\dots,v_d}
				&= \sum_{p=1}^{r_k}\sum_{q=1}^{r_{k+1}} [\mathbf{M}_{(\neg k)}]_{q,p}\, \cG^{(k)}_{p,v_k,q}.
			\end{align*}
			Applying the mode-$k$ DFT along the $k$-th dimension, we obtain
			\begin{align*}
				&\big(\mathscr{F}_k[\cX]\big)_{v_1,\dots,\omega_k,\dots,v_d}\\
				&=\sum_{v_k=1}^{n_k} \mathcal{X}_{v_1,\dots,v_k,\dots,v_d} \cdot [\mathbf{F}_k]_{v_k, \omega_k}\\
				&=\sum_{v_k=1}^{n_k} \left( \sum_{p=1}^{r_k} \sum_{q=1}^{r_{k+1}} [\mathbf{M}_{(\neg k)}]_{q,p} \cdot \mathcal{G}^{(k)}_{p,v_k,q} \right) \cdot [\mathbf{F}_k]_{v_k, \omega_k}\\
				&= \sum_{p=1}^{r_k} \sum_{q=1}^{r_{k+1}} [\mathbf{M}_{(\neg k)}]_{q,p} \cdot \left( \sum_{v_k=1}^{n_k} \mathcal{G}^{(k)}_{p,v_k,q} \cdot [\mathbf{F}_k]_{v_k, \omega_k} \right).
			\end{align*}
			Using the definition of mode-2 DFT for the TR factor $\cG^{(k)}$,
			\[
			\big(\mathscr{F}_2[\mathcal{G}^{(k)}]\big)_{p,\omega_k,q} = \sum_{v_k=1}^{n_k} \mathcal{G}^{(k)}_{p,v_k,q} \cdot [\mathbf{F}_k]_{v_k, \omega_k},
			\]
			we can rewrite the expression as
			\begin{align*}
				&\big(\mathscr{F}_k[\cX]\big)_{v_1,\dots,\omega_k,\dots,v_d}\\
				&=
				\sum_{p=1}^{r_k} \sum_{q=1}^{r_{k+1}} [\mathbf{M}_{(\neg k)}]_{q,p} \cdot \big(\mathscr{F}_2[\mathcal{G}^{(k)}]\big)_{p,\omega_k,q}.
			\end{align*}
			We now bound the magnitude of this term. Using the triangle inequality:
			\begin{align*}
				&\left| \big(\mathscr{F}_k[\cX]\big)_{v_1,\dots,\omega_k,\dots,v_d} \right| \\
				&\le \sum_{p=1}^{r_k} \sum_{q=1}^{r_{k+1}} \left| [\mathbf{M}_{(\neg k)}]_{q,p} \right| \cdot \left| \big(\mathscr{F}_2[\mathcal{G}^{(k)}]\big)_{p,\omega_k,q} \right|.
			\end{align*}
			By the assumption on $\cG^{(k)}$, for all $|\omega_k| > \Omega_k$, we have
			\[
			\left| \big(\mathscr{F}_2[\cG^{(k)}]\big)_{p,\omega_k,q} \right| \le \epsilon.
			\]
			Substituting this bound:
			\begin{align*}
				\left| \big(\mathscr{F}_k[\cX]\big)_{v_1,\dots,\omega_k,\dots,v_d} \right| 
				&\le \epsilon \cdot \left( \sum_{p=1}^{r_k} \sum_{q=1}^{r_{k+1}} \left| [\mathbf{M}_{(\neg k)}]_{q,p} \right| \right).
			\end{align*}
			Defining the constant $c_k$ as the maximum possible sum of magnitudes of the product of the remaining cores:
			\[
			c_k := \max_{\{v_j\}_{j\ne k}} \sum_{p=1}^{r_k} \sum_{q=1}^{r_{k+1}} \left| [\mathbf{M}_{(\neg k)}]_{q,p} \right|,
			\]
			we obtain the final bound for the infinity norm over the spatial indices $\{v_j\}_{j\ne k}$:
			\[
			\|\big(\mathscr{F}_k[\cX]\big)_{:\omega_k:}\|_\infty \le c_k \, \epsilon, \quad \forall |\omega_k| > \Omega_k.
			\]
			This completes the proof.
		\end{proof}
		
		% \begin{theorem}
			% \label{thm:highfreq_update}
			% Let $\cX^{t} = \Phi(\cG^{(1)}_t,\dots,\cG^{(d)}_t)$ be the TR reconstruction at iteration $t$, and let $\cG^{(k)}_{t+1} = \cG^{(k)}_t + \Delta_{\cG^{(k)}}$ be a update to the $k$-th factor.  
			% Assume that the update is small in high-frequency components:
			% \[
			% |\hat{\Delta}_{\cG^{(k)}_{p,\omega_k,q}}| \le \delta, \quad |\omega_k| > \Omega_k.
			% \]  
			% Then the high-frequency components of the updated reconstruction satisfy
			% \begin{equation}
				% |\hat{\mathcal{X}}^{t+1}_{i_1,\dots,i_{k-1},\omega_k,i_{k+1},\dots,i_d}| 
				% \le C_{(\neg k)} (\epsilon + \delta), \quad |\omega_k| > \Omega_k,
				% \end{equation}
			% where $\epsilon$ bounds the high-frequency components of $\mathcal{G}^{(k,t)}$ as in Theorem~\ref{thm:spectral_preservation}, and $C_{(\neg k)}$ depends only on the magnitudes of the other cores $\{\mathcal{G}^{(j)}\}_{j\ne k}$.
			% \end{theorem}
		
		% \begin{proof}
			
			% \end{proof}
		
		\begin{theorem}\label{thm:rep_supp}
			Consider a TR factor $\cG$ reparameterized as $\cG = \cC \times_3 \Bb$, where $\cC \in \mathbb{R}^{r_1 \times n \times R_2}$ is a trainable tensor and $\Bb \in \mathbb{R}^{r_2 \times R_2}$ is a fixed
			% for some 
			basis. Let $\mathsf{L}(\omega)$ be the loss associated with frequency $\omega$. For any $\omega_{\rm high} > \omega_{\rm low} > 0$, given any $\epsilon \geq 0$, and fixed indices $p, q$ with $1 \le p \le r_1$ and $1 \le q \le n$, 
			% we have 
			there exists a matrix $\Bb$ such that for all $s=1,\dots,R_2$,
			\begin{equation*}
				\resizebox{\linewidth}{!}{$\displaystyle\left|
					\tfrac{\partial \mathsf{L}(\omega_{\rm high})}{\partial \cC_{p q s}}
					\Big/
					\tfrac{\partial \mathsf{L}(\omega_{\rm low})}{\partial \cC_{p q s}}
					\right|\ge
					\max_{j=1,\dots,r_2}
					\left|
					\tfrac{\partial \mathsf{L}(\omega_{\rm high}))}{\partial \cG_{p q j}}
					\Big/
					\tfrac{\partial \mathsf{L}(\omega_{\rm low})}{\partial \cG_{p q j}}
					\right|
					- \epsilon.$}
			\end{equation*}
		\end{theorem} 
		
		\begin{proof}
			First, from the reparameterization $\cG = \cC \times_3 \Bb$, we have
			\begin{equation*}
				\cG_{p q j} = \sum_{s=1}^{R_2} \cC_{p q s} \Bb_{j s},
			\end{equation*}
			Viewing $\cG_{p q 1}, \ldots, \cG_{p q r_2}$ as latent variables dependent on $\cC_{p q s}$, by the chain rule we obtain
			\begin{equation*}
				\frac{\partial \mathsf{L}(\omega)}{\partial \cC_{p q s}} = \sum_{t=1}^{r_2} \frac{\partial \mathsf{L}(\omega)}{\partial \cG_{p q t}}\Bb_{t s}.
			\end{equation*}
			
			Next, given two frequencies $\omega_{\rm high} > \omega_{\rm low} > 0$ and fixed fiber indices $p, q$, define
			\begin{equation*}
				\tau = \argmax_{j}
				\left|
				\frac{\partial \mathsf{L}(\omega_{\rm high}))}{\partial \cG_{p q j}}
				\Big/
				\frac{\partial \mathsf{L}(\omega_{\rm low})}{\partial \cG_{p q j}}
				\right|.
			\end{equation*}
			By the triangle inequality, we have
			\begin{align*}
				&\left|\frac{\partial \mathsf{L}(\omega_{\rm high})}{\partial \cC_{p q s}}\Big/\frac{\partial \mathsf{L}(\omega_{\rm low})}{\partial \cC_{p q s}}\right| \\
				&= \left|\frac{\frac{\partial \mathsf{L}(\omega_{\rm high})}{\partial \cG_{p q \tau}}\Bb_{\tau s} + \sum_{t\neq \tau} \frac{\partial \mathsf{L}(\omega_{\rm high})}{\partial \cG_{p q t}}\Bb_{t s}}{\frac{\partial \mathsf{L}(\omega_{\rm low})}{\partial \cG_{p q \tau}}\Bb_{\tau s} + \sum_{t\neq \tau} \frac{\partial \mathsf{L}(\omega_{\rm low})}{\partial \cG_{p q t}}\Bb_{t s}}\right| \\
				&\ge \frac{\left|\frac{\partial \mathsf{L}(\omega_{\rm high})}{\partial \cG_{p q \tau}}\Bb_{\tau s}\right| - \left|\sum_{t\neq \tau} \frac{\partial \mathsf{L}(\omega_{\rm high})}{\partial \cG_{p q t}}\Bb_{t s}\right|}{\left|\frac{\partial \mathsf{L}(\omega_{\rm low})}{\partial \cG_{p q \tau}}\Bb_{\tau s}\right| + \left|\sum_{t\neq \tau} \frac{\partial \mathsf{L}(\omega_{\rm low})}{\partial \cG_{p q t}}\Bb_{t s}\right|}.
			\end{align*}
			
			For the elements of $\Bb$, construct $\Bb_{t s}$ such that for $s=1, \ldots, R_2$, $|\Bb_{t s}| \le \alpha$ when $t \neq \tau$ and $\Bb_{\tau s} = 1$, where $\alpha > 0$ is a positive upper bound. Define
			\begin{equation*}
				M_1 = \sum_{t \neq \tau} \left| \frac{\partial\mathsf{L}(\omega_{\rm high})}{\partial\cG_{p q t}}\right|, \quad M_2 = \sum_{t \neq \tau} \left| \frac{\partial\mathsf{L}(\omega_{\rm low})}{\partial\cG_{p q t}}\right|.
			\end{equation*}
			Without loss of generality, for any $\epsilon\ge 0$,
			% \begin{equation*}
				%     0 \leq \epsilon \leq \frac{\left| \frac{\partial \mathsf{L}(\omega_{\rm high})}{\partial \cG_{p q \tau}} \right|}{\left| \frac{\partial \mathsf{L}(\omega_{\rm low})}{\partial \cG_{p q \tau}} \right|},
				% \end{equation*}
			choose
			\begin{equation*}
				\resizebox{\linewidth}{!}{$
					\alpha \le \min \left\{ \displaystyle\frac{ \left| \frac{\partial \mathsf{L}(\omega_{\rm low})}{\partial \cG_{p q \tau}} \right| \epsilon }{ M_1 + M_2 \left| \frac{\partial \mathsf{L}(\omega_{\rm high})}{\partial \cG_{p q \tau}}/ \frac{\partial \mathsf{L}(\omega_{\rm low})}{\partial \cG_{p q \tau}} \right| - M_2 \epsilon }, \frac{1}{M_1}\left| \frac{\partial \mathsf{L}(\omega_{\rm high})}{\partial \cG_{p q \tau}} \right| \right\}.$}
			\end{equation*}
			From the upper bound on $\alpha$, it holds that
			\begin{equation*}
				\left|\sum_{t\neq \tau} \frac{\partial \mathsf{L}(\omega_{\rm high})}{\partial \cG_{p q t}}\Bb_{t s}\right| \leq \alpha M_1 \leq \left|\frac{\partial \mathsf{L}(\omega_{\rm high})}{\partial \cG_{p q \tau}}\right|,
			\end{equation*}
			which implies
			\begin{equation*}
				\left|\frac{\partial \mathsf{L}(\omega_{\rm high})}{\partial \cG_{p q \tau}}\right| - \left|\sum_{t\neq \tau} \frac{\partial \mathsf{L}(\omega_{\rm high})}{\partial \cG_{p q t}}\Bb_{t s}\right| \geq 0.
			\end{equation*}
			Thus, for fixed $p, q$ and for $s=1, \ldots, R_2$, we get
			\begin{align*}
				&\left|\frac{\partial \mathsf{L}(\omega_{\rm high})}{\partial \cC_{p q s}}\Big/\frac{\partial \mathsf{L}(\omega_{\rm low})}{\partial \cC_{p q s}}\right| \\
				&\ge \frac{ \left|\frac{\partial \mathsf{L}(\omega_{\rm high})}{\partial \cG_{p q \tau}}\right| - \alpha M_1 }{ \left|\frac{\partial \mathsf{L}(\omega_{\rm low})}{\partial \cG_{p q \tau}}\right| + \alpha M_2 } \\
				% &\ge \frac{\left(\left|\frac{\partial \mathsf{L}(\omega_{\rm high})}{\partial \cG_{p q \tau}}\right|-\epsilon \left|\frac{\partial \mathsf{L}(\omega_{\rm low})}{\partial \cG_{p q \tau}}\right|\right)\left(M_1\left|\frac{\partial \mathsf{L}(\omega_{\rm low})}{\partial \cG_{p q \tau}}\right|+ M_2\left|\frac{\partial \mathsf{L}(\omega_{\rm high})}{\partial \cG_{p q \tau}}\right|\right)}{\left|\frac{\partial \mathsf{L}(\omega_{\rm low})}{\partial \cG_{p q \tau}}\right|\left(M_1 \left|\frac{\partial \mathsf{L}(\omega_{\rm low})}{\partial \cG_{p q \tau}}\right|+M_2 \left|\frac{\partial \mathsf{L}(\omega_{\rm high})}{\partial \cG_{p q \tau}}\right|\right)} \\
				&\ge \frac{\left|\frac{\partial \mathsf{L}(\omega_{\rm high})}{\partial \cG_{p q \tau}}\right|-\epsilon \left|\frac{\partial \mathsf{L}(\omega_{\rm low})}{\partial \cG_{p q \tau}}\right|}{\left|\frac{\partial \mathsf{L}(\omega_{\rm low})}{\partial \cG_{p q \tau}}\right|} \\
				&= \max_{j=1,\dots,r_2}\left|\frac{\partial \mathsf{L}(\omega_{\rm high}))}{\partial \cG_{p q j}}\Big/\frac{\partial \mathsf{L}(\omega_{\rm low})}{\partial \cG_{p q j}}\right|- \epsilon.
			\end{align*}
			The second inequality is obtained by applying the first upper bound of $\alpha$. This completes the proof.
		\end{proof}
		
		\begin{remark}
			We note that the proof of Theorem~\ref{thm:rep_supp} provides an existence guarantee, demonstrating that a suitable choice of $\Bb$ can in principle amplify high-frequency gradients. However, the specific construction in the proof is not employed in practice, as the optimal index $\tau$ is dynamic and depends on the data, making it infeasible to pre-specify. Instead, we adopt the principled Xavier-style initialization described in Theorem~\ref{thm:basis_init_supp}, which preserves variance in both forward and backward passes. This provides a stable starting point for optimization while allowing the model to learn from a fully expressive random state.
		\end{remark}
		
		\begin{theorem}\label{thm:basis_init_supp}
			% Let $\Bb^{(k)} \in \mathbb{R}^{r_{k+1} \times R_{k+1}}$ be the fixed basis used in the reparameterized TR factor. 
			Suppose the entries of $\Bb^{(k)}$ in the reparameterized TR factor  are sampled independently from a uniform distribution as 
			% , the entries should be drawn from
			\begin{equation*}
				\Bb^{(k)}_{ij} \sim \mathcal{U}\left(-\sqrt{\tfrac{6}{r_{k+1}+R_{k+1}}}, \ \sqrt{\tfrac{6}{r_{k+1}+R_{k+1}}}\right), \forall i,j,k, 
			\end{equation*}
			then the variances in both the forward and backward passes are preserved.  
		\end{theorem}
		
		\begin{proof}
			We derive the variance of the initialization parameter $\sigma_{\Bb}^2$ by enforcing variance preservation in both the forward and backward passes.
			
			\noindent\textbf{Forward pass.} Each element of $\cG^{(k)}$ is computed as
			\begin{equation*}
				\cG_{pqj}^{(k)}=\sum_{s=1}^{R_{k+1}}\cC^{(k)}_{pqs} \Bb_{js}^{(k)}.
			\end{equation*}
			Since both $\cC^{(k)}$ and $\Bb^{(k)}$ are zero-mean and independent, the variance of $\cG^{(k)}_{pqj}$ is
			\begin{equation*}
				\operatorname{Var}(\cG_{pgj}^{(k)})=R_{k+1} \sigma_\cC^2 \sigma_\Bb^2.
			\end{equation*}
			To maintain signal variance, we require that
			\begin{equation*}
				\mathrm{Var}(\cG_{pgj}^{(k)}) = \mathrm{Var}(\cC_{pqj}^{(k)})=\sigma_{\cC}^2,
			\end{equation*}
			which gives
			\begin{equation*}
				\sigma_\Bb^2 = \frac{1}{\hat{R}_{i+1}}.
			\end{equation*}
			
			\noindent\textbf{Backward pass.} During backpropagation, the gradient of the loss function with respect to $\cC^{(k)}_{pqj}$ is
			\begin{equation*}
				\frac{\partial \mathsf{L}}{\partial \cC_{pqj}^{(k)}} = \sum_{t=1}^{r_{k+1}} \frac{\partial \mathsf{L}}{\partial \cG_{pqt}^{(k)}} \frac{\partial \cG_{pqt}^{(k)}}{\partial \cC_{pqj}^{(k)}} = \sum_{t=1}^{r_{k+1}} \frac{\partial \mathsf{L}}{\partial \cG_{pqt}^{(k)}} \Bb_{ts}^{(k)}.
			\end{equation*}
			Assuming i.i.d. gradients with $\operatorname{Var}\big(\frac{\partial \mathsf{L}}{\partial \cG_{pqt}^{(k)}}\big)=\sigma_{\nabla \cG}^2$, we obtain
			\begin{equation*}
				\operatorname{Var}\left(\frac{\partial \mathsf{L}}{\partial \cC_{pqj}^{(k)}}\right)=R_{i+1} \sigma_{\nabla \cG}^2 \sigma_\Bb^2.
			\end{equation*}
			To preserve the variance of the backpropagated gradients, we require
			\begin{equation*}
				\sigma_\Bb^2=\frac{1}{R_{i+1}}.
			\end{equation*}
			Following the principle of Xavier Glorot initialization~\cite{glorot2010understanding}, which balances the forward and backward variance constraints, we take the harmonic mean of the two bounds, yielding
			\begin{equation*}
				\sigma_{\Bb}^2=\frac{2}{r_{k+1}+R_{k+1}}.
			\end{equation*}
			Since a zero-mean uniform distribution $\cU(-a,a)$ has variance $\operatorname{Var} = a^2/3$, setting $a = \sqrt{6/(r_{k+1}+R_{k+1})}$ makes its variance exactly equal to $\sigma_{\Bb}^2$, yielding the final form
			\begin{equation*}
				\Bb^{(k)}_{ij} \sim \mathcal{U}\left(-\sqrt{\tfrac{6}{r_{k+1}+R_{k+1}}},\ \sqrt{\tfrac{6}{r_{k+1}+R_{k+1}}}\right), \quad \forall i,j,k,
			\end{equation*}
			which completes the proof.
		\end{proof}

		\begin{remark}
			The analysis in Theorem~\ref{thm:basis_init_supp} is not restricted to Xavier initialization and naturally extends to other variance-preserving schemes. In particular, enforcing forward-pass variance preservation under Kaiming initialization yields $\mathbf{B}^{(k)}_{ij} \sim \mathcal{U}\left(-\sqrt{\tfrac{6}{R_{k+1}}},\ \sqrt{\tfrac{6}{R_{k+1}}}\right)$. Numerical results are provided in Section~\ref{supp_sec:experiments}.
		\end{remark}

		\begin{theorem}\label{thm:lipschitz_supp}
			Let $g_{\phi}(\vb): \mathbb{R}^d \to \mathbb{R}$ be defined as 
			$$g_{\phi}(\vb) = \Phi(g_{\phi_1}(v_1)\times_3 \Bb^{(1)}, \dots, g_{\phi_d}(v_d)\times_3 \Bb^{(d)}),$$
			where $\vb = [v_1, \dots, v_d]^{\mathrm{T}} \in \mathbb{R}^d$ denotes the coordinate vector, and each $\Bb^{(k)} \in \mathbb{R}^{r_{k+1} \times R_{k+1}}$ is a fixed basis.
			Assume for each mode $k = 1, \dots, d$, the following conditions hold:
			\begin{itemize}
				\item $g_{\phi_k}(\cdot)$ is an $L_k$-layer MLP with activation function $\sigma(\cdot)$ that is $\kappa$-Lipschitz continuous;
				\item the spectral norm of each weight matrix in $g_{\phi_k}(\cdot)$ is bounded by $\eta$;
				\item the output is bounded: $\sup_{v_k} \|g_{\phi_k}(v_k)\|_F \le C_k < \infty$.
			\end{itemize}
			Then, $g_{\phi}(\vb)$ is globally Lipschitz continuous, i.e.,
			\begin{equation*}
				|g_{\phi}(\vb) - g_{\phi}(\vb')| \le \delta \|\vb - \vb'\|_2,
			\end{equation*}
			where the global Lipschitz constant $\delta = \sqrt{\sum_{k=1}^d \delta_k^2}$, and $\delta_k = (\kappa \eta)^{L_k} \cdot \|\mathbf{B}_k\|_2 \cdot \left( \prod_{j \ne k} C_j \right)$.
		\end{theorem}
		
		\begin{proof}
			For the $i$-th layer of the $k$-th MLP, defined as 
			\begin{equation*}
				g_i^{(k)}(\mathbf{x}) = \sigma(\mathbf{W}_i \mathbf{x} + \mathbf{b}_i),
			\end{equation*}
			the Lipschitz constant is bounded by the product of the activation’s constant and the affine transformation’s constant:
			$$L(g_i^{(k)}) \le L(\sigma) \cdot L(\mathbf{W}_i \mathbf{x} + \mathbf{b}_i) \le \kappa \cdot \|\mathbf{W}_i\|_2 \le \kappa \eta.$$
			Hence, for an $L_k$-layer MLP $g_{\phi_k}$, the overall Lipschitz constant satisfies
			$$L_g^{(k)} \le \prod_{i=1}^{L_k} L(g_i^{(k)}) \le (\kappa \eta)^{L_k}.$$
			We then have 
			\begin{equation*}
				\begin{aligned}
					&\norm{\cG_{:v_k:}^{(k)}-\cG_{:v'_k:}^{(k)}}_F\\
					&=\|g_{\phi_k}(v_k)\times_3 \Bb^{(k)} - g_{\phi_k}(v'_k)\times_3 \Bb^{(k)})\|_F\\
					&= \|(g_{\phi_k}(v_k) - g_{\phi_k}(v'_k)) \times_3 \mathbf{B}_k\|_F \\
					&\le \|g_{\phi_k}(v_k) - g_{\phi_k}(v'_k)\|_F \cdot \|\mathbf{B}_k\|_2 \\
					&\le (\kappa \eta)^{L_k} \|\mathbf{B}_k\|_2 \cdot |v_k-v'_k|.
				\end{aligned}
			\end{equation*}
			Let $\Gb^{(k)}_{v_k}=\cG^{(k)}_{:v_k:}\in\mathbb{R}^{r_k\times r_{k+1}}$, $\Mb_{(\neg k)}=\Gb^{(k+1)}_{v_{k+1}}\cdots\Gb^{(d)}_{v_d}\Gb^{(1)}_{v_1}\cdots\Gb^{(k-1)}_{v_{k-1}}\in \mathbb{R}^{r_{k+1}\times r_k}$, and denote by $\vb_k' = (v_1, \dots, v_{k-1}, v_k', v_{k+1}, \dots, v_d)$ the coordinate vector differing from $\vb$ only at the $k$-th entry. Then,
			\begin{equation*}
				\begin{aligned}
					&|g_{\phi}(\vb) - g_{\phi}(\vb'_k)|\\
					&= \left|\operatorname{trace}\left( \mathbf{M}_{(\neg k)} \cdot \Gb^{(k)}_{v_k} \right)- \operatorname{trace}\left( \mathbf{M}_{(\neg k)} \cdot \Gb^{(k)}_{v'_k}\right)\right|\\
					&\le \|\Mb\|_F \cdot \norm{\Gb^{(k)}_{v_k}-\Gb^{(k)}_{v'_k}}_F \\
					&\le \prod_{j\ne k}\norm{\Gb_{v_j}^{(j)}}_F\cdot \norm{\cG_{:v_k:}^{(k)}-\cG_{:v'_k:}^{(k)}}_F \\
					&\le (\kappa \eta)^{L_k}\cdot\|\mathbf{B}_k\|_2\cdot \prod_{j\ne k}C_j \cdot |v_k-v'_k|.
				\end{aligned}
			\end{equation*}
			Aggregating across all dimensions yields
			\begin{equation*}
				\begin{aligned}
					&|g_{\phi}(\vb) - g_{\phi}(\vb')| \\
					&\le \sum_{k=1}^d |g_{\phi}(\vb) - g_{\phi}(\vb'_k)|\\
					&\le \sum_{k=1}^d (\kappa \eta)^{L_k}\cdot\|\mathbf{B}_k\|_2\cdot \prod_{j\ne k}C_j\cdot|v_k - v'_k|.
				\end{aligned}
			\end{equation*}
			Finally, applying the Cauchy–Schwarz inequality gives
			\begin{equation*}
				\begin{aligned}
					&|g_{\phi}(\vb) - g_{\phi}(\vb')| \\
					&\le \sqrt{\sum_{k=1}^d\left((\kappa \eta)^{L_k} \cdot \|\mathbf{B}_k\|_2 \cdot \prod_{j \ne k} C_j\right)^2} \cdot \norm{\vb-\vb'}_2,
				\end{aligned}
			\end{equation*}
			which completes the proof.
		\end{proof}
		
		\section{Task-Specific Losses and Settings}\label{supp_sec:settings}
		
		In this section, we provide the loss formulations and hyperparameter settings for the four tasks studied in our experiments. 
		Recall that
		\begin{equation*}
			g_{\phi}(\{\vb_k\}_{k=1}^d) = \Phi(g_{\phi_1}(\vb_1)\times_3 \Bb^{(1)}, \dots, g_{\phi_d}(\vb_d)\times_3 \Bb^{(d)})
		\end{equation*}
		denotes the RepTRFD reconstruction with learnable parameters $\phi$ and fixed bases $\{\Bb^{(k)}\}_{k=1}^d$. The main hyperparameters in our model include the TR rank $[r_1, \dots, r_d]$, the latent tensor expansion factor $\beta$, the frequency parameter $\omega_0$ for the sinusoidal mapping, and the number of layers in each branch network. Each task is formulated as a single optimization problem minimizing a combination of a data fidelity term and an optional regularization term.
		
		\noindent{\bf Image and Video Inpainting, }which aims to reconstruct missing regions from partially observed visual data:
		\begin{equation*}
			\begin{aligned}
				\min_{\phi} \| \mathscr{P}_{\Omega}(g_{\phi}(\{\vb_k\}_{k=1}^d)-\cO)\|_F^2 + \gamma_1 \|g_{\phi}(\{\vb_k\}_{k=1}^d)\|_{\rm TV} \\
				+ \gamma_2 \|g_{\phi}(\{\vb_k\}_{k=1}^d)\|_{\rm SSTV},
			\end{aligned}
		\end{equation*}
		where $\cO$ is the observed incomplete data, and $\mathscr{P}_{\Omega}(\cdot)$ denotes projection onto observed entries. The total variation (TV) is defined along spatial dimensions as $\|\cX\|_{\rm TV} = \|\nabla_{v_1} \cX\|_1 + \|\nabla_{v_2} \cX\|_1$, and the spatial-spectral TV (SSTV) involves the spectral dimension, defined as $\|\cX\|_{\rm SSTV} = \|\nabla_{v_1} \nabla_{v_3} \cX\|_1 + \|\nabla_{v_2} \nabla_{v_3} \cX\|_1$. For all modes $k$, $r_k=20$ and $\beta=10$. Frequency parameter $\omega_0$ is set to 90 for color images and videos, and 120 for MSIs and HSIs. Branch layers $[L_1,L_2,L_3]$ are set to $[1,1,2]$ for color images and videos, and $[1,1,1]$ for MSIs and HSIs. Regularization parameters are set as follows: $\gamma_1=\gamma_2=5\times 10^{-5}$ for color images; $\gamma_1=5\times 10^{-5}$ and $\gamma_2=5\times 10^{-4}$ for MSIs and videos; $\gamma_1=5\times 10^{-6}$ and $\gamma_2=5\times 10^{-5}$ for HSIs.
		
		\noindent{\bf Image Denoising, }which aims to recover clean images from noisy observations:
		\begin{equation*}
			\begin{aligned}
				\min_{\phi} \| g_{\phi}(\{\vb_k\}_{k=1}^d)-\cO\|_F^2 + \gamma_1 \|g_{\phi}(\{\vb_k\}_{k=1}^d)\|_{\rm TV} \\
				+ \gamma_2 \|g_{\phi}(\{\vb_k\}_{k=1}^d)\|_{\rm SSTV},
			\end{aligned}
		\end{equation*}
		where $\cO$ is the observed noisy data. For all modes $k$, $r_k=16$ and $\beta=5$. Frequency parameter $\omega_0$ is set to 120, and branch layers $[L_1,L_2,L_3]$ are $[1,1,2]$. Regularization weights are set as $\gamma_1=\gamma_2 = 10^{-4}$ for MSIs and $\gamma_1=\gamma_2 =10^{-5}$ for HSIs.

		\noindent{\bf Image Super-Resolution, }which seeks to reconstruct high-resolution images from their low-resolution counterparts:
		\begin{equation*}
			\min_{\phi} \| \operatorname{Down}(g_{\phi}(\{\vb_k\}_{k=1}^d)) - \cO\|_F^2 + \gamma_1 \|g_{\phi}(\{\vb_k\}_{k=1}^d)\|_{\rm TV},
		\end{equation*}
		where $\operatorname{Down}(\cdot)$ is the downsampling operator and $\cO$ is the low-resolution observation. For all modes $k$, $r_k=20$ and $\beta=10$. Frequency parameter $\omega_0$ is set to 90, branch layers $[L_1,L_2,L_3]$ are $[1,1,2]$, and $\gamma_1=5\times 10^{-5}$.
		
		\noindent{\bf Point Cloud Recovery, }which aims to infer continuous point-wise attributes from partially observed point cloud data:
		\begin{equation*}
			\min_{\phi} \sum_{k=1}^{N} (g_{\phi}(\vb_k)-\Ob_{k,5})^2,
		\end{equation*}
		where the observed data $\Ob \in \mathbb{R}^{N \times 5}$ contains $N$ points, with each row encoding spatial coordinates $(x, y, z)$, color information $c$, and a scalar value $s$.  
		Here, $\vb_k = (x_k, y_k, z_k, c_k)$ is the input coordinate-color vector of the $k$-th point, and $\Ob_{k,5}$ is the corresponding observed scalar value. For all modes $k$, $r_k=20$ and $\beta=3$. Frequency parameter $\omega_0$ is set to 240, and all branches use a single layer.

		\section{Supplementary Experimental Results}\label{supp_sec:experiments}
		
		\subsection{Additional Ablation Studies}
		
		\noindent{\bf Effectiveness of the Fixed Basis Strategy. }To validate the necessity of a fixed basis, we compare our proposed strategy with a variant where the basis matrices are treated as learnable parameters. 
		As illustrated in Fig.~\ref{fig:basis_comparison}, although both methods exhibit similar convergence speeds in the early stage, they diverge significantly in later iterations. 
		The learnable basis model reaches a lower peak PSNR and subsequently suffers from performance degradation. 
		This phenomenon suggests that allowing basis matrices to vary introduces excessive degrees of freedom, causing the model to overfit high-frequency noise or drift away from the variance-preserving initialization. 
		In contrast, our fixed basis strategy acts as a structural regularizer, maintaining improved stability and consistently higher reconstruction quality throughout the training process.
		
		\begin{figure}[htbp]
			\centering
			\resizebox{\linewidth}{!}{
				\begin{tabular}{cc}
					\includegraphics[width=0.5\linewidth]{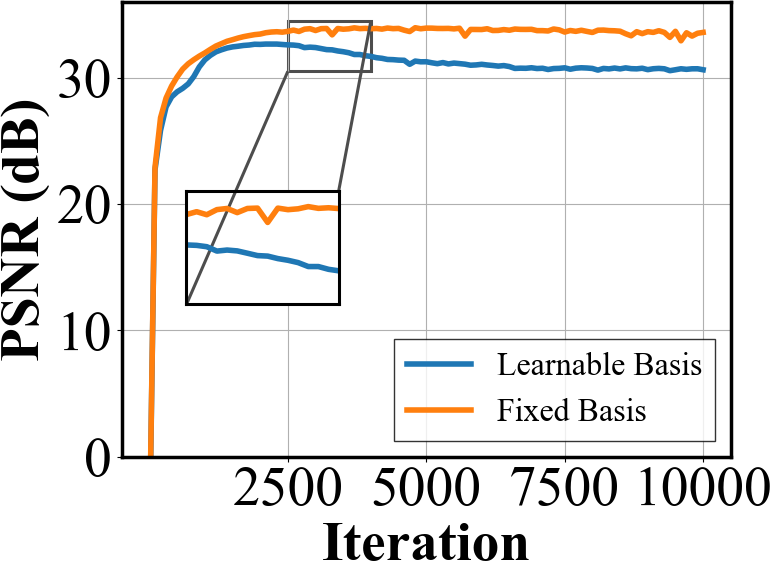} & \includegraphics[width=0.5\linewidth]{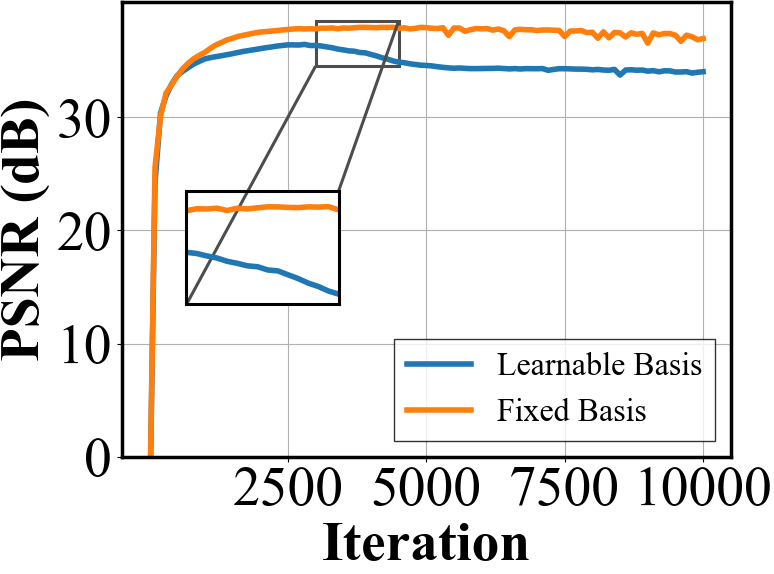} 
			\end{tabular}}
			% \vspace{-4mm}
			\caption{Impact of the fixed basis strategy on denoising performance for the MSIs \textit{Toy} (left) and \textit{Face} (right) with $\text{SD}=0.2$.}
			\label{fig:basis_comparison}
			\vspace{-1em}
		\end{figure}
		
		\noindent{\bf Hyperparameter Sensitivity Analysis. }We investigate the sensitivity of two key hyperparameters in our framework, the TR rank and the frequency scaling parameter $\omega_0$. Experiments are conducted on the MSI datasets \textit{Toy} and \textit{Face} for denoising with $\text{SD} = 0.2$. Fig.~\ref{fig:hyperparameter_sensitivity} presents the variation of PSNR with different values of these hyperparameters. The results indicate that our method is relatively robust to the choice of both the TR rank and $\omega_0$, maintaining high reconstruction quality across a wide range of settings. In addition, we examine the interaction between the regularization parameter $\beta$ and the frequency scaling parameter $\omega_0$. 
		As shown in Fig.~\ref{fig:sensitivity_beta_omega}, RepTRFD maintains stable performance over a broad range of parameter combinations.
		
		\begin{figure}[htbp]
			\centering
			\resizebox{\linewidth}{!}{
				\begin{tabular}{cc}
					% TR rank
					\begin{tikzpicture}
						\begin{axis}[
							width=0.48\textwidth,
							height=6cm,
							xlabel={TR rank},
							ylabel={PSNR},
							xlabel near ticks, 
							ylabel near ticks, 
							label style={font=\LARGE}, 
							ticklabel style={font=\Large},
							grid=both,
							ymin=25, ymax=40,
							xtick={4,8,12,16,20,24,28},
							ytick={26,28,...,40},
							legend style={at={(0.95,0.05)}, anchor=south east, font=\Large}
							]
							\addplot[color=blue, mark=* , thick] coordinates {
								(4,27.51)(8,32.13)(12,33.23)(16,33.93)(20,33.81)(24,33.40)(28,33.11)
							};
							\addlegendentry{Toy}
							\addplot[color=red, mark=square*, thick] coordinates {
								(4,32.81)(8,35.81)(12,37.42)(16,37.89)(20,37.58)(24,37.19)(28,36.82)
							};
							\addlegendentry{Face}
						\end{axis}
					\end{tikzpicture} &
					% omega_0
					\begin{tikzpicture}
						\begin{axis}[
							width=0.48\textwidth,
							height=6cm,
							xlabel={$\omega_0$},
							ylabel={PSNR},
							xlabel near ticks, 
							ylabel near ticks, 
							label style={font=\LARGE}, 
							ticklabel style={font=\Large},
							grid=both,
							ymin=25, ymax=40,
							xtick={30,60,90,120,150,180,210},
							ytick={26,28,...,40},
							legend style={at={(0.95,0.05)}, anchor=south east, font=\Large}
							]
							\addplot[color=blue, mark=* , thick] coordinates {
								(30,29.47)(60,32.03)(90,33.31)(120,33.93)(150,33.82)(180,34.01)(210,33.93)
							};
							\addlegendentry{Toy}
							\addplot[color=red, mark=square*, thick] coordinates {
								(30,34.17)(60,36.72)(90,37.15)(120,37.89)(150,37.86)(180,37.82)(210,37.83)
							};
							\addlegendentry{Face}
						\end{axis}
					\end{tikzpicture} \\
			\end{tabular}}
			\caption{PSNR variation with TR rank and frequency scaling $\omega_0$ for MSI denoising on \textit{Toy} (left) and \textit{Face} (right) datasets with $\text{SD}=0.2$.}
			\label{fig:hyperparameter_sensitivity}
			\vspace{-1em}
		\end{figure}
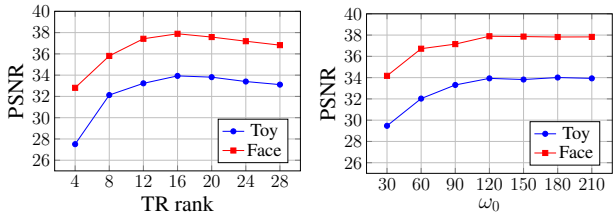

		\begin{figure}[htbp]
			\centering
			\setlength{\tabcolsep}{0pt}
			\resizebox{\linewidth}{!}{
				\begin{tabular}{cc}
					\begin{tikzpicture}
						\begin{axis}[
							width=0.75\linewidth,
							height=0.75\linewidth,
							scale only axis,
							xlabel={$\beta$},
							ylabel={$\omega_0$},
							zlabel={PSNR},
							zlabel style={yshift=-3pt, xshift=-8pt,font=\Large},
							label style={font=\Large},
							ticklabel style={font=\large},
							grid=major,
							view={-35}{45},
							colormap/viridis,
							colorbar,
							colorbar style={
								width=2mm,
								height=0.4\linewidth,
								yticklabel style={font=\large},
								at={(1.1,0.5)}, anchor=west,
							},
							xtick={0,1,2,3,4,5,6},
							xticklabels={1,3,5,7,10,15,20},
							ytick={0,1,2,3,4,5,6},
							yticklabels={30,60,90,120,150,180,210},
							zmin=28,
							zmax=34,
							ztick={28,30,32,34},
							mesh/rows=7,
							mesh/cols=7,
							xmin=-0.3, xmax=6.3,
							ymin=-0.3, ymax=6.3,
							]
							\addplot3[
							surf,
							shader=faceted,
							faceted color=black
							] coordinates {
								% Row 1: beta_k = 1
								(0,0,30.42)(0,1,32.07)(0,2,32.09)(0,3,31.38)(0,4,30.51)(0,5,29.16)(0,6,28.21)
								% Row 2: beta_k = 3
								(1,0,31.54)(1,1,32.94)(1,2,33.18)(1,3,33.10)(1,4,32.50)(1,5,31.68)(1,6,31.06)
								% Row 3: beta_k = 5
								(2,0,31.61)(2,1,33.03)(2,2,33.47)(2,3,33.48)(2,4,32.69)(2,5,31.74)(2,6,31.39)
								% Row 4: beta_k = 7
								(3,0,31.70)(3,1,33.03)(3,2,33.52)(3,3,33.64)(3,4,32.81)(3,5,31.83)(3,6,31.49)
								% Row 5: beta_k = 10
								(4,0,31.72)(4,1,32.98)(4,2,33.61)(4,3,33.22)(4,4,32.86)(4,5,31.89)(4,6,31.52)
								% Row 5: beta_k = 15
								(5,0,31.55)(5,1,33.14)(5,2,33.60)(5,3,33.20)(5,4,32.83)(5,5,32.06)(5,6,31.57)
								% Row 5: beta_k = 20
								(6,0,31.35)(6,1,32.87)(6,2,33.17)(6,3,33.10)(6,4,32.59)(6,5,32.09)(6,6,31.58)
							};
						\end{axis}
					\end{tikzpicture}
					& \begin{tikzpicture}
						\begin{axis}[
							width=0.75\linewidth,
							height=0.75\linewidth,
							scale only axis,
							xlabel={$\beta$},
							ylabel={$\omega_0$},
							zlabel={PSNR},
							zlabel style={yshift=-3pt, xshift=-8pt,font=\Large},
							label style={font=\Large},
							ticklabel style={font=\large},
							grid=major,
							view={-35}{45},
							colormap/viridis,
							colorbar,
							colorbar style={
								width=2mm,
								height=0.4\linewidth,
								yticklabel style={font=\large},
								at={(1.1,0.5)}, anchor=west,
							},
							xtick={0,1,2,3,4,5,6},
							xticklabels={1,3,5,7,10,15,20},
							ytick={0,1,2,3,4,5,6},
							yticklabels={30,60,90,120,150,180,210},
							zmin=28,
							zmax=34,
							ztick={28,30,32,34},
							mesh/rows=7,
							mesh/cols=7,
							xmin=-0.3, xmax=6.3,
							ymin=-0.3, ymax=6.3,
							]
							\addplot3[
							surf,
							shader=faceted,
							faceted color=black
							] coordinates {
								% Row 1: beta_k = 1
								(0,0,30.06)(0,1,31.31)(0,2,31.47)(0,3,31.61)(0,4,31.53)(0,5,31.05)(0,6,30.09)
								% Row 2: beta_k = 3
								(1,0,30.64)(1,1,32.82)(1,2,33.25)(1,3,33.67)(1,4,33.58)(1,5,33.56)(1,6,33.18)
								% Row 3: beta_k = 5
								(2,0,30.94)(2,1,33.03)(2,2,33.64)(2,3,33.93)(2,4,33.90)(2,5,33.90)(2,6,33.57)
								% Row 4: beta_k = 7
								(3,0,30.70)(3,1,32.88)(3,2,33.57)(3,3,33.80)(3,4,33.78)(3,5,33.70)(3,6,33.50)
								% Row 5: beta_k = 10
								(4,0,30.24)(4,1,32.99)(4,2,33.63)(4,3,33.80)(4,4,33.97)(4,5,33.87)(4,6,33.91)
								% Row 5: beta_k = 15
								(5,0,29.54)(5,1,32.88)(5,2,33.43)(5,3,33.79)(5,4,33.90)(5,5,33.84)(5,6,33.83)
								% Row 5: beta_k = 20
								(6,0,28.25)(6,1,32.29)(6,2,33.10)(6,3,33.65)(6,4,33.74)(6,5,33.75)(6,6,33.47)
							};
						\end{axis}
					\end{tikzpicture}
			\end{tabular}}
			\caption{PSNR sensitivity to hyperparameters $\beta$ and $\omega_0$ for inpainting on \textit{Airplane} ($\text{SR}=0.3$) (left) and denoising on \textit{Toy} ($\text{SD}=0.2$) (right).}
			\label{fig:sensitivity_beta_omega}
			\vspace{-0.5em}
		\end{figure}
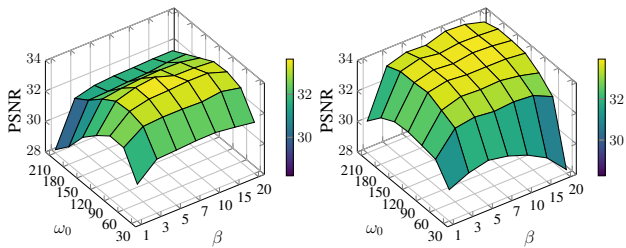
		
		\noindent{\bf Impact of Explicit Regularization. }To disentangle the effect of our reparameterized architecture from that of explicit smoothness priors, we further evaluate a variant of RepTRFD trained solely with a data fidelity loss, i.e., without incorporating TV or SSTV regularization. We conduct experiments on both inpainting and super-resolution tasks, contrasting performance against the representative functional tensor baseline, LRTFR~\cite{luo2023low}. As reported in Tables~\ref{tab:ablation_reg_inpainting} and \ref{tab:ablation_reg_sr}, even without explicit regularization, our method yields significant PSNR gains over LRTFR across all datasets. This demonstrates the inherent advantage of the proposed reparameterization scheme in capturing latent structures. When regularization is incorporated, performance is further boosted, confirming its complementary role in enhancing spatial smoothness and mitigating artifacts. These results collectively verify that the primary performance improvement stems from the RepTRFD architecture itself, while explicit regularization serves as an effective refinement mechanism.
		
		\begin{table}[htbp]
			\centering
			\caption{PSNR comparison on image and video inpainting at $\text{SR}=0.2$ under different regularization settings.}
			\small
			\setlength{\tabcolsep}{4pt}
			\resizebox{\linewidth}{!}{
				\begin{tabular}{lccc}
					\toprule
					Data & LRTFR (Baseline) & RepTRFD (w/o TV) & RepTRFD (w/ TV) \\
					\midrule
					\textit{Airplane} & 27.21 & 29.37 & \textbf{30.45} \\
					\textit{Flowers} & 44.27 & 48.53 & \textbf{50.13} \\
					\textit{Botswana} & 41.86 & 44.94 & \textbf{45.27} \\
					\textit{Carphone} & 30.18 & 31.86 & \textbf{32.58} \\
					\bottomrule
				\end{tabular}
			}
			\label{tab:ablation_reg_inpainting}
			\vspace{-1em}
		\end{table}
		
		\begin{table}[htbp]
			\centering
			\caption{PSNR comparison on $\times 4$ image super-resolution under different regularization settings.}
			\small
			\setlength{\tabcolsep}{6pt}
			\resizebox{\linewidth}{!}{
				\begin{tabular}{lccc}
					\toprule
					Data & LRTFR (Baseline) & RepTRFD (w/o TV) & RepTRFD (w/ TV) \\
					\midrule
					\textit{Lion} & 28.10 & 30.37 & \textbf{31.01} \\
					\textit{Parrot} & 27.81 & 29.79 & \textbf{30.39} \\
					\bottomrule
				\end{tabular}
			}
			\label{tab:ablation_reg_sr}
			\vspace{-0.5em}
		\end{table}
		
		\noindent{\bf Convergence Behavior. }To evaluate the effect of the proposed reparameterization on optimization efficiency, we compare the training loss curves of TRFD, LRTFR, and RepTRFD under matched parameter budgets. Fig.~\ref{fig:ablation_convergence} shows the training loss curves of TRFD, LRTFR, and RepTRFD under matched parameter settings. RepTRFD converges faster and reaches a lower loss, demonstrating improved convergence behavior.
		
		\begin{figure}[htbp]
			\centering
			\resizebox{0.5\linewidth}{!}{
				\begin{tabular}{c}
					\includegraphics[width=\linewidth]{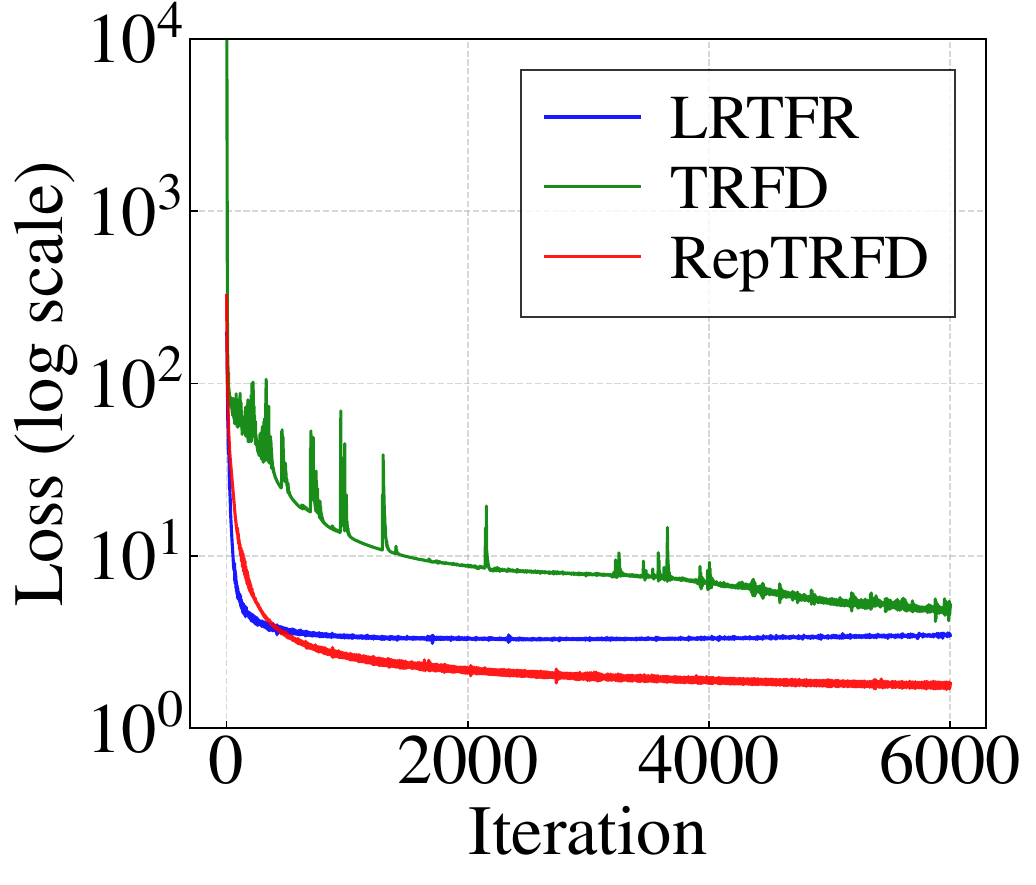} 
			\end{tabular}}
			\vspace{-1em}
			\caption{Loss curve on the color image \textit{Airplane} ($\text{SR}=0.3$).}
			\label{fig:ablation_convergence}
		\end{figure}
		
		\noindent{\bf Initialization Scheme. }To examine whether the theoretical analysis depends on a specific variance-preserving strategy, we further compare Xavier and Kaiming initialization schemes. 
		Under Kaiming initialization with forward-pass variance preservation, the basis entries follow $\mathbf{B}^{(k)}_{ij} \sim \mathcal{U}\!\left(-\sqrt{\tfrac{6}{R_{k+1}}},\ \sqrt{\tfrac{6}{R_{k+1}}}\right)$. Table~\ref{tab:ablation_kaiming} reports the reconstruction performance on four datasets with $\text{SR}=0.2$. 
		We observe that the two initialization strategies yield highly comparable results across all datasets, with only marginal differences. 
		This indicates that our framework is robust to the specific choice of variance-preserving initialization and that the theoretical analysis naturally extends beyond Xavier initialization.
		
		\begin{table}[htbp]
			\centering
			\caption{PSNR results of different variance-preserving initialization schemes on multiple datasets with $\text{SR}=0.2$.}
			\small
			\setlength{\tabcolsep}{3pt}
			\resizebox{0.7\linewidth}{!}{
				\begin{tabular}{lcccc}
					\toprule
					Scheme & \textit{Airplane} & \textit{Toy} & \textit{Botswana} & \textit{News} \\
					\midrule
					Kaiming & 30.23 & 48.59 & \textbf{45.30} & 34.75 \\
					Xavier & \textbf{30.45} & \textbf{48.67} & 45.27 & \textbf{34.90} \\
					\bottomrule
				\end{tabular}
			}
			\label{tab:ablation_kaiming}
		\end{table}
		
		\subsection{Extended Experimental Results}
		
		\noindent{\bf Extended Inpainting Results. }To provide a more comprehensive evaluation, we further include extensive qualitative and quantitative comparisons under various SRs and across diverse data modalities. As shown in Fig.~\ref{fig:inpainting_supp}, we visualize representative inpainting results on RGB, MSI, HSI, and video datasets, covering a wide range of missing patterns and SRs. The visual comparisons consistently demonstrate that our method yields clearer structural details and fewer artifacts compared to existing approaches. For quantitative evaluation, we report full numerical results of color image inpainting in Table~\ref{tab:inpainting_color_image_supp}, and summarize results on MSI, HSI, and video sequences in Table~\ref{tab:inpainting_others_supp}. These results further confirm the robustness and superiority of our method.
		
		\begin{figure*}[tb]
			\renewcommand{\arraystretch}{0.5}
			\setlength\tabcolsep{0.5pt}
			\centering
			\resizebox{\linewidth}{!}{
				\begin{tabular}{ccccccccc}
					\scriptsize{PSNR/SSIM} & \scriptsize{22.25/0.617} & \scriptsize{20.85/0.619} & \scriptsize{25.40/0.781} & \scriptsize{26.26/0.796} & \scriptsize{26.89/0.828} &  \scriptsize{26.99/0.856} & \scriptsize{\textbf{28.60}/\textbf{0.900}} & \scriptsize{\textbf{Inf/1.000}} \\
					\includegraphics[width=18.5mm]{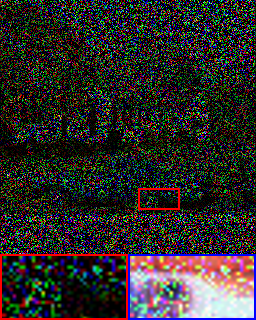} & \includegraphics[width=18.5mm]{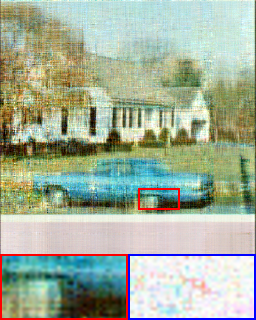} &
					\includegraphics[width=18.5mm]{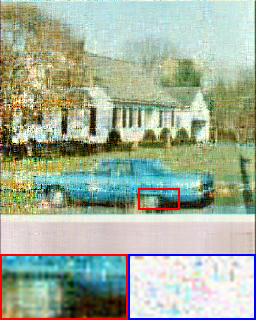} &
					\includegraphics[width=18.5mm]{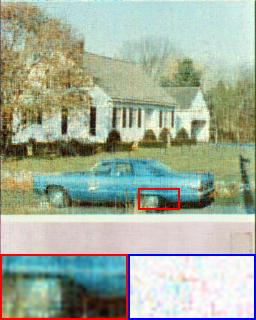} &
					\includegraphics[width=18.5mm]{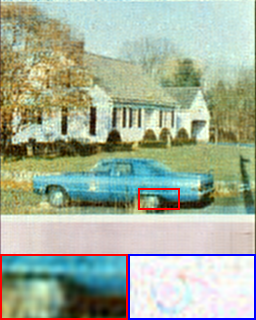} &
					\includegraphics[width=18.5mm]{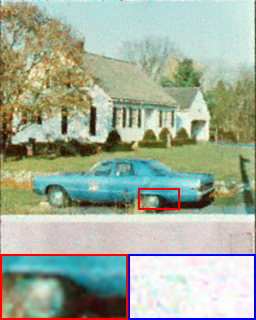} &
					\includegraphics[width=18.5mm]{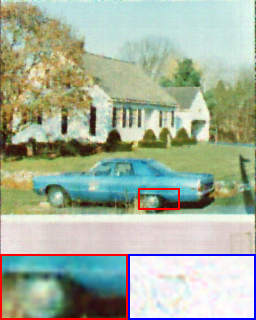} &
					\includegraphics[width=18.5mm]{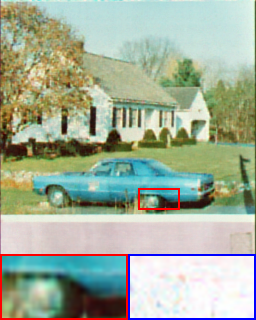} &
					\includegraphics[width=18.5mm]{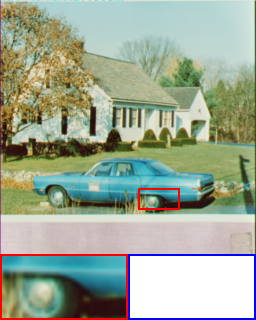} \\
					\scriptsize{PSNR/SSIM} & \scriptsize{20.99/0.493} & \scriptsize{20.90/0.501} & \scriptsize{24.03/0.695} & \scriptsize{24.97/0.720} & \scriptsize{28.52/0.871} &  \scriptsize{28.73/0.909} & \scriptsize{\textbf{29.82}/\textbf{0.919}} & \scriptsize{\textbf{Inf/1.000}} \\
					\includegraphics[width=18.5mm]{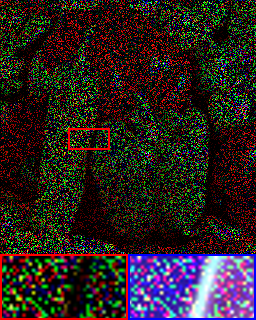} & \includegraphics[width=18.5mm]{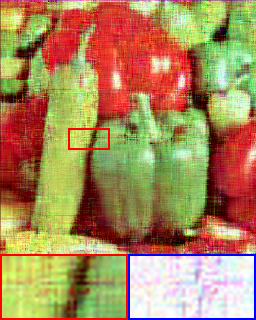} &
					\includegraphics[width=18.5mm]{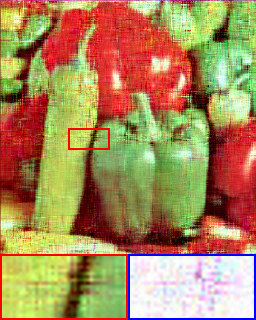} &
					\includegraphics[width=18.5mm]{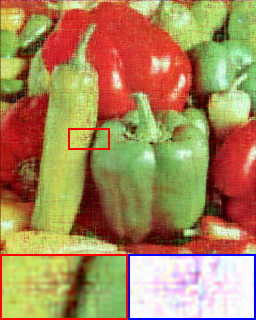} &
					\includegraphics[width=18.5mm]{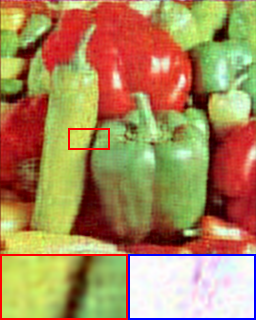} &
					\includegraphics[width=18.5mm]{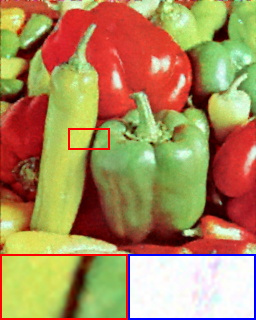} &
					\includegraphics[width=18.5mm]{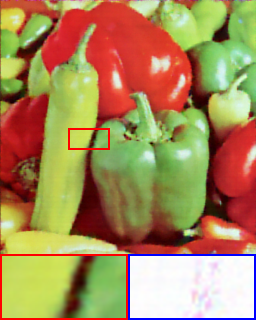} &
					\includegraphics[width=18.5mm]{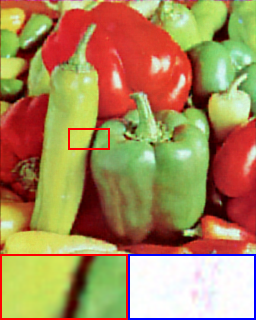} &
					\includegraphics[width=18.5mm]{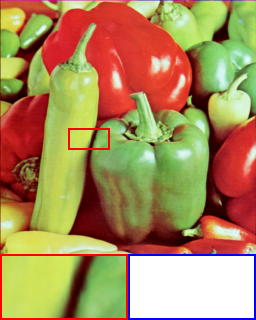} \\
					\scriptsize{PSNR/SSIM} & \scriptsize{21.53/0.547} & \scriptsize{21.07/0.545} & \scriptsize{24.86/0.728} & \scriptsize{25.00/0.744} & \scriptsize{27.01/0.839} &  \scriptsize{27.29/0.865} & \scriptsize{\textbf{28.62}/\textbf{0.907}} & \scriptsize{\textbf{Inf/1.000}} \\
					\includegraphics[width=18.5mm]{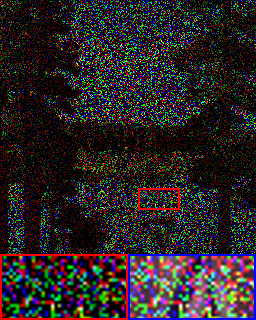} & \includegraphics[width=18.5mm]{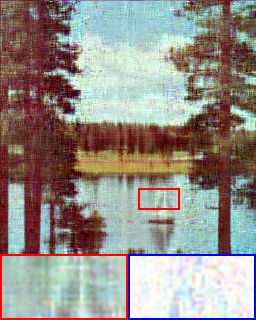} &
					\includegraphics[width=18.5mm]{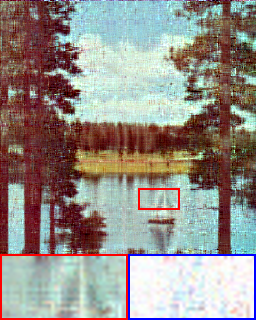} &
					\includegraphics[width=18.5mm]{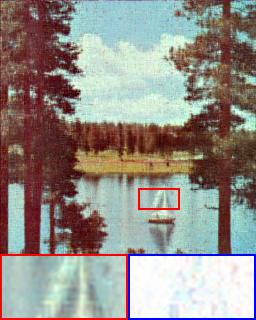} &
					\includegraphics[width=18.5mm]{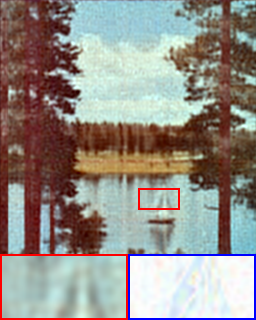} &
					\includegraphics[width=18.5mm]{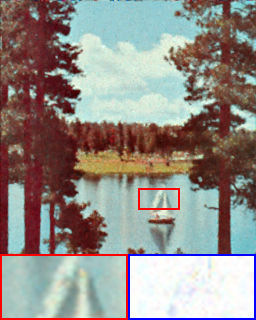} &
					\includegraphics[width=18.5mm]{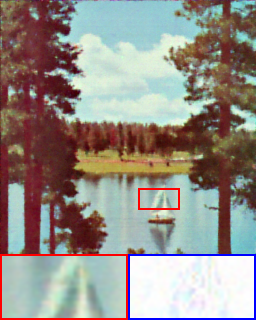} &
					\includegraphics[width=18.5mm]{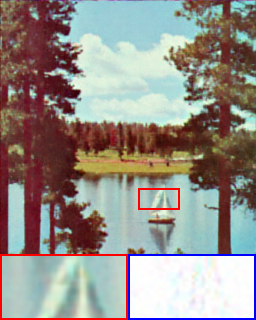} &
					\includegraphics[width=18.5mm]{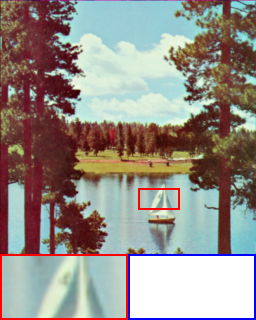} \\
					\scriptsize{PSNR/SSIM} & \scriptsize{29.32/0.791} & \scriptsize{32.24/0.882} & \scriptsize{33.85/0.929} & \scriptsize{34.92/0.948} & \scriptsize{37.19/0.982} &  \scriptsize{36.31/0.962} & \scriptsize{\textbf{38.39}/\textbf{0.985}} & \scriptsize{\textbf{Inf/1.000}} \\
					\includegraphics[width=18.5mm]{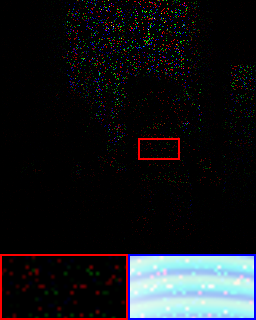} & \includegraphics[width=18.5mm]{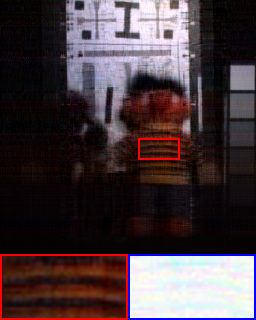} &
					\includegraphics[width=18.5mm]{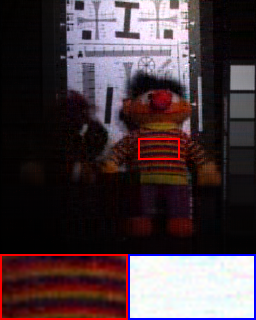} &
					\includegraphics[width=18.5mm]{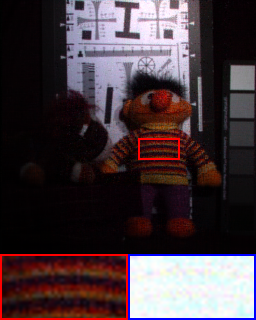} &
					\includegraphics[width=18.5mm]{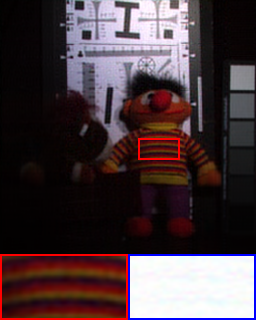} &
					\includegraphics[width=18.5mm]{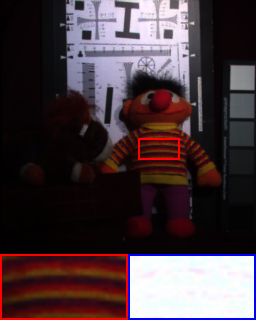} &
					\includegraphics[width=18.5mm]{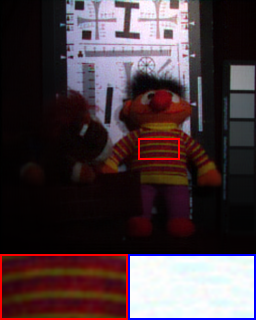} &
					\includegraphics[width=18.5mm]{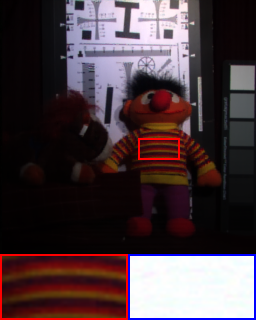} &
					\includegraphics[width=18.5mm]{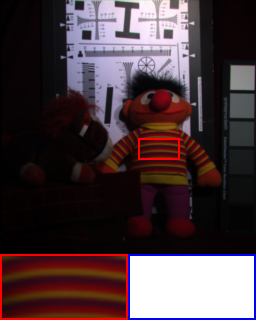} \\
					\scriptsize{PSNR/SSIM} & \scriptsize{33.13/0.836} & \scriptsize{36.07/0.904} & \scriptsize{41.58/0.979} & \scriptsize{41.47/0.978} & \scriptsize{43.94/0.991} &  \scriptsize{42.90/0.985} & \scriptsize{\textbf{45.29}/\textbf{0.992}} & \scriptsize{\textbf{Inf/1.000}} \\
					\includegraphics[width=18.5mm]{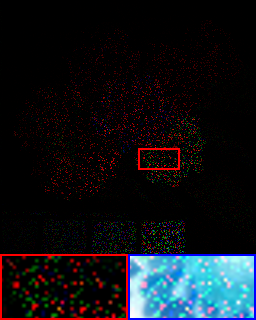} & \includegraphics[width=18.5mm]{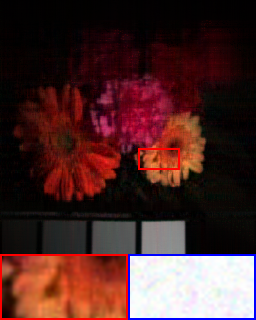} &
					\includegraphics[width=18.5mm]{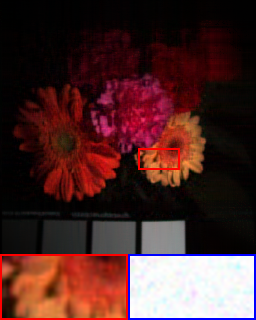} &
					\includegraphics[width=18.5mm]{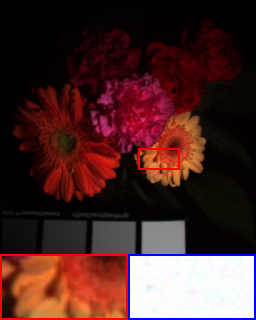} &
					\includegraphics[width=18.5mm]{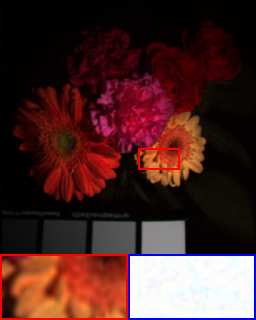} &
					\includegraphics[width=18.5mm]{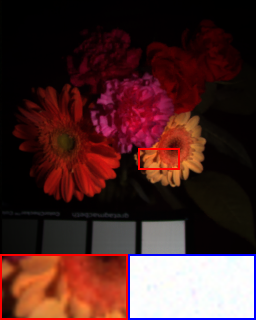} &
					\includegraphics[width=18.5mm]{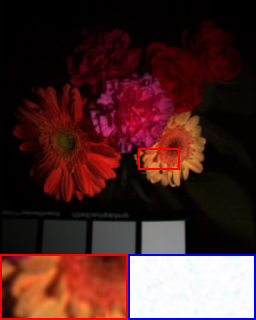} &
					\includegraphics[width=18.5mm]{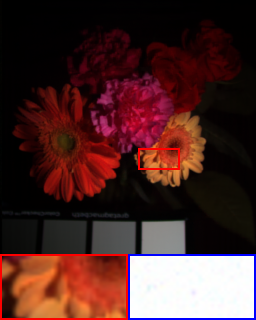} &
					\includegraphics[width=18.5mm]{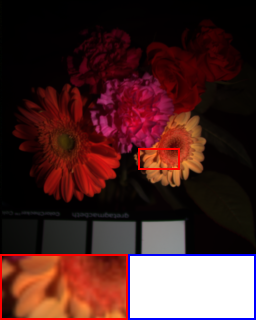} \\
					\scriptsize{PSNR/SSIM} & \scriptsize{30.07/0.781} & \scriptsize{32.74/0.897} & \scriptsize{36.43/0.963} & \scriptsize{35.81/0.949} & \scriptsize{37.96/0.968} &  \scriptsize{36.60/0.967} & \scriptsize{\textbf{39.52}/\textbf{0.977}} & \scriptsize{\textbf{Inf/1.000}} \\
					\includegraphics[width=18.5mm]{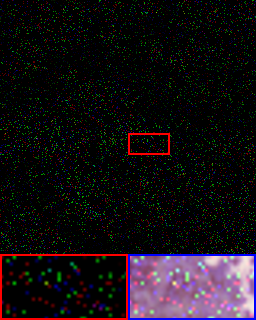} & \includegraphics[width=18.5mm]{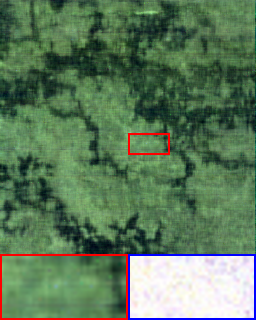} &
					\includegraphics[width=18.5mm]{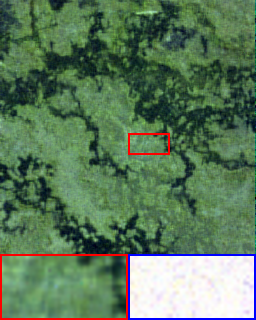} &
					\includegraphics[width=18.5mm]{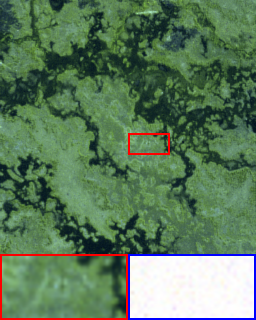} &
					\includegraphics[width=18.5mm]{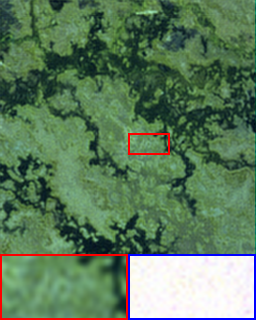} &
					\includegraphics[width=18.5mm]{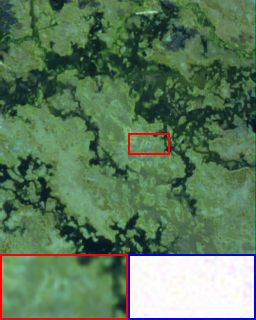} &
					\includegraphics[width=18.5mm]{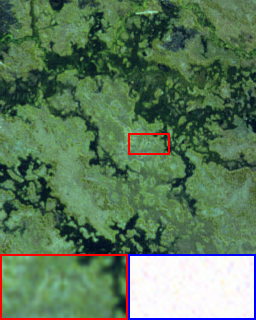} &
					\includegraphics[width=18.5mm]{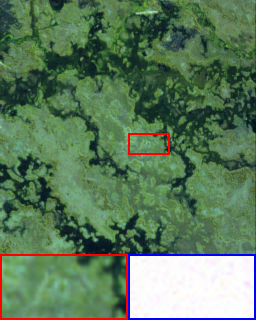} &
					\includegraphics[width=18.5mm]{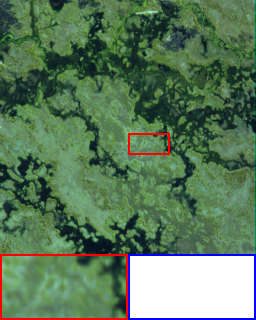} \\
					\scriptsize{PSNR/SSIM} & \scriptsize{31.36/0.886} & \scriptsize{32.38/0.911} & \scriptsize{35.24/0.962} & \scriptsize{34.40/0.943} & \scriptsize{35.78/0.963} &  \scriptsize{35.45/0.962} & \scriptsize{\textbf{37.05}/\textbf{0.973}} & \scriptsize{\textbf{Inf/1.000}} \\
					\includegraphics[width=18.5mm]{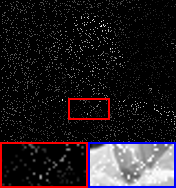} & \includegraphics[width=18.5mm]{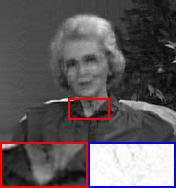} &
					\includegraphics[width=18.5mm]{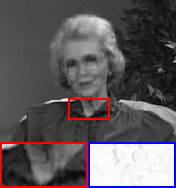} &
					\includegraphics[width=18.5mm]{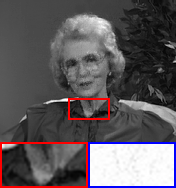} &
					\includegraphics[width=18.5mm]{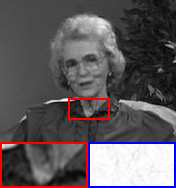} &
					\includegraphics[width=18.5mm]{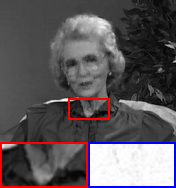} &
					\includegraphics[width=18.5mm]{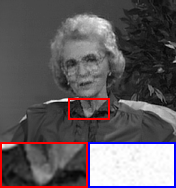} &
					\includegraphics[width=18.5mm]{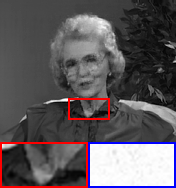} &
					\includegraphics[width=18.5mm]{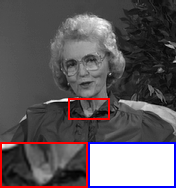} \\
					\scriptsize{PSNR/SSIM} & \scriptsize{28.71/0.844} & \scriptsize{29.38/0.861} & \scriptsize{30.74/0.885} & \scriptsize{30.18/0.890} & \scriptsize{31.12/0.918} &  \scriptsize{30.62/0.903} & \scriptsize{\textbf{32.58}/\textbf{0.943}} & \scriptsize{\textbf{Inf/1.000}} \\
					\includegraphics[width=18.5mm]{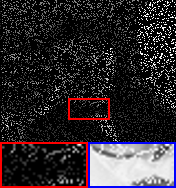} & \includegraphics[width=18.5mm]{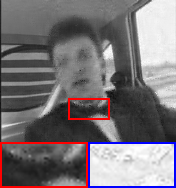} &
					\includegraphics[width=18.5mm]{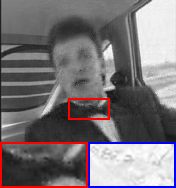} &
					\includegraphics[width=18.5mm]{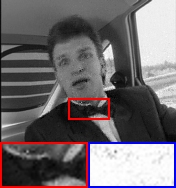} &
					\includegraphics[width=18.5mm]{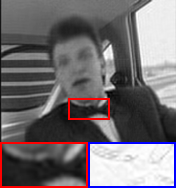} &
					\includegraphics[width=18.5mm]{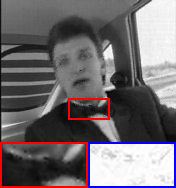} &
					\includegraphics[width=18.5mm]{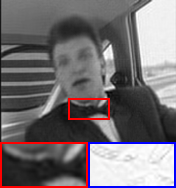} &
					\includegraphics[width=18.5mm]{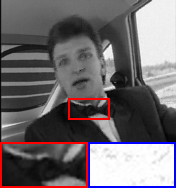} &
					\includegraphics[width=18.5mm]{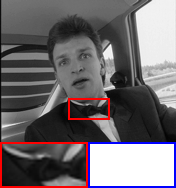} \\
					\scriptsize{Observed} & \scriptsize TRLRF & \scriptsize FCTN & \scriptsize HLRTF & \scriptsize LRTFR & \scriptsize DRO-TFF & \scriptsize NeurTV & \scriptsize Ours & \scriptsize{Ground truth}
			\end{tabular}}
			\caption{Visual inpainting results on various datasets. From top to bottom: color images (\textit{House}, \textit{Peppers}, \textit{Sailboat}) with $\text{SR}=0.2$; MSIs (\textit{Toy}, \textit{Flowers}) with $\text{SR}=0.05$ and $0.1$, respectively; HSI (\textit{Botswana}) with $\text{SR}=0.05$; and video sequences (\textit{Grandma}, \textit{Carphone}) with $\text{SR}=0.1$ and $0.2$, respectively.}
			\label{fig:inpainting_supp}
		\end{figure*}
		
		\begin{table*}[tb]
			\renewcommand{\arraystretch}{1.15}
			\setlength\tabcolsep{2.5pt}
			\small
			\caption{Quantitative comparison on color images. PSNR, SSIM, and NRMSE are reported under different SRs.}
			\label{tab:inpainting_color_image_supp}
			\centering
			\resizebox{\linewidth}{!}{
				\begin{tabular}{c|l|ccc|ccc|ccc|ccc|ccc|ccc}
					\hline
					\multirow{2}{*}{Data} & \multirow{2}{*}{Method} 
					& \multicolumn{3}{c|}{$\text{SR}=0.05$} & \multicolumn{3}{c|}{$\text{SR}=0.1$} & \multicolumn{3}{c|}{$\text{SR}=0.15$} & \multicolumn{3}{c|}{$\text{SR}=0.2$} & \multicolumn{3}{c|}{$\text{SR}=0.25$} & \multicolumn{3}{c}{$\text{SR}=0.3$} \\
					\cline{3-20}
					& & PSNR & SSIM & NRMSE & PSNR & SSIM & NRMSE & PSNR & SSIM & NRMSE & PSNR & SSIM & NRMSE & PSNR & SSIM & NRMSE & PSNR & SSIM & NRMSE \\
					\hline
					\multirow{7}{*}{\textit{Airplane}} 
					& TRLRF & 15.08 & 0.172 & 0.218 & 17.98 & 0.347 & 0.156 & 19.90 & 0.444 & 0.125 & 21.50 & 0.529 & 0.104 & 23.35 & 0.623 & 0.084 & 25.58 & 0.712 & 0.065 \\
					& FCTN & 14.64 & 0.219 & 0.229 & 17.79 & 0.370 & 0.159 & 20.05 & 0.504 & 0.123 & 22.08 & 0.600 & 0.097 & 24.11 & 0.689 & 0.077 & 25.70 & 0.747 & 0.064 \\
					& HLRTF & 18.46 & 0.396 & 0.148 & 21.89 & 0.598 & 0.099 & 24.82 & 0.746 & 0.071 & 26.77 & 0.811 & 0.057 & 27.61 & 0.795 & 0.051 & 29.40 & 0.852 & 0.042 \\
					& LRTFR & 19.66 & 0.499 & 0.129 & 22.09 & 0.655 & 0.097 & 25.57 & 0.765 & 0.065 & 27.21 & 0.822 & 0.054 & 28.60 & 0.858 & 0.046 & 29.38 & 0.878 & 0.042 \\
					& DRO-TFF & 21.08 & \underline{0.741} & 0.109 & 22.71 & 0.763 & 0.090 & 26.28 & 0.843 & 0.060 & 27.64 & 0.853 & 0.051 & 29.17 & 0.888 & 0.043 & 30.37 & 0.906 & 0.037 \\
					& NeurTV & \underline{22.00} & 0.729 & \underline{0.098} & \underline{24.86} & \underline{0.827} & \underline{0.071} & \underline{26.55} & \underline{0.871} & \underline{0.058} & \underline{28.21} & \underline{0.905} & \underline{0.048} & \underline{29.80} & \underline{0.927} & \underline{0.040} & \underline{31.02} & \underline{0.932} & \underline{0.035} \\
					& Ours & \textbf{23.44} & \textbf{0.784} & \textbf{0.083} & \textbf{26.33} & \textbf{0.865} & \textbf{0.060} & \textbf{28.53} & \textbf{0.907} & \textbf{0.046} & \textbf{30.45} & \textbf{0.932} & \textbf{0.037} & \textbf{32.05} & \textbf{0.946} & \textbf{0.031} & \textbf{33.61} & \textbf{0.958} & \textbf{0.026} \\
					\hline
					\multirow{7}{*}{\textit{House}} 
					& TRLRF & 15.69 & 0.208 & 0.240 & 17.92 & 0.360 & 0.186 & 20.58 & 0.520 & 0.137 & 22.25 & 0.617 & 0.113 & 23.75 & 0.694 & 0.095 & 25.12 & 0.751 & 0.081 \\
					& FCTN & 15.47 & 0.248 & 0.246 & 18.13 & 0.424 & 0.181 & 19.76 & 0.545 & 0.150 & 20.85 & 0.619 & 0.133 & 22.20 & 0.681 & 0.113 & 24.44 & 0.756 & 0.088 \\
					& HLRTF & 17.93 & 0.378 & 0.186 & 21.48 & 0.589 & 0.123 & 23.93 & 0.723 & 0.093 & 25.40 & 0.781 & 0.078 & 26.76 & 0.828 & 0.067 & 28.19 & 0.871 & 0.057 \\
					& LRTFR & 19.51 & 0.460 & 0.155 & 21.72 & 0.604 & 0.120 & 24.77 & 0.743 & 0.084 & 26.26 & 0.796 & 0.071 & 26.91 & 0.834 & 0.066 & 28.58 & 0.879 & 0.054 \\
					& DRO-TFF & 19.17 & 0.582 & 0.161 & 23.06 & 0.696 & 0.103 & 25.04 & 0.759 & 0.082 & 26.89 & 0.828 & 0.066 & 27.97 & 0.848 & 0.058 & 28.99 & 0.875 & 0.052 \\
					& NeurTV & \underline{21.06} & \underline{0.651} & \underline{0.129} & \underline{23.28} & \underline{0.761} & \underline{0.100} & \underline{25.56} & \underline{0.820} & \underline{0.077} & \underline{26.99} & \underline{0.856} & \underline{0.065} & \underline{28.26} & \underline{0.881} & \underline{0.056} & \underline{29.02} & \underline{0.910} & \underline{0.052} \\
					& Ours & \textbf{22.68} & \textbf{0.715} & \textbf{0.107} & \textbf{25.36} & \textbf{0.821} & \textbf{0.079} & \textbf{27.22} & \textbf{0.870} & \textbf{0.064} & \textbf{28.60} & \textbf{0.900} & \textbf{0.054} & \textbf{29.99} & \textbf{0.923} & \textbf{0.046} & \textbf{30.98} & \textbf{0.934} & \textbf{0.041} \\
					\hline
					\multirow{7}{*}{\textit{Pepper}} 
					& TRLRF & 12.97 & 0.139 & 0.409 & 16.83 & 0.273 & 0.262 & 18.77 & 0.384 & 0.210 & 20.99 & 0.493 & 0.162 & 22.92 & 0.580 & 0.130 & 24.23 & 0.646 & 0.112 \\
					& FCTN & 13.86 & 0.166 & 0.369 & 16.24 & 0.282 & 0.281 & 18.82 & 0.403 & 0.209 & 20.90 & 0.501 & 0.164 & 23.02 & 0.594 & 0.129 & 24.37 & 0.650 & 0.110 \\
					& HLRTF & 15.66 & 0.249 & 0.300 & 19.19 & 0.451 & 0.200 & 21.68 & 0.579 & 0.150 & 24.03 & 0.695 & 0.114 & 25.62 & 0.765 & 0.095 & 26.84 & 0.801 & 0.083 \\
					& LRTFR & 15.06 & 0.289 & 0.321  & 20.76 & 0.537 & 0.167 & 22.95 & 0.655 & 0.130 & 24.97 & 0.720 & 0.103 & 26.15 & 0.767 & 0.090 & 27.98 & 0.796 & 0.073 \\
					& DRO-TFF & 18.84 & 0.502 & 0.208 & 23.73 & 0.730 & 0.118 & 26.55 & 0.824 & 0.086 & 28.52 & 0.871 & 0.068 & 30.22 & 0.906 & 0.056 & 31.38 & 0.923 & 0.049 \\
					& NeurTV & \underline{22.14} & \underline{0.729} & \underline{0.142} & \underline{24.69} & \underline{0.822} & \underline{0.106} & \underline{26.99} & \underline{0.876} & \underline{0.081} & \underline{28.73} & \underline{0.909} & \underline{0.067} & \underline{30.37} & \underline{0.929} & \underline{0.055} & \underline{31.46} & \underline{0.938} & \underline{0.049} \\
					& Ours & \textbf{23.14} & \textbf{0.755} & \textbf{0.127} & \textbf{26.10} & \textbf{0.854} & \textbf{0.090} & \textbf{28.02} & \textbf{0.878} & \textbf{0.072} & \textbf{29.82} & \textbf{0.919} & \textbf{0.059} & \textbf{31.29} & \textbf{0.935} & \textbf{0.050} & \textbf{32.46} & \textbf{0.945} & \textbf{0.043} \\
					\hline
					\multirow{7}{*}{\textit{Sailboat}} 
					& TRLRF & 15.17 & 0.197 & 0.306 & 17.35 & 0.324 & 0.238 & 19.44 & 0.438 & 0.187 & 21.53 & 0.547 & 0.147 & 22.66 & 0.606 & 0.129 & 24.45 & 0.683 & 0.105 \\
					& FCTN & 14.69 & 0.209 & 0.323 & 17.37 & 0.362 & 0.238 & 18.61 & 0.434 & 0.206 & 21.07 & 0.545 & 0.155 & 22.79 & 0.631 & 0.127 & 24.16 & 0.687 & 0.109 \\
					& HLRTF & 18.12 & 0.366 & 0.218 & 20.98 & 0.539 & 0.157 & 23.07 & 0.648 & 0.123 & 24.86 & 0.728 & 0.100 & 26.35 & 0.794 & 0.084 & 27.66 & 0.838 & 0.073 \\
					& LRTFR & 18.88 & 0.412 & 0.200 & 20.96 & 0.552 & 0.157 & 23.95 & 0.693 & 0.111 & 25.00 & 0.744 & 0.099 & 27.04 & 0.806 & 0.078 & 29.03 & 0.871 & 0.062 \\
					& DRO-TFF & 21.07 & 0.658 & 0.155 & 23.37 & 0.747 & 0.119 & 25.47 & 0.800 & \underline{0.093} & 27.01 & 0.839 & 0.078 & 28.15 & 0.862 & 0.069 & 29.40 & 0.890 & 0.059 \\
					& NeurTV & \underline{21.55} & \underline{0.663} & \underline{0.147} & \underline{23.79} & \underline{0.778} & \underline{0.113} & \underline{25.53} & \underline{0.819} & \underline{0.093} & \underline{27.29} & \underline{0.865} & \underline{0.076} & \underline{28.47} & \underline{0.888} & \underline{0.066} & \underline{29.60} & \underline{0.916} & \underline{0.058} \\
					& Ours & \textbf{22.43} &\textbf{ 0.722} & \textbf{0.133} & \textbf{25.00} & \textbf{0.823} & \textbf{0.099} & \textbf{27.11} & \textbf{0.876} & \textbf{0.077} & \textbf{28.62} & \textbf{0.907} & \textbf{0.065} & \textbf{29.80} & \textbf{0.923} & \textbf{0.057} & \textbf{31.00} & \textbf{0.938} & \textbf{0.049} \\
					\hline
			\end{tabular}}
		\end{table*}
		
		\begin{table*}[tb]
			\renewcommand{\arraystretch}{1.15}
			\setlength\tabcolsep{2.5pt}
			\small
			\caption{Quantitative comparison on MSIs, HSIs and videos. PSNR, SSIM, and NRMSE are reported under different SRs.}
			\label{tab:inpainting_others_supp}
			\centering
			\resizebox{\linewidth}{!}{
				\begin{tabular}{c|l|ccc|ccc|ccc|ccc|ccc}
					\hline
					\multirow{2}{*}{Data} & \multirow{2}{*}{Method} 
					& \multicolumn{3}{c|}{$\text{SR}=0.03$} & \multicolumn{3}{c|}{$\text{SR}=0.05$} & \multicolumn{3}{c|}{$\text{SR}=0.1$} & \multicolumn{3}{c|}{$\text{SR}=0.15$} & \multicolumn{3}{c}{$\text{SR}=0.2$} \\
					\cline{3-17}
					& & PSNR & SSIM & NRMSE & PSNR & SSIM & NRMSE & PSNR & SSIM & NRMSE & PSNR & SSIM & NRMSE & PSNR & SSIM & NRMSE \\
					\hline
					\multirow{8}{*}{\makecell[c]{\textit{Toys}\\$(256\times 256\times 31)$}}
					& TRLRF & 25.13 & 0.625 & 0.189 & 29.32 & 0.791 & 0.117 & 33.67 & 0.901 & 0.071 & 36.56 & 0.943 & 0.051 & 38.67 & 0.964 & 0.040 \\
					& FCTN & 28.22 & 0.788 & 0.132 & 32.24 & 0.882 & 0.083 & 37.46 & 0.955 & 0.046 & 40.04 & 0.972 & 0.034 & 42.04 & 0.982 & 0.027 \\
					& HLRTF & 30.73 & 0.880 & 0.099 & 33.85 & 0.929 & 0.069 & 40.15 & 0.982 & 0.033 & 42.64 & 0.988 & 0.025 & 45.56 & 0.993 & 0.018 \\
					& LRTFR & 32.95 & 0.924 & 0.077 & 34.92 & 0.948 & 0.061 & 39.95 & 0.982 & 0.034 & 41.87 & 0.987 & 0.027 & 43.64 & 0.988 & 0.022 \\
					& DRO-TFF & \underline{33.88} & \underline{0.966} & \underline{0.069} & \underline{37.19} & \underline{0.982} & \underline{0.047} & \underline{40.62} & \underline{0.992} & \underline{0.032} & \underline{43.73} & \underline{0.994} & \underline{0.022} & \underline{45.75} & \underline{0.995} & \underline{0.017} \\
					& NeurTV & 33.34 & 0.927 & 0.073 & 36.31 & 0.962 & 0.052 & 40.43 & 0.984 & \underline{0.032} & 42.31 & 0.987 & 0.026 & 43.09 & 0.987 & 0.024 \\
					& Ours & \textbf{35.16} & \textbf{0.972} & \textbf{0.059} & \textbf{38.39} & \textbf{0.984} & \textbf{0.041} & \textbf{44.04} & \textbf{0.993} & \textbf{0.021} & \textbf{47.16} & \textbf{0.995} & \textbf{0.015} & \textbf{48.67} & \textbf{0.996} & \textbf{0.011} \\
					\hline
					\multirow{8}{*}{\makecell[c]{\textit{Flowers}\\$(256\times 256\times 31)$}} 
					& TRLRF & 25.38 & 0.604 & 0.318 & 28.57 & 0.705 & 0.220 & 33.13 & 0.836 & 0.130 & 35.79 & 0.902 & 0.096 & 38.81 & 0.944 & 0.068 \\
					& FCTN & 28.31 & 0.718 & 0.227 & 31.17 & 0.793 & 0.163 & 36.07 & 0.904 & 0.093 & 38.61 & 0.939 & 0.069 & 42.22 & 0.970 & 0.046 \\
					& HLRTF & 30.42 & 0.830 & 0.178 & 35.92 & 0.945 & 0.095 & 41.58 & 0.979 & 0.049 & 44.32 & 0.988 & 0.036 & 46.84 & 0.992 & 0.027 \\
					& LRTFR & 32.64 & 0.884 & 0.138 & 37.84 & 0.956 & 0.076 & 41.47 & 0.978 & 0.050 & 43.10 & 0.982 & 0.041 & 44.27 & 0.986 & 0.036 \\
					& DRO-TFF & \underline{37.09} & \textbf{0.970} & \textbf{0.082} & \underline{39.71} & \underline{0.982} & \underline{0.061} & \underline{43.94} & \underline{0.991} & \underline{0.038} & \underline{46.28} & \underline{0.994} & \underline{0.029} & \underline{47.51} & \underline{0.995} & \underline{0.025} \\
					& NeurTV & 35.61 & \underline{0.934} & \underline{0.098} & 38.67 & 0.966 & 0.069 & 42.90 & 0.985 & 0.042 & 45.12 & 0.990 & 0.033 & 46.54 & 0.992 & 0.028 \\
					& Ours & \textbf{37.11} & \textbf{0.970} & \textbf{0.082} & \textbf{40.30} & \textbf{0.983} & \textbf{0.057} & \textbf{45.29} & \textbf{0.992} & \textbf{0.032} & \textbf{48.32} & \textbf{0.995} & \textbf{0.023} & \textbf{50.13} & \textbf{0.996} & \textbf{0.018} \\
					\hline
					\multirow{8}{*}{\makecell[c]{\textit{Washinton DC}\\$(256\times 256\times 191)$}} 
					& TRLRF & 26.83 & 0.635 & 0.258 & 29.22 & 0.771 & 0.196 & 32.40 & 0.914 & 0.136 & 35.31 & 0.954 & 0.097 & 37.98 & 0.973 & 0.071 \\
					& FCTN & 28.76 & 0.830 & 0.207 & 32.09 & 0.912 & 0.141 & 34.52 & 0.946 & 0.106 & 36.38 & 0.963 & 0.086 & 41.62 & 0.987 & 0.047 \\
					& HLRTF & 35.31 & 0.957 & 0.097 & 38.29 & 0.976 & 0.069 & 41.68 & 0.988 & 0.047 & 45.04 & \underline{0.994} & 0.032 & 46.66 & \underline{0.995} & 0.026 \\
					& LRTFR & 34.08 & 0.943 & 0.112 & 37.42 & 0.975 & 0.076 & 40.72 & 0.988 & 0.052 & 45.08 & 0.993 & 0.032 & 46.56 & 0.994 & 0.027 \\
					& DRO-TFF & \underline{36.83} & \underline{0.961} & \underline{0.082} & \underline{39.22} & 0.975 & \underline{0.062} & 41.62 & 0.988 & 0.047 & 42.84 & 0.990 & 0.041 & 44.13 & 0.993 & 0.035 \\
					& NeurTV & 35.03 & 0.950 & 0.100 & 38.98 & \underline{0.977} & 0.063 & \underline{44.66} & \underline{0.992} & \underline{0.033} & \underline{46.40} & 0.993 & \underline{0.027} & \underline{47.20} & \underline{0.995} & \underline{0.025} \\
					& Ours & \textbf{38.44} & \textbf{0.979} & \textbf{0.068} & \textbf{41.97} & \textbf{0.990} & \textbf{0.045} & \textbf{45.62} & \textbf{0.994} & \textbf{0.030} & \textbf{47.26} & \textbf{0.995} & \textbf{0.025} & \textbf{47.96} & \textbf{0.996} & \textbf{0.023} \\
					\hline
					\multirow{8}{*}{\makecell[c]{\textit{Botswana}\\$(256\times 256\times 145)$}} 
					& TRLRF & 26.89 & 0.620 & 0.186 & 30.07 & 0.781 & 0.129 & 32.08 & 0.883 & 0.102 & 34.12 & 0.922 & 0.081 & 35.49 & 0.941 & 0.069 \\
					& FCTN & 29.32 & 0.804 & 0.141 & 32.74 & 0.897 & 0.095 & 35.36 & 0.939 & 0.070 & 36.22 & 0.949 & 0.063 & 37.70 & 0.962 & 0.054 \\
					& HLRTF & 33.04 & 0.918 & 0.092 & 36.43 & 0.963 & 0.062 & 38.60 & 0.979 & 0.048 & 39.13 & 0.981 & 0.045 & 40.39 & 0.989 & 0.039 \\
					& LRTFR & 32.87 & 0.906 & 0.093 & 35.81 & 0.949 & 0.067 & 37.49 & 0.964 & 0.055 & 40.05 & 0.980 & 0.041 & 41.86 & 0.986 & \underline{0.033} \\
					& DRO-TFF & \underline{34.74} & \underline{0.945} & \underline{0.075} & \underline{37.96} & \underline{0.968} & \underline{0.052} & \underline{40.84} & 0.981 & \underline{0.037} & \underline{41.38} & 0.982 & \underline{0.035} & \underline{42.01} & 0.984 & \underline{0.033} \\
					& NeurTV & 34.23 & 0.930 & 0.080 & 36.60 & 0.967 & 0.061 & 39.32 & \underline{0.982} & 0.044 & 40.25 & \underline{0.984} & 0.040 & 40.74 & \underline{0.990} & 0.038 \\
					& Ours & \textbf{37.11} & \textbf{0.965} & \textbf{0.057} & \textbf{39.52} & \textbf{0.977} & \textbf{0.043} & \textbf{42.81} & \textbf{0.987} & \textbf{0.030} & \textbf{44.40} & \textbf{0.990} & \textbf{0.025} & \textbf{45.27} & \textbf{0.992} & \textbf{0.022} \\
					\hline
					\multirow{8}{*}{\makecell[c]{\textit{News}\\$(144\times 176\times 100)$}} 
					& TRLRF & 23.66 & 0.687 & 0.180 & 25.05 & 0.751 & 0.153 & 26.76 & 0.812 & 0.126 & 28.18 & 0.857 & 0.107 & 30.00 & 0.895 & 0.087 \\
					& FCTN & 24.28 & 0.710 & 0.168 & 25.81 & 0.776 & 0.141 & 27.74 & 0.837 & 0.113 & 29.24 & 0.877 & 0.095 & 31.00 & 0.915 & 0.077 \\
					& HLRTF & 25.28 & 0.811 & 0.149 & 26.61 & 0.842 & 0.128 & 29.64 & 0.914 & 0.090 & 31.28 & 0.937 & 0.075 & 32.37 & 0.949 & 0.066 \\
					& LRTFR & 26.15 & 0.815 & 0.135 & 26.63 & 0.811 & 0.128 & 29.62 & 0.893 & 0.091 & 30.58 & 0.910 & 0.081 & 31.86 & 0.931 & 0.070 \\
					& DRO-TFF & \underline{27.91} & \underline{0.894} & \underline{0.110} & \underline{28.89} & \underline{0.909} & \underline{0.099} & \underline{30.66} & \underline{0.937} & \underline{0.080} & \underline{32.28} & \underline{0.953} & \underline{0.067} & \underline{33.02} & \underline{0.960} & \underline{0.061} \\
					& NeurTV & 26.31 & 0.811 & 0.133 & 27.67 & 0.842 & 0.113 & 30.05 & 0.902 & 0.086 & 31.42 & 0.919 & 0.074 & 32.58 & 0.939 & 0.064 \\
					& Ours & \textbf{28.74} & \textbf{0.917} & \textbf{0.100} & \textbf{30.35} & \textbf{0.938} & \textbf{0.083} & \textbf{32.55} & \textbf{0.957} & \textbf{0.065} & \textbf{33.90} & \textbf{0.967} & \textbf{0.055} & \textbf{34.90} & \textbf{0.972} & \textbf{0.049} \\
					\hline
					\multirow{8}{*}{\makecell[c]{\textit{Carphone}\\$(144\times 176\times 100)$}} 
					& TRLRF & 22.51 & 0.607 & 0.156 & 24.13 & 0.680 & 0.129 & 26.39 & 0.769 & 0.100 & 27.73 & 0.816 & 0.085 & 28.71 & 0.844 & 0.076 \\
					& FCTN & 23.89 & 0.663 & 0.133 & 25.45 & 0.729 & 0.111 & 27.09 & 0.792 & 0.092 & 28.22 & 0.831 & 0.081 & 29.38 & 0.861 & 0.071 \\
					& HLRTF & 24.28 & 0.661 & 0.127 & 25.71 & 0.732 & 0.108 & 28.12 & 0.817 & 0.082 & 29.63 & 0.859 & 0.069 & 30.74 & 0.885 & 0.060 \\
					& LRTFR & 24.23 & 0.702 & 0.128 & 26.33 & 0.786 & 0.100 & 27.77 & 0.837 & 0.085 & 29.32 & 0.867 & 0.071 & 30.18 & 0.890 & 0.064 \\
					& DRO-TFF & \underline{27.80} & \underline{0.856} & \underline{0.085} & \underline{28.69} & \underline{0.875} & \underline{0.076} & \underline{29.61} & \underline{0.892} & \underline{0.069} & \underline{30.67} & \underline{0.911} & \underline{0.061} & \underline{31.12} & \underline{0.918} & \underline{0.058} \\
					& NeurTV & 25.36 & 0.709 & 0.112 & 27.29 & 0.793 & 0.090 & 28.82 & 0.847 & 0.075 & 29.69 & 0.894 & 0.068 & 30.62 & 0.903 & 0.061 \\
					& Ours & \textbf{28.24} & \textbf{0.879} & \textbf{0.080} & \textbf{29.19} & \textbf{0.896} & \textbf{0.072} & \textbf{30.74} & \textbf{0.921} & \textbf{0.060} & \textbf{31.81} & \textbf{0.934} & \textbf{0.053} & \textbf{32.58} & \textbf{0.943} & \textbf{0.049} \\
					\hline
			\end{tabular}}
		\end{table*}
		
		\noindent{\bf Extended Denoising Results. }We further conduct denoising experiments on MSI and HSI data under different noise levels, where Gaussian noise with SD ranging from 0.1 to 0.3 is added independently to all pixels and spectral bands. Fig.~\ref{fig:denoising_supp} provides visual comparisons across MSI (\textit{Toys}, \textit{Balloons}) and HSI scenes (\textit{KSC}, \textit{Indian Pines}), showing that our method restores cleaner structural details and yields fewer spectral distortions. Corresponding quantitative results are summarized in Table~\ref{tab:denoising_supp}, where PSNR, SSIM, and NRMSE are reported across different SD settings.
		
		\begin{figure*}[tb]
			\renewcommand{\arraystretch}{0.5}
			\setlength\tabcolsep{0.5pt}
			\centering
			\resizebox{\linewidth}{!}{
				\begin{tabular}{ccccccccc}
					\scriptsize{PSNR/SSIM} & \scriptsize{30.29/0.734} & \scriptsize{30.91/0.807} & \scriptsize{31.89/0.821} & \scriptsize{32.36/0.870} & \scriptsize{33.11/0.902} &  \scriptsize{33.19/0.904} & \scriptsize{\textbf{33.93}/\textbf{0.921}} & \scriptsize{\textbf{Inf/1.000}} \\
					\includegraphics[width=18.5mm]{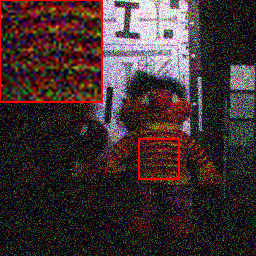} & \includegraphics[width=18.5mm]{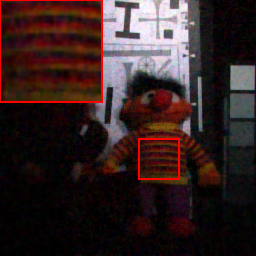} &
					\includegraphics[width=18.5mm]{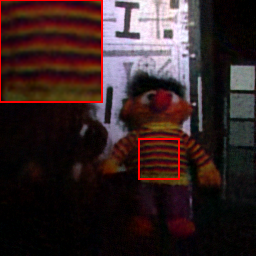} &
					\includegraphics[width=18.5mm]{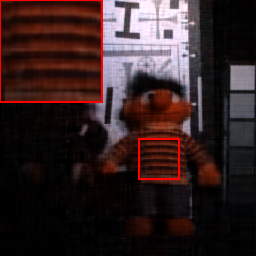} &
					\includegraphics[width=18.5mm]{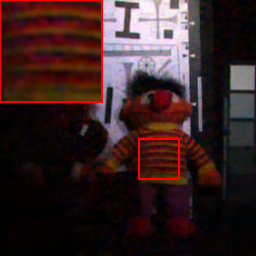} &
					\includegraphics[width=18.5mm]{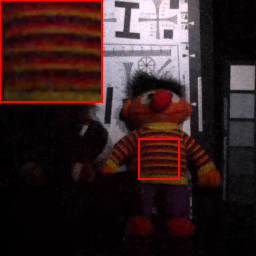} &
					\includegraphics[width=18.5mm]{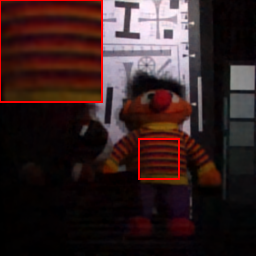} &
					\includegraphics[width=18.5mm]{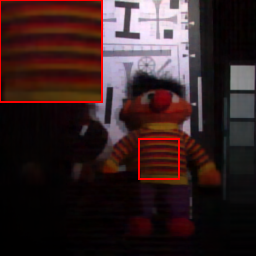} &
					\includegraphics[width=18.5mm]{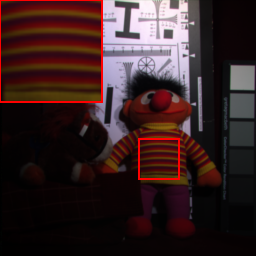} \\
					\scriptsize{PSNR/SSIM} & \scriptsize{31.54/0.712} & \scriptsize{34.73/0.888} & \scriptsize{35.18/0.918} & \scriptsize{36.03/0.910} & \scriptsize{37.21/0.945} &  \scriptsize{37.30/0.947} & \scriptsize{\textbf{37.83}/\textbf{0.955}} & \scriptsize{\textbf{Inf/1.000}} \\
					\includegraphics[width=18.5mm]{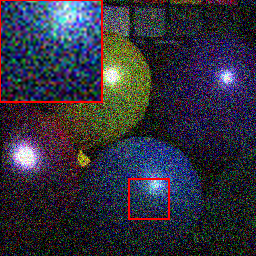} & \includegraphics[width=18.5mm]{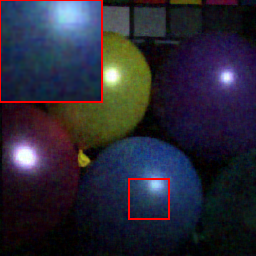} &
					\includegraphics[width=18.5mm]{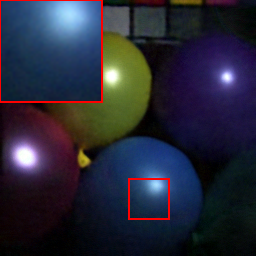} &
					\includegraphics[width=18.5mm]{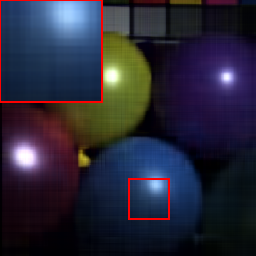} &
					\includegraphics[width=18.5mm]{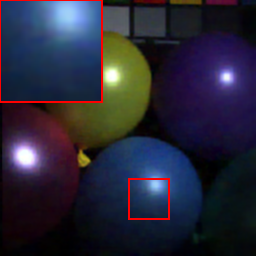} &
					\includegraphics[width=18.5mm]{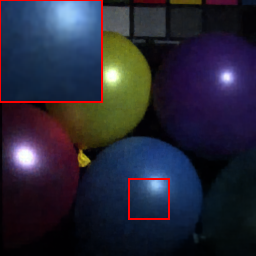} &
					\includegraphics[width=18.5mm]{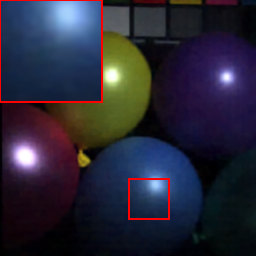} &
					\includegraphics[width=18.5mm]{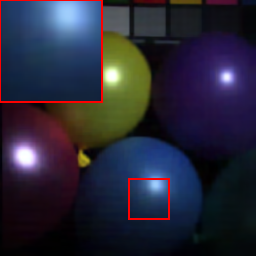} &
					\includegraphics[width=18.5mm]{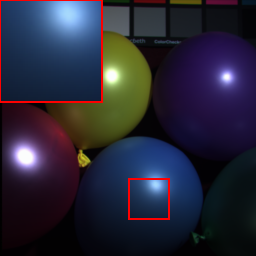} \\
					\scriptsize{PSNR/SSIM} & \scriptsize{32.58/0.887} & \scriptsize{33.32/0.902} & \scriptsize{33.25/0.907} & \scriptsize{33.52/0.889} & \scriptsize{34.79/0.930} &  \scriptsize{34.45/0.909} & \scriptsize{\textbf{35.56}/\textbf{0.932}} & \scriptsize{\textbf{Inf/1.000}} \\
					\includegraphics[width=18.5mm]{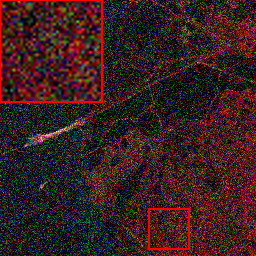} & \includegraphics[width=18.5mm]{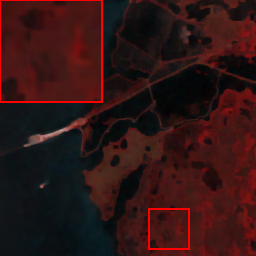} &
					\includegraphics[width=18.5mm]{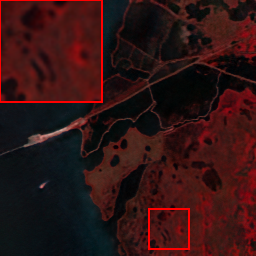} &
					\includegraphics[width=18.5mm]{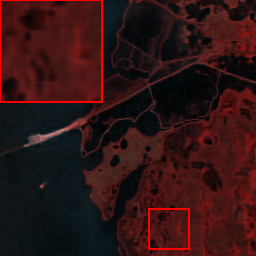} &
					\includegraphics[width=18.5mm]{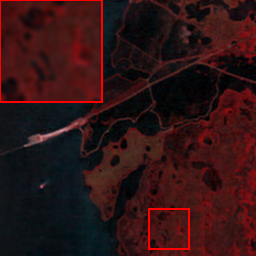} &
					\includegraphics[width=18.5mm]{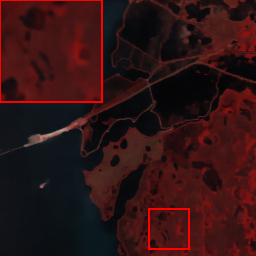} &
					\includegraphics[width=18.5mm]{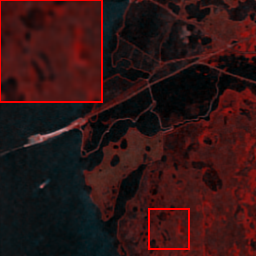} &
					\includegraphics[width=18.5mm]{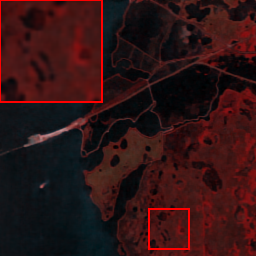} &
					\includegraphics[width=18.5mm]{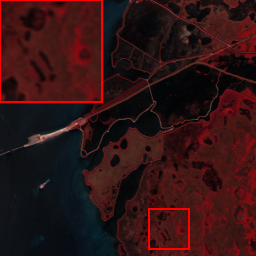} \\
					\scriptsize{PSNR/SSIM} & \scriptsize{33.77/0.894} & \scriptsize{33.31/0.858} & \scriptsize{34.24/0.890} & \scriptsize{33.91/0.873} & \scriptsize{34.44/0.911} &  \scriptsize{35.16/0.909} & \scriptsize{\textbf{35.67}/\textbf{0.914}} & \scriptsize{\textbf{Inf/1.000}} \\
					\includegraphics[width=18.5mm]{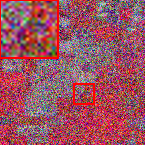} & \includegraphics[width=18.5mm]{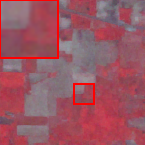} &
					\includegraphics[width=18.5mm]{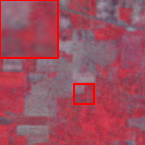} &
					\includegraphics[width=18.5mm]{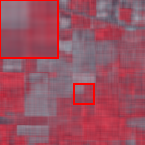} &
					\includegraphics[width=18.5mm]{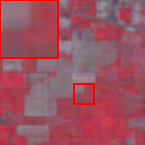} &
					\includegraphics[width=18.5mm]{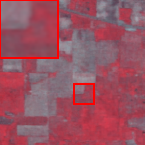} &
					\includegraphics[width=18.5mm]{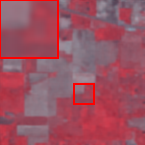} &
					\includegraphics[width=18.5mm]{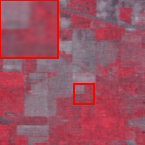} &
					\includegraphics[width=18.5mm]{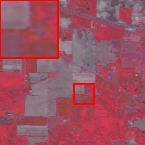} \\
					\scriptsize{Observed} & \scriptsize LRTDTV & \scriptsize DeepTensor & \scriptsize HLRTF & \scriptsize LRTFR & \scriptsize DRO-TFF & \scriptsize NeurTV & \scriptsize Ours & \scriptsize{Ground truth}
			\end{tabular}}
			\caption{Visual denoising results on MSI \textit{Toy} and \textit{Balloons}, and HSI \textit{KSC} and \textit{Indian Pines} with $\text{SD}=0.2$.}
			\label{fig:denoising_supp}
		\end{figure*}

		\begin{table*}[tb]
			\renewcommand{\arraystretch}{1.15}
			\setlength\tabcolsep{2.5pt}
			\small
			\caption{Quantitative comparison for denoising under different noise levels. PSNR, SSIM, and NRMSE are reported.}
			\label{tab:denoising_supp}
			\centering
			\resizebox{\linewidth}{!}{
				\begin{tabular}{c|l|ccc|ccc|ccc|ccc|ccc}
					\hline
					\multirow{2}{*}{Data} & \multirow{2}{*}{Method} 
					& \multicolumn{3}{c|}{$\text{SD}=0.1$} & \multicolumn{3}{c|}{$\text{SD}=0.15$} & \multicolumn{3}{c|}{$\text{SD}=0.2$} & \multicolumn{3}{c|}{$\text{SD}=0.25$} & \multicolumn{3}{c}{$\text{SD}=0.3$} \\
					\cline{3-17}
					& & PSNR & SSIM & NRMSE & PSNR & SSIM & NRMSE & PSNR & SSIM & NRMSE & PSNR & SSIM & NRMSE & PSNR & SSIM & NRMSE \\
					\hline
					\multirow{7}{*}{\makecell[c]{\textit{Toys}\\$(256\times 256\times 31)$}} 
					& LRTDTV & 34.85 & 0.906 & 0.062 & 32.17 & 0.824 & 0.084 & 30.29 & 0.734 & 0.104 & 28.79 & 0.653 & 0.124 & 27.65 & 0.581 & 0.141 \\
					& DeepTensor & 35.11 & 0.902 & 0.060 & 32.50 & 0.875 & 0.081 & 30.91 & 0.807 & 0.097 & 29.50 & 0.760 & 0.114 & 28.62 & 0.672 & 0.126 \\
					& HLRTF & 35.70 & 0.912 & 0.056 & 33.30 & 0.883 & 0.074 & 31.89 & 0.821 & 0.087 & 30.78 & 0.768 & 0.098 & 29.46 & 0.698 & 0.115 \\
					& LRTFR & 36.01 & 0.936 & 0.054 & 34.01 & 0.918 & 0.068 & 32.36 & 0.870 & 0.082 & 31.36 & 0.869 & 0.092 & 30.60 & 0.829 & 0.101 \\
					& DRO-TFF & \textbf{37.34} & \underline{0.954} & \textbf{0.046} & 35.07 & \underline{0.933} & \underline{0.060} & 33.11 & 0.902 & \underline{0.075} & \underline{32.24} & \underline{0.887} & \underline{0.083} & 31.10 & \underline{0.875} & 0.095 \\
					& NeurTV & 36.71 & 0.951 & 0.050 & \underline{35.23} & 0.920 & \textbf{0.059} & \underline{33.19} & \underline{0.904} & \underline{0.075} & 32.17 & 0.876 & 0.084 & \underline{31.51} & 0.863 & \underline{0.091} \\
					& Ours & \underline{37.25} & \textbf{0.958} & \underline{0.047} & \textbf{35.30} & \textbf{0.941} & \textbf{0.059} & \textbf{33.93} & \textbf{0.921} & \textbf{0.069} & \textbf{32.89} & \textbf{0.902} & \textbf{0.077} & \textbf{32.04} & \textbf{0.889} & \textbf{0.085} \\
					\hline
					\multirow{7}{*}{\makecell[c]{\textit{Face}\\$(256\times 256\times 31)$}} 
					& LRTDTV & 36.85 & 0.883 & 0.090 & 34.11 & 0.788 & 0.123 & 32.26 & 0.701 & 0.152 & 30.55 & 0.615 & 0.185 & 29.24 & 0.528 & 0.216 \\
					& DeepTensor & 37.36 & 0.926 & 0.014 & 35.51 & 0.904 & 0.017 & 34.06 & 0.859 & 0.020 & 32.62 & 0.814 & 0.023 & 31.87 & 0.732 & 0.026 \\
					& HLRTF & 38.49 & 0.922 & 0.074 & 36.40 & 0.882 & 0.095 & 34.67 & 0.835 & 0.115 & 33.37 & 0.792 & 0.134 & 32.30 & 0.754 & 0.152 \\
					& LRTFR & 39.09 & 0.935 & 0.069 & 36.81 & 0.896 & 0.090 & 35.13 & 0.858 & 0.109 & 33.91 & 0.821 & 0.126 & 33.01 & 0.808 & 0.140 \\
					& DRO-TFF & 40.00 & 0.964 & 0.062 & 37.94 & 0.937 & 0.079 & 36.37 & 0.920 & 0.095 & 35.20 & 0.900 & 0.109 & 34.71 & 0.887 & 0.115 \\
					& NeurTV & \underline{40.51} & \underline{0.968} & \underline{0.059} & \underline{38.46} & \underline{0.949} & \underline{0.075} & \underline{36.78} & \underline{0.927} & \underline{0.091} & \underline{35.77} & \underline{0.914} & \underline{0.102} & \underline{35.22} & \underline{0.902} & \underline{0.108} \\
					& Ours & \textbf{40.80} & \textbf{0.970} & \textbf{0.057} & \textbf{39.16} & \textbf{0.957} & \textbf{0.069} & \textbf{37.89} & \textbf{0.945} & \textbf{0.080} & \textbf{36.92} & \textbf{0.933} & \textbf{0.089} & \textbf{36.12} & \textbf{0.921} & \textbf{0.098} \\
					\hline
					\multirow{7}{*}{\makecell[c]{\textit{Washington DC}\\$(256\times 256\times 191)$}} 
					& LRTDTV & 34.32 & 0.924 & 0.109 & 32.83 & 0.896 & 0.129 & 31.71 & 0.869 & 0.147 & 30.84 & 0.845 & 0.163 & 30.12 & 0.823 & 0.177 \\
					& DeepTensor & 36.01 & 0.955 & 0.090 & 33.80 & 0.917 & 0.116 & 32.32 & 0.891 & 0.137 & 30.92 & 0.849 & 0.161 & 30.16 & 0.816 & 0.176 \\
					& HLRTF & 36.27 & 0.959 & 0.087 & 34.17 & 0.927 & 0.111 & 32.75 & 0.913 & 0.131 & 31.60 & 0.892 & 0.149 & 30.85 & 0.879 & 0.162 \\
					& LRTFR & 35.78 & 0.953 & 0.092 & 34.22 & 0.934 & 0.110 & 33.02 & 0.914 & 0.127 & 32.09 & 0.895 & 0.141 & 31.34 & 0.880 & 0.154 \\
					& DRO-TFF & 36.65 & 0.948 & 0.083 & 35.26 & \underline{0.934} & 0.098 & \underline{34.26} & \underline{0.919} & \underline{0.110} & \underline{33.35} & \underline{0.904} & \underline{0.122} & \underline{32.57} & \underline{0.885} & \underline{0.133} \\
					& NeurTV & \underline{37.06} & \underline{0.960} & \underline{0.079} & \underline{35.29} & 0.929 & \underline{0.097} & 33.86 & 0.910 & 0.115 & 32.79 & 0.891 & 0.130 & 31.92 & 0.869 & 0.144 \\
					& Ours & \textbf{37.96} & \textbf{0.971} & \textbf{0.072} & \textbf{36.32} & \textbf{0.957} & \textbf{0.087} & \textbf{35.08} & \textbf{0.945} & \textbf{0.100} & \textbf{34.16} & \textbf{0.932} & \textbf{0.111} & \textbf{33.37} & \textbf{0.918} & \textbf{0.122} \\
					\hline
					\multirow{7}{*}{\makecell[c]{\textit{Salinas}\\$(217\times 217\times 224)$}} 
					& LRTDTV & 40.13 & 0.957 & 0.053 & 38.19 & 0.938 & 0.066 & 36.82 & 0.920 & 0.077 & 35.83 & 0.907 & 0.087 & 35.01 & 0.891 & 0.095 \\
					& DeepTensor & 40.25 & 0.952 & 0.052 & 38.69 & 0.938 & 0.062 & 37.05 & 0.921 & 0.075 & 36.07 & 0.902 & 0.084 & 35.47 & 0.896 & 0.090 \\
					& HLRTF & 40.20 & 0.958 & 0.052 & 38.85 & 0.943 & 0.061 & 37.50 & 0.928 & 0.071 & 36.52 & 0.918 & 0.080 & 36.06 & 0.907 & 0.084 \\
					& LRTFR & 40.50 & 0.963 & 0.051 & 39.16 & 0.953 & 0.059 & 37.98 & 0.943 & 0.068 & 36.98 & 0.932 & 0.076 & 36.18 & 0.921 & 0.083 \\
					& DRO-TFF & 40.39 & 0.960 & 0.051 & 39.36 & 0.955 & 0.058 & 38.44 & 0.949 & 0.064 & 37.79 & 0.936 & 0.069 & 37.32 & \underline{0.937} & 0.073 \\
					& NeurTV & \underline{41.66} & \underline{0.966} & \underline{0.044} & \underline{40.25} & \underline{0.961} & \underline{0.052} & \underline{39.64} & \underline{0.954} & \underline{0.056} & \underline{38.61} & \underline{0.939} & \underline{0.063} & \underline{37.65} & 0.930 & \underline{0.070} \\
					& Ours & \textbf{42.77} & \textbf{0.973} & \textbf{0.040} & \textbf{41.23} & \textbf{0.965} & \textbf{0.046} & \textbf{40.19} & \textbf{0.955} & \textbf{0.052} & \textbf{39.37} & \textbf{0.947} & \textbf{0.058} & \textbf{38.60} & \textbf{0.938} & \textbf{0.063} \\
					\hline
			\end{tabular}}
		\end{table*}
		
		\noindent{\bf Extended Super-Resolution Results. }To further demonstrate the generalization ability of RepTRFD in recovering fine structural details under high upscaling factors, we provide additional $\times 4$ visual super-resolution results on the DIV2K dataset, as shown in Fig.~\ref{fig:super_resolution_supp}. Compared with both INR (PEMLP~\cite{mildenhall2021nerf}, SIREN~\cite{sitzmann2020implicit}, Gauss~\cite{ramasinghe2022beyond}, WIRE~\cite{saragadam2023wire}, and FINER~\cite{liu2024finer}) and tensor functional representation methods (LRTFR~\cite{luo2023low}), our approach produces sharper edges, cleaner texture patterns, and fewer artifacts. The quantitative results (PSNR/SSIM) shown above each image further verify the visual comparisons, where RepTRFD consistently achieves the best performance across all samples.
		
		\begin{figure*}[tb]
			\renewcommand{\arraystretch}{0.5}
			\setlength\tabcolsep{0.5pt}
			\centering
			\resizebox{\linewidth}{!}{
				\begin{tabular}{ccccccccc}
					\scriptsize{21.57/0.585} & \scriptsize{26.25/0.743} & \scriptsize{26.15/0.723} & \scriptsize{27.05/0.808} & \scriptsize{26.80/0.751} &  \scriptsize{25.15/0.684} & \scriptsize{\textbf{27.38}/\textbf{0.852}} & \scriptsize{PSNR/SSIM} \\
					\includegraphics[width=18.5mm]{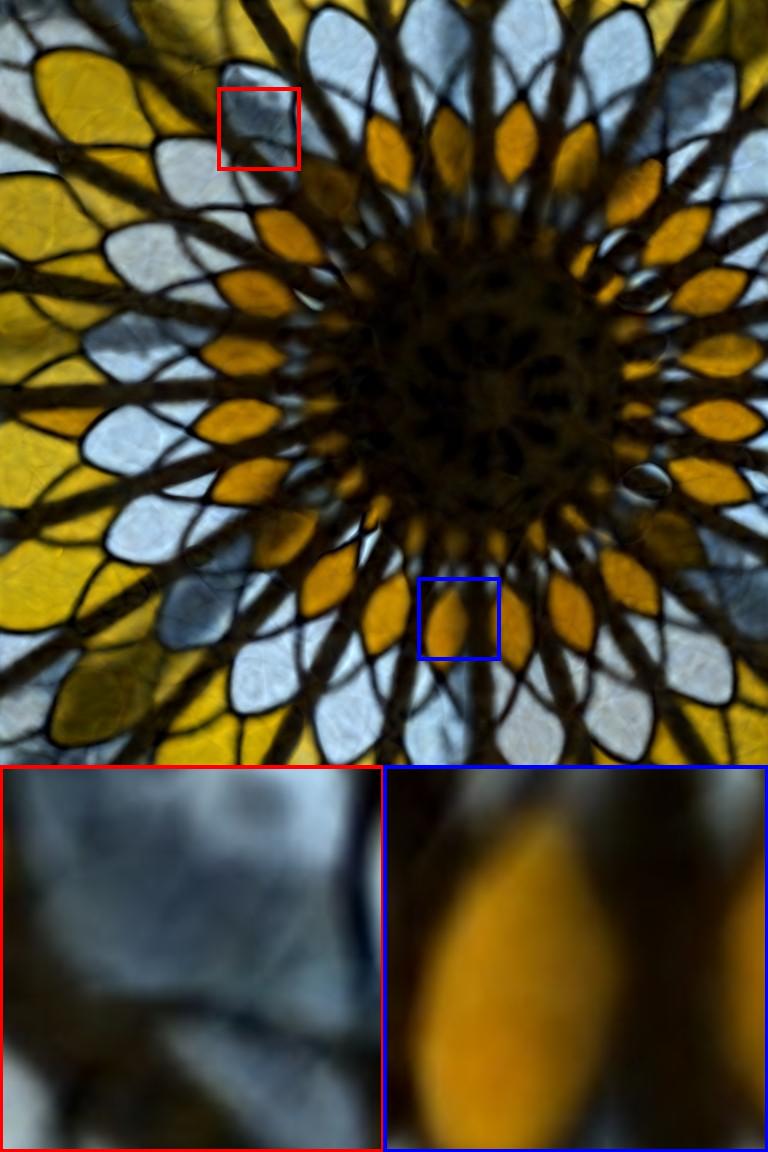} &
					\includegraphics[width=18.5mm]{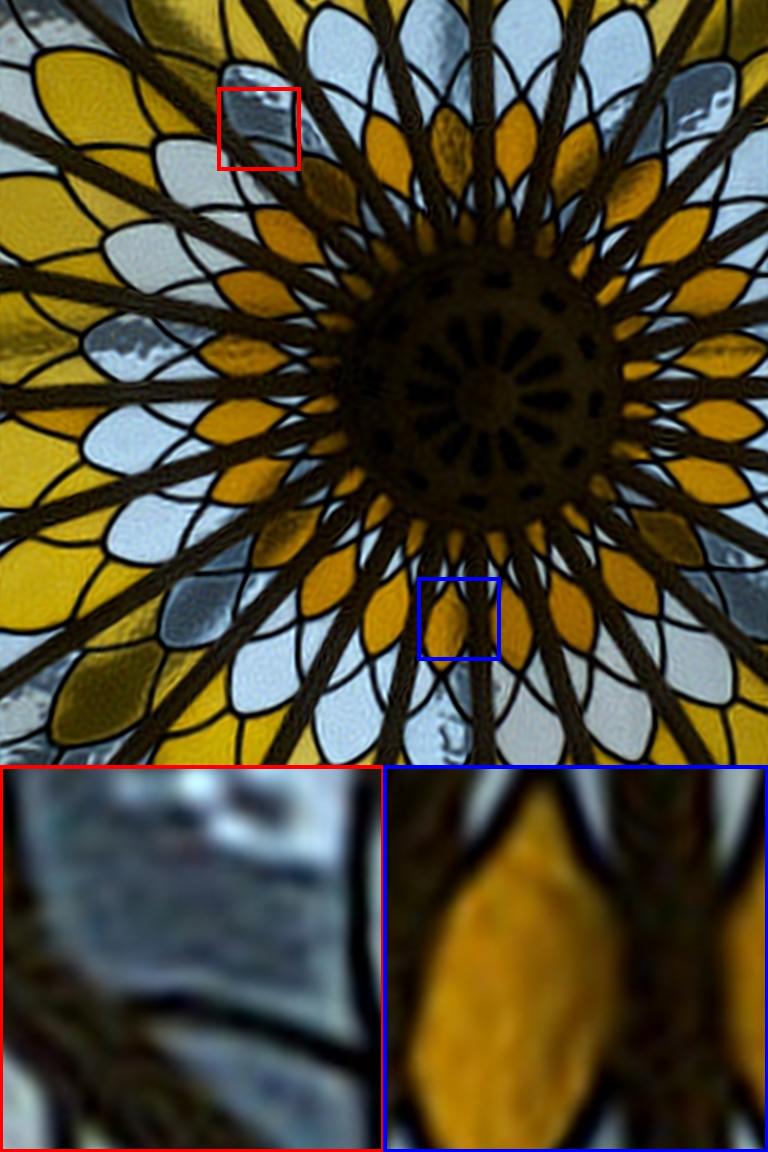} &
					\includegraphics[width=18.5mm]{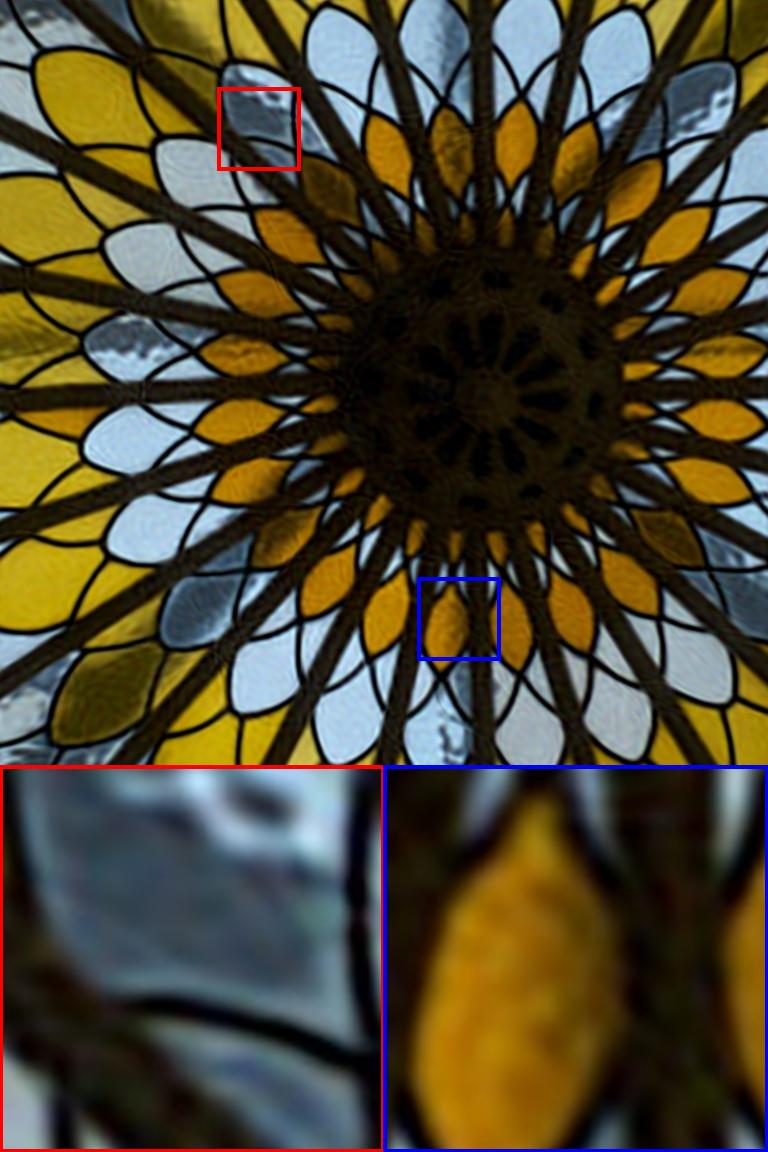} &
					\includegraphics[width=18.5mm]{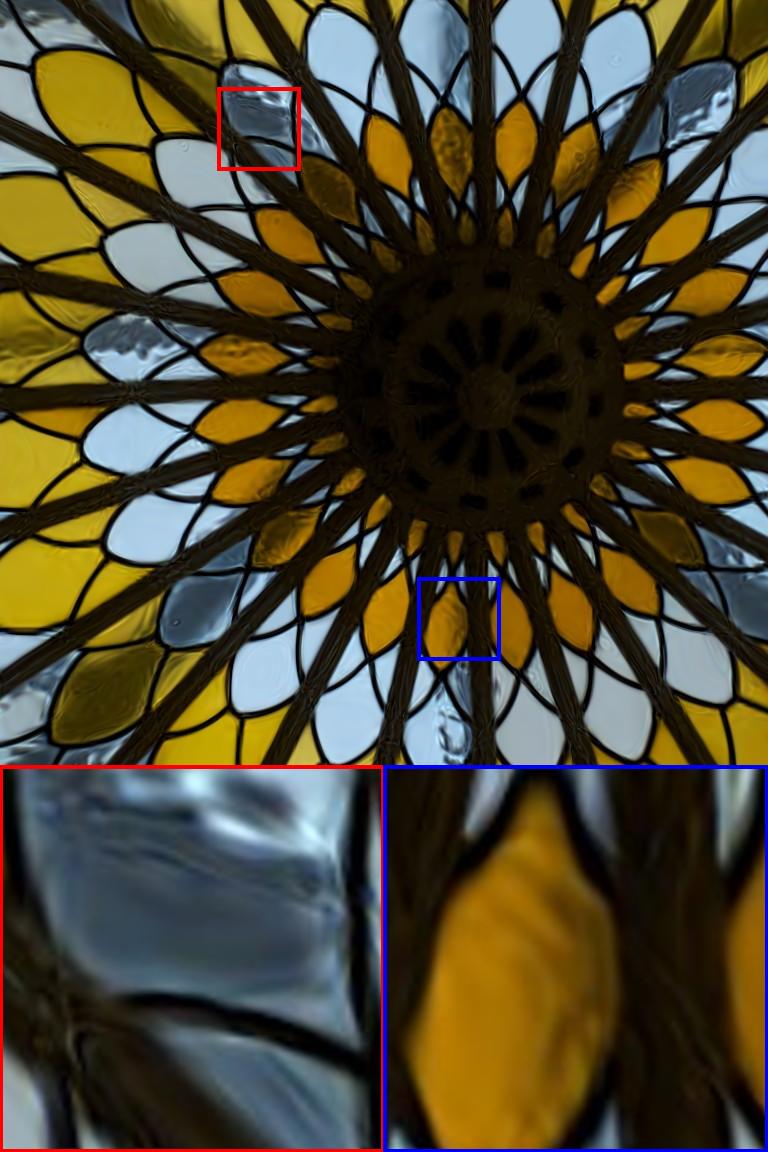} &
					\includegraphics[width=18.5mm]{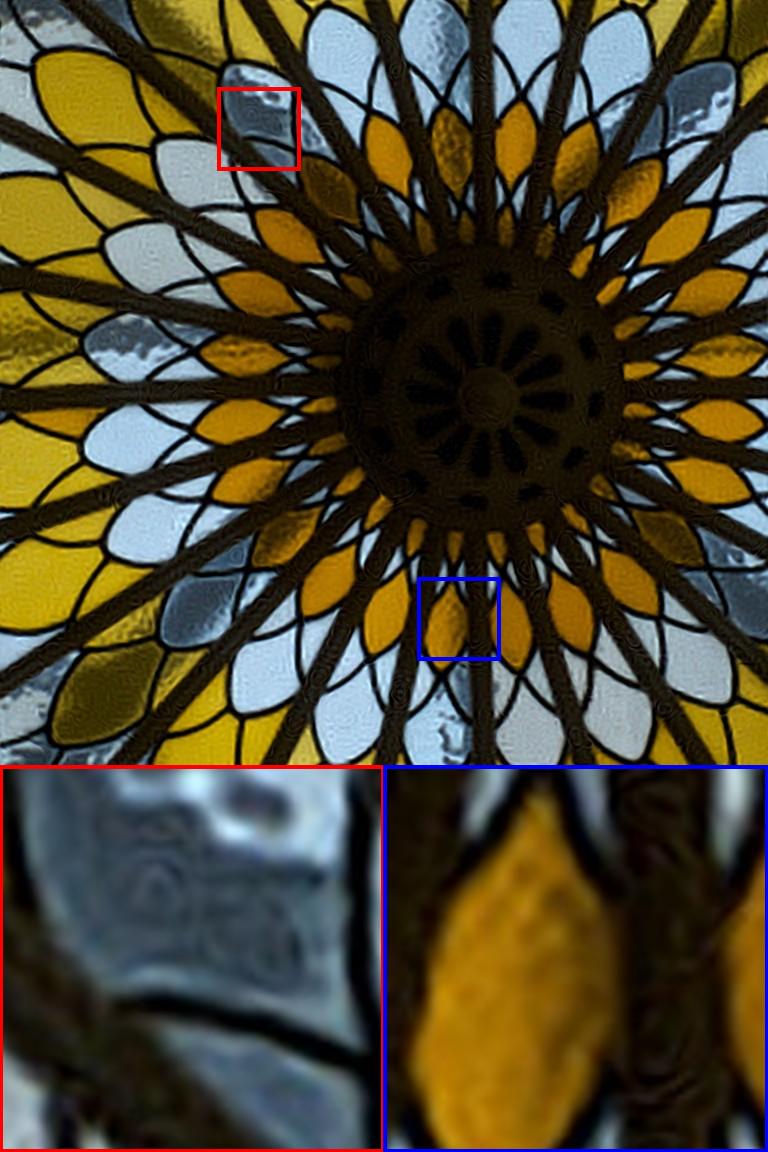} &
					\includegraphics[width=18.5mm]{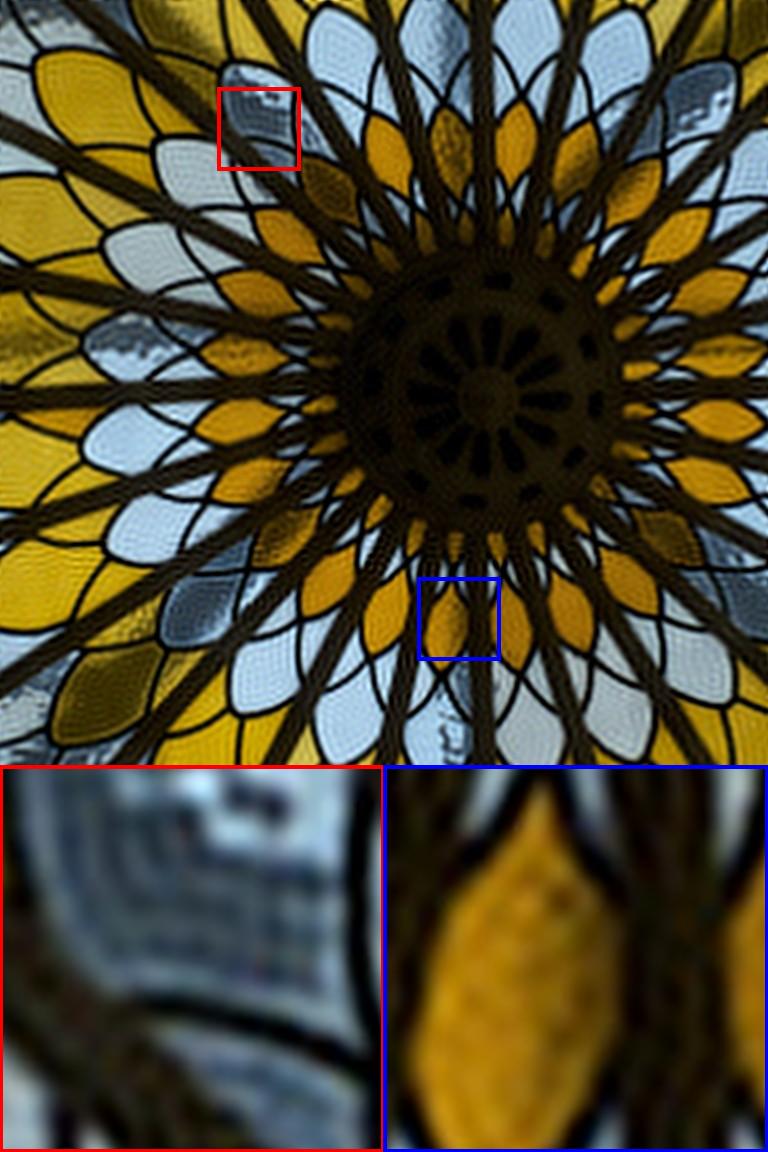} &
					\includegraphics[width=18.5mm]{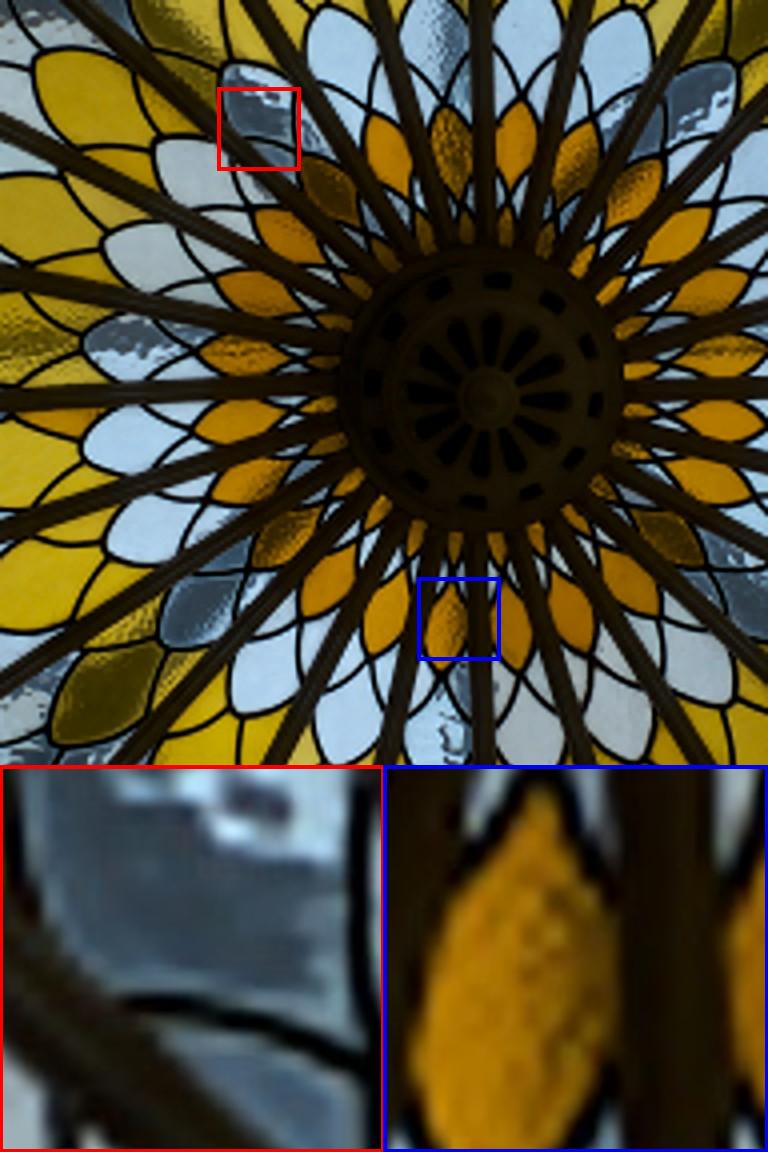} &
					\includegraphics[width=18.5mm]{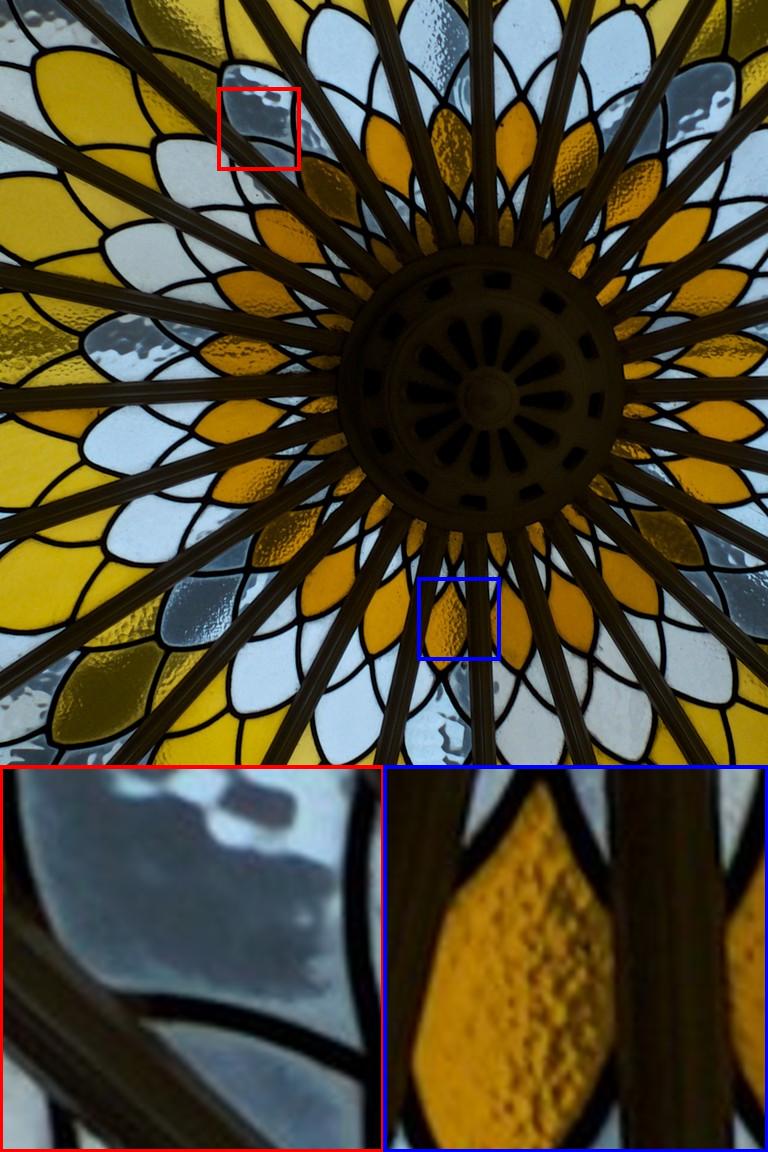} \\
					\scriptsize{26.99/0.724} & \scriptsize{29.69/0.804} & \scriptsize{29.21/0.789} & \scriptsize{30.42/0.830} & \scriptsize{30.19/0.800} &  \scriptsize{29.14/0.795} & \scriptsize{\textbf{30.88}/\textbf{0.864}} & \scriptsize{PSNR/SSIM} \\
					\includegraphics[width=18.5mm]{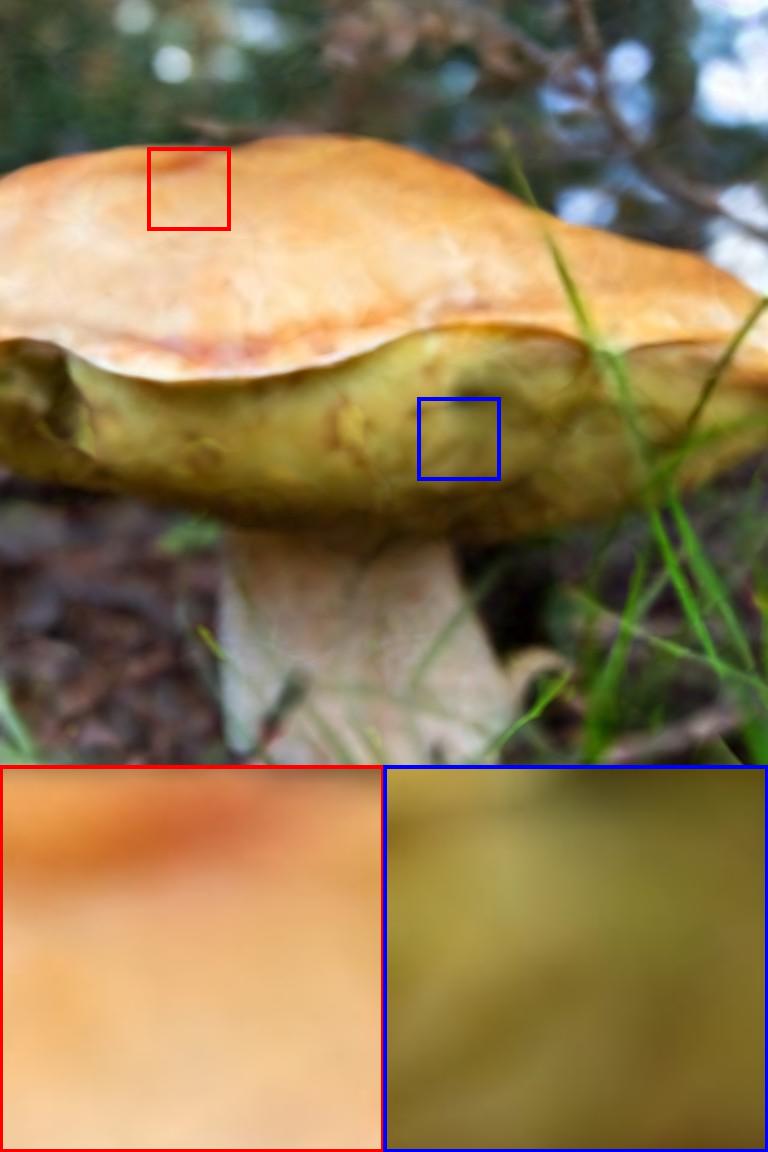} &
					\includegraphics[width=18.5mm]{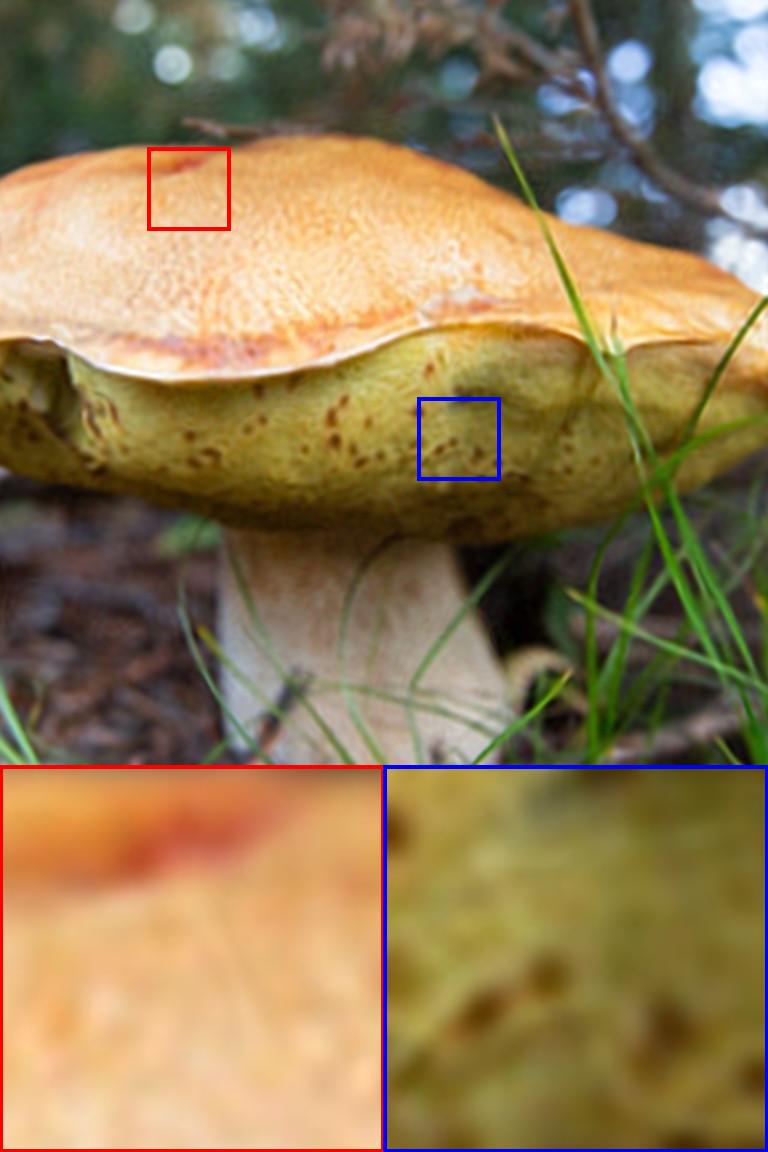} &
					\includegraphics[width=18.5mm]{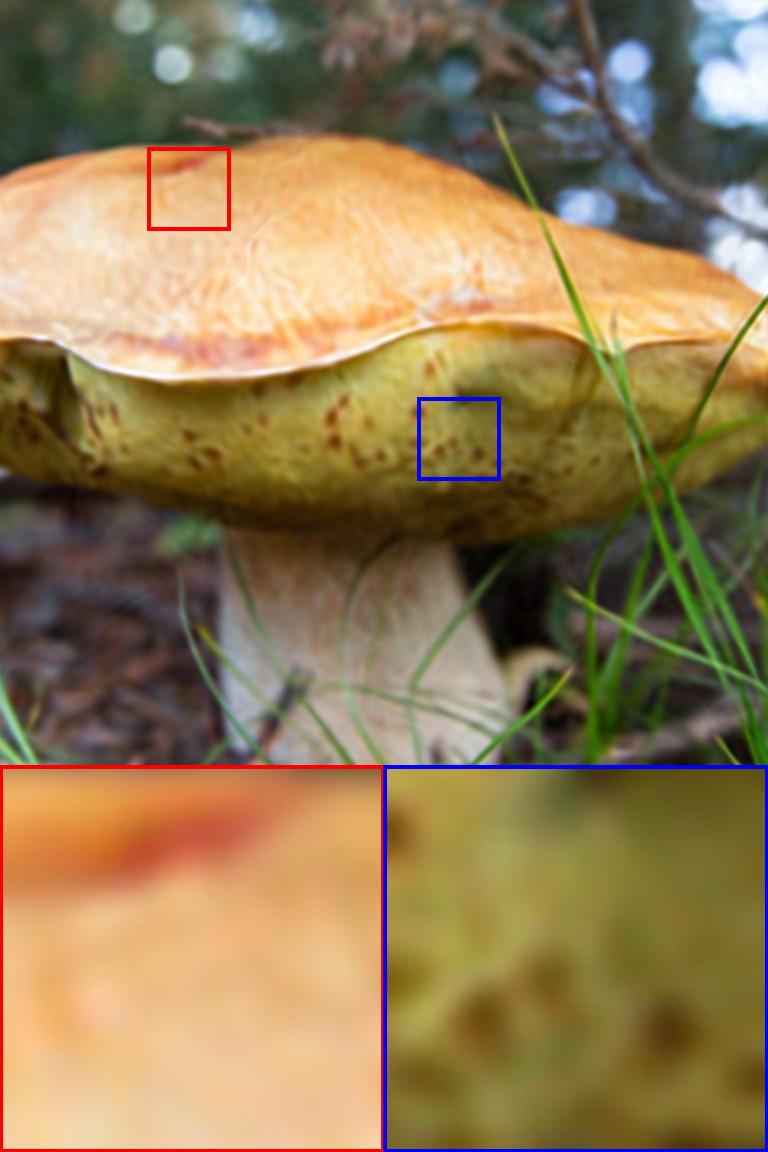} &
					\includegraphics[width=18.5mm]{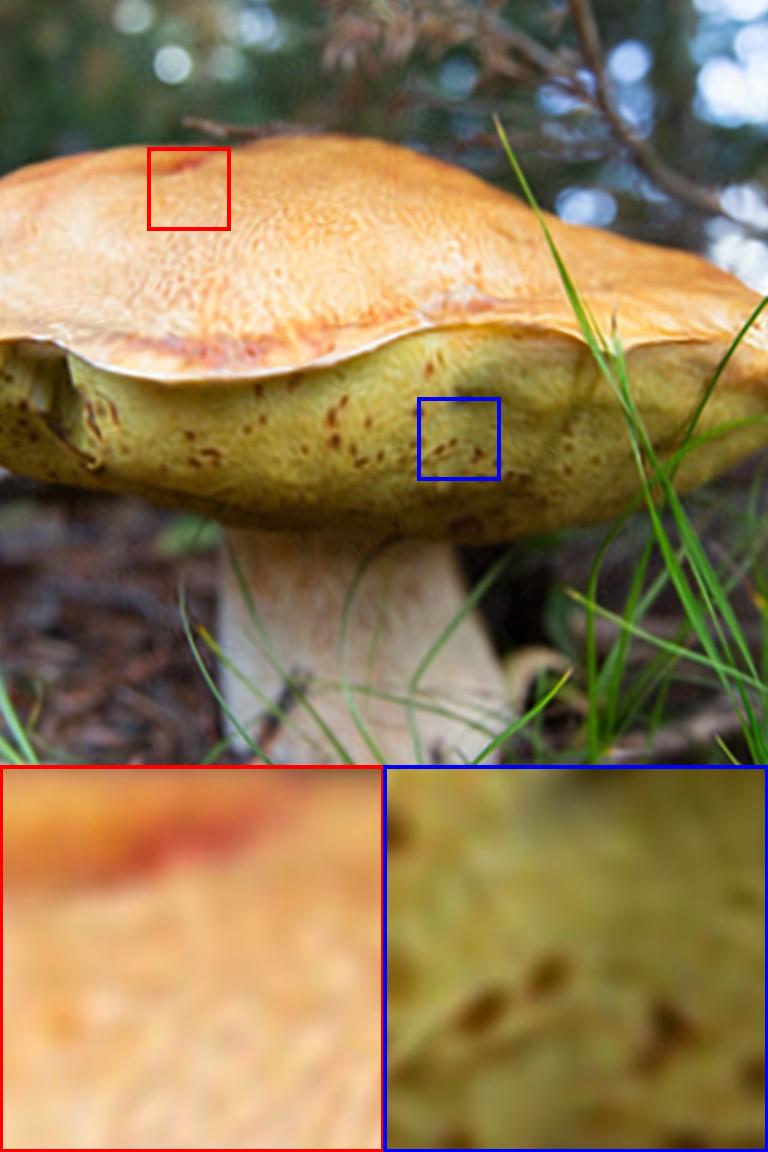} &
					\includegraphics[width=18.5mm]{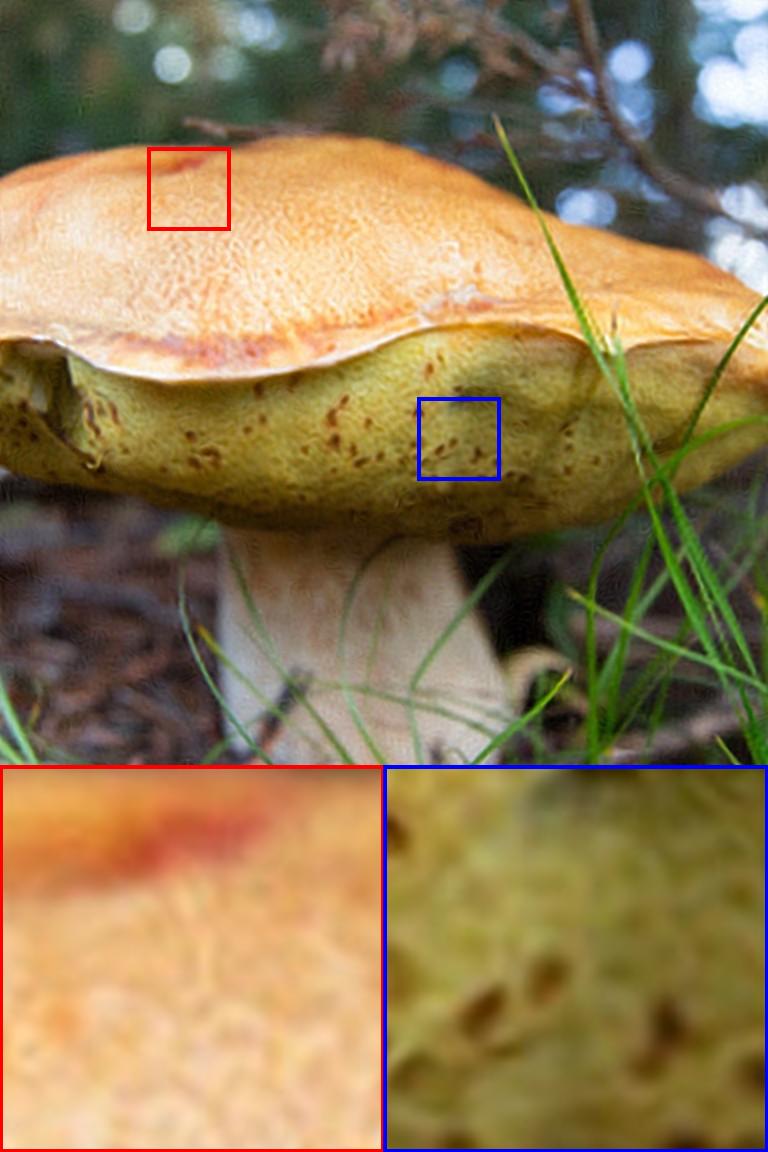} &
					\includegraphics[width=18.5mm]{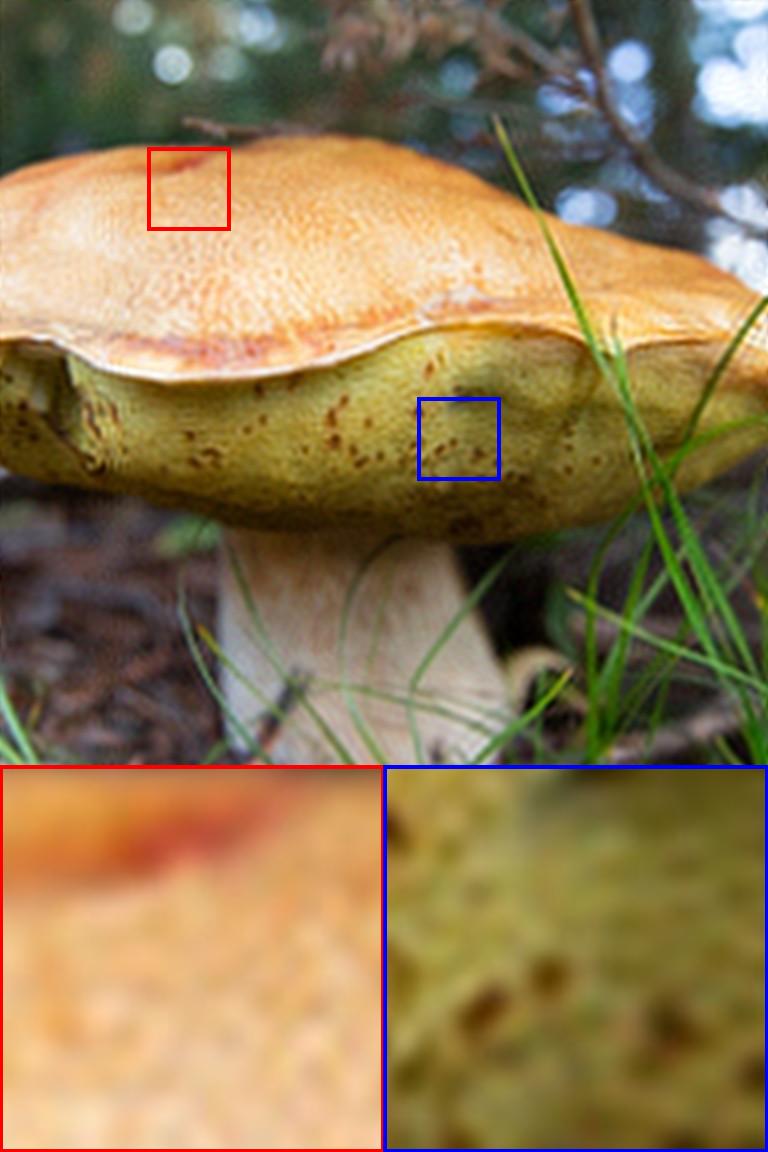} &
					\includegraphics[width=18.5mm]{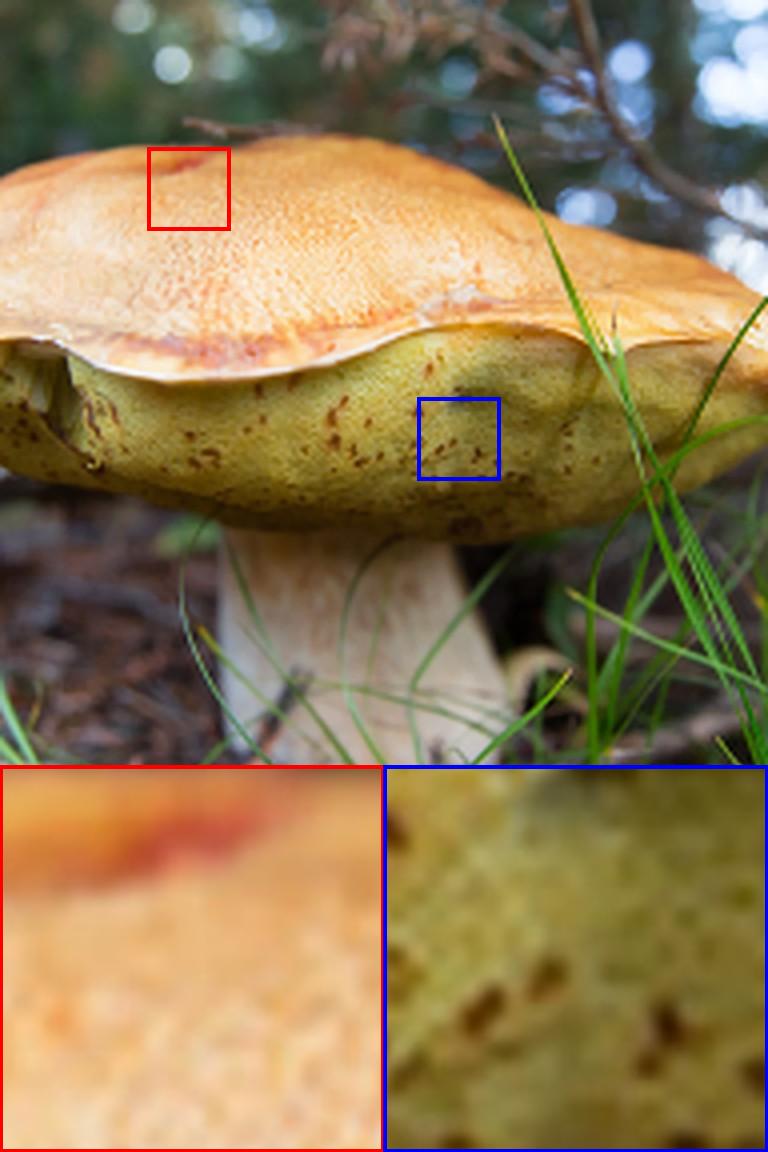} &
					\includegraphics[width=18.5mm]{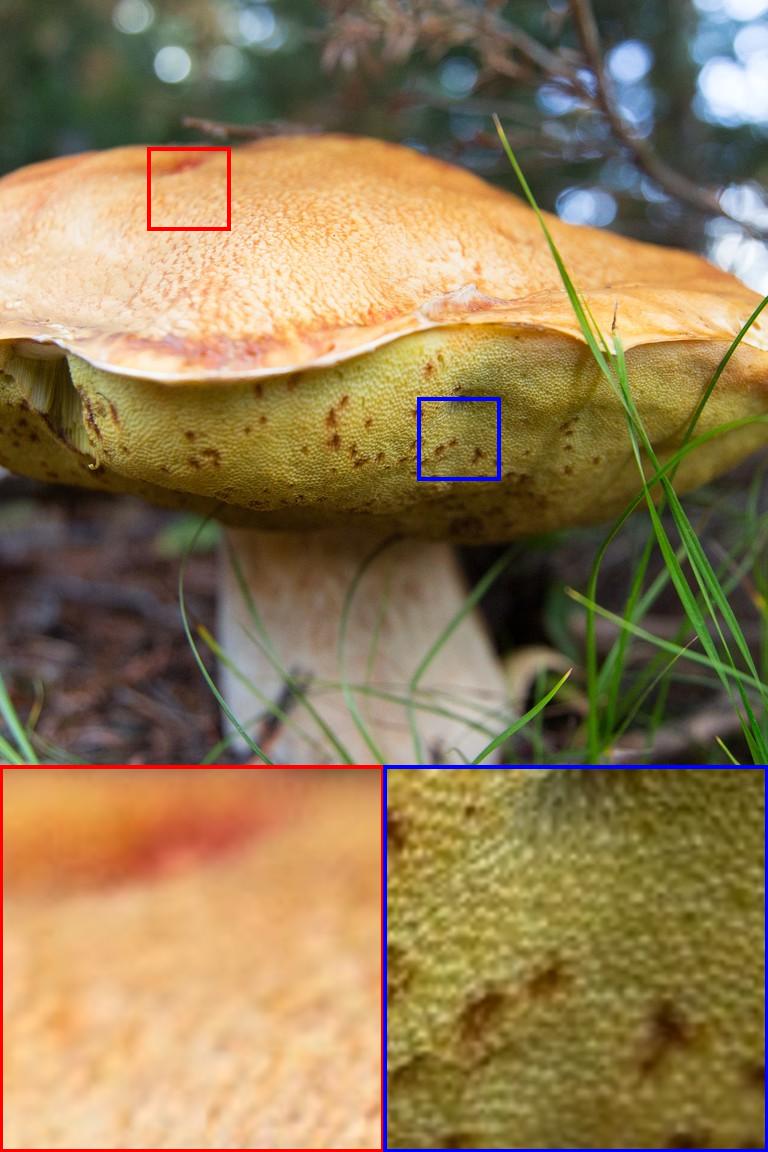} \\
					\scriptsize{22.30/0.588} & \scriptsize{26.79/0.793} & \scriptsize{26.14/0.763} & \scriptsize{27.71/0.828} & \scriptsize{27.12/0.788} &  \scriptsize{26.64/0.786} & \scriptsize{\textbf{28.09}/\textbf{0.868}} & \scriptsize{PSNR/SSIM} \\
					\includegraphics[width=18.5mm]{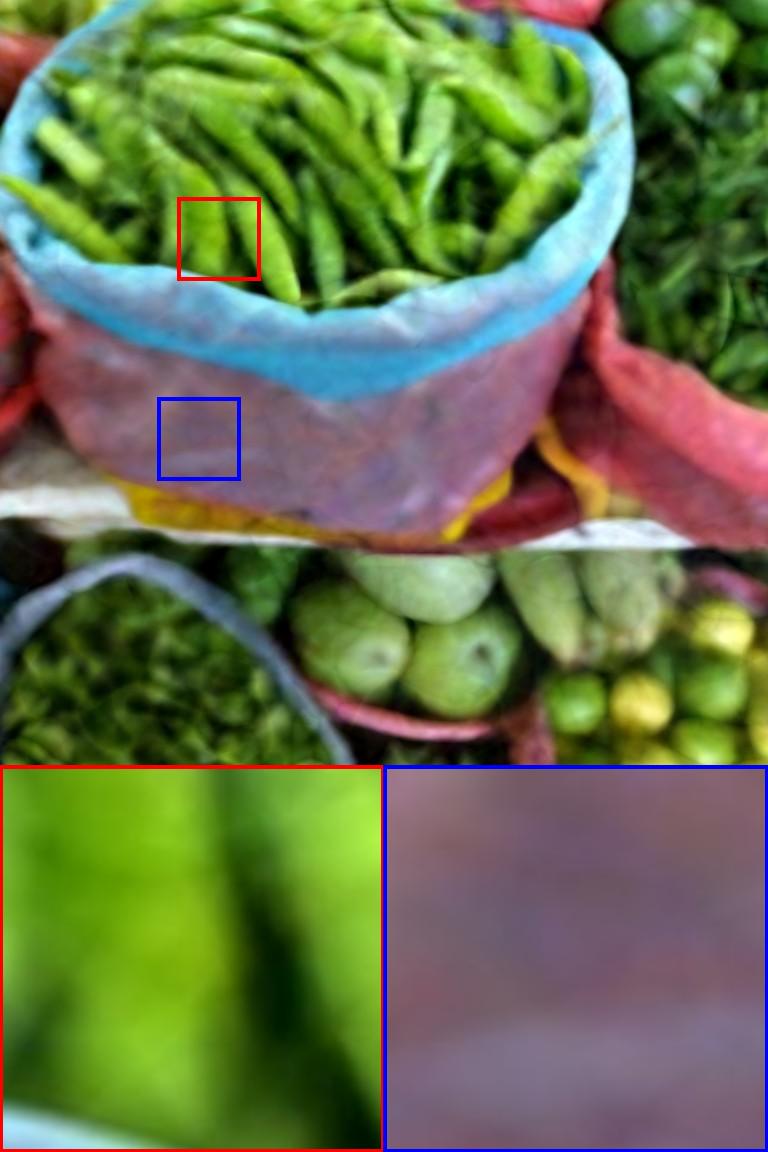} &
					\includegraphics[width=18.5mm]{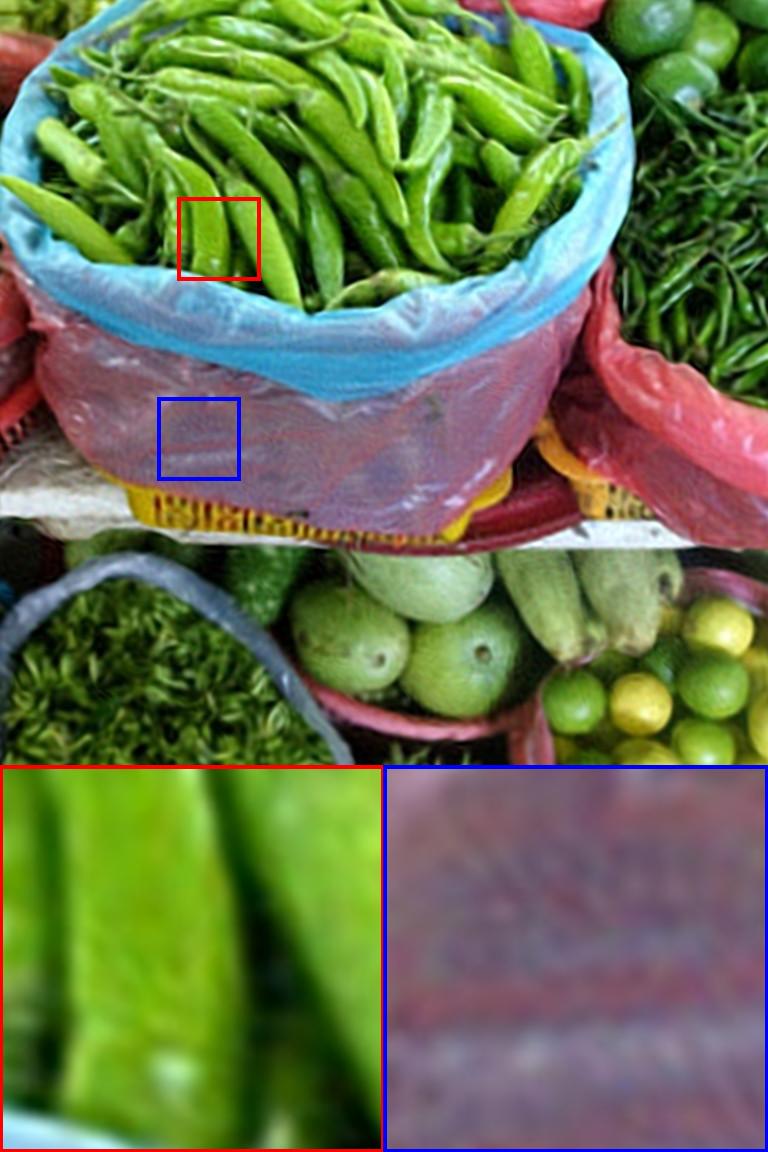} &
					\includegraphics[width=18.5mm]{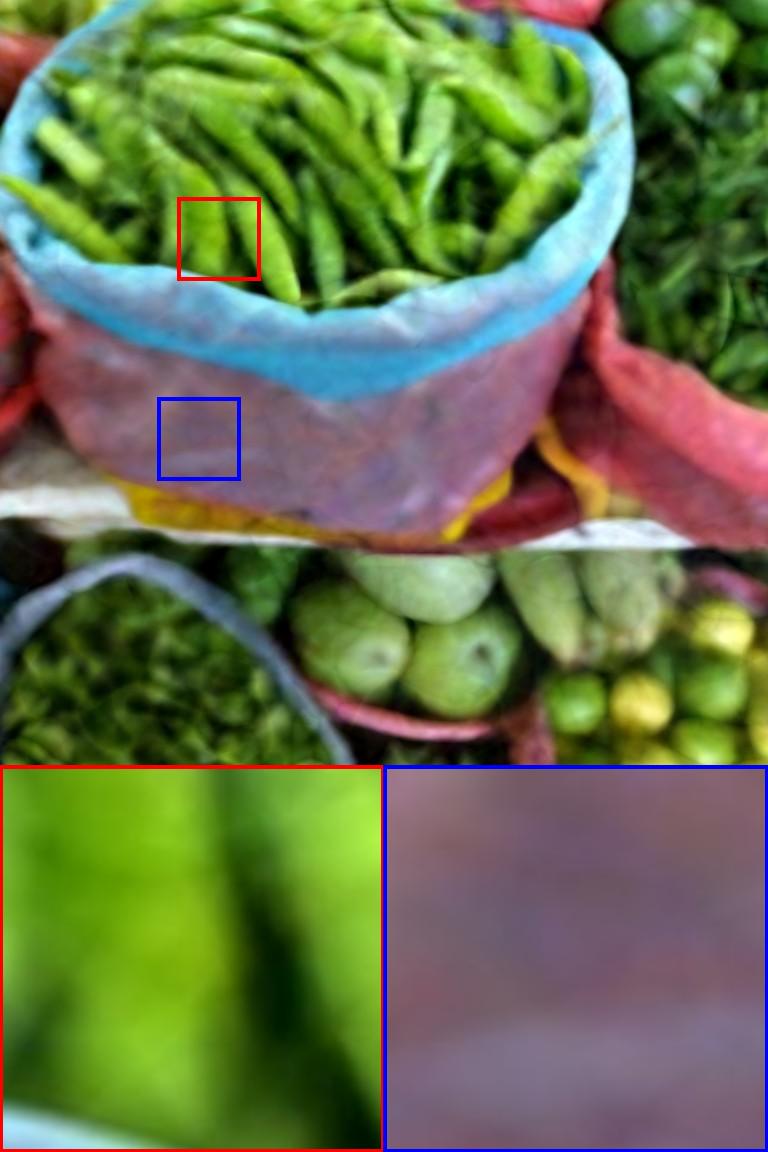} &
					\includegraphics[width=18.5mm]{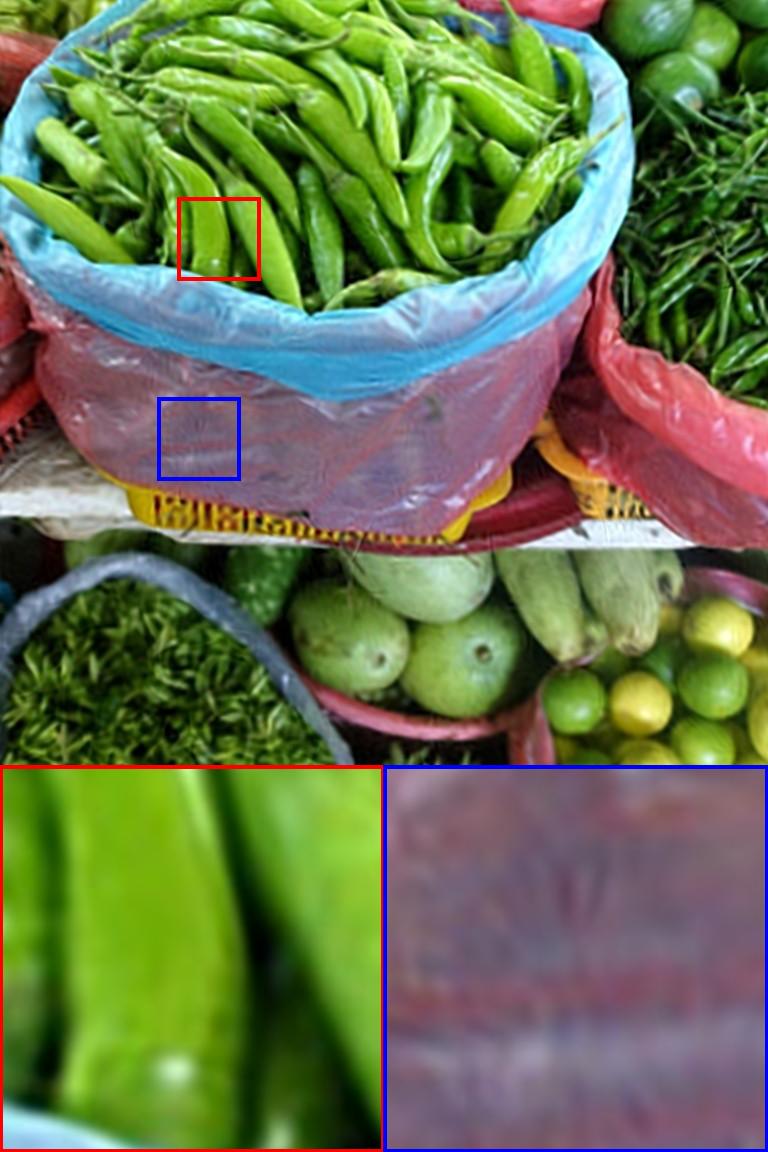} &
					\includegraphics[width=18.5mm]{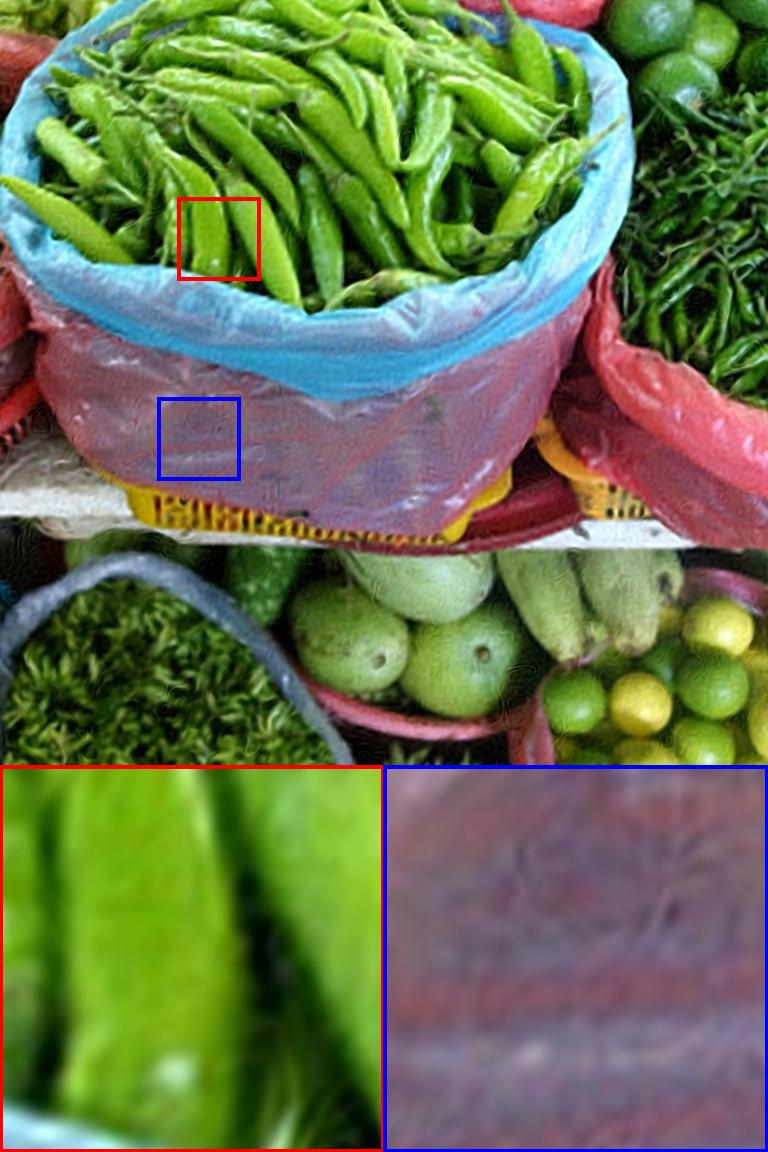} &
					\includegraphics[width=18.5mm]{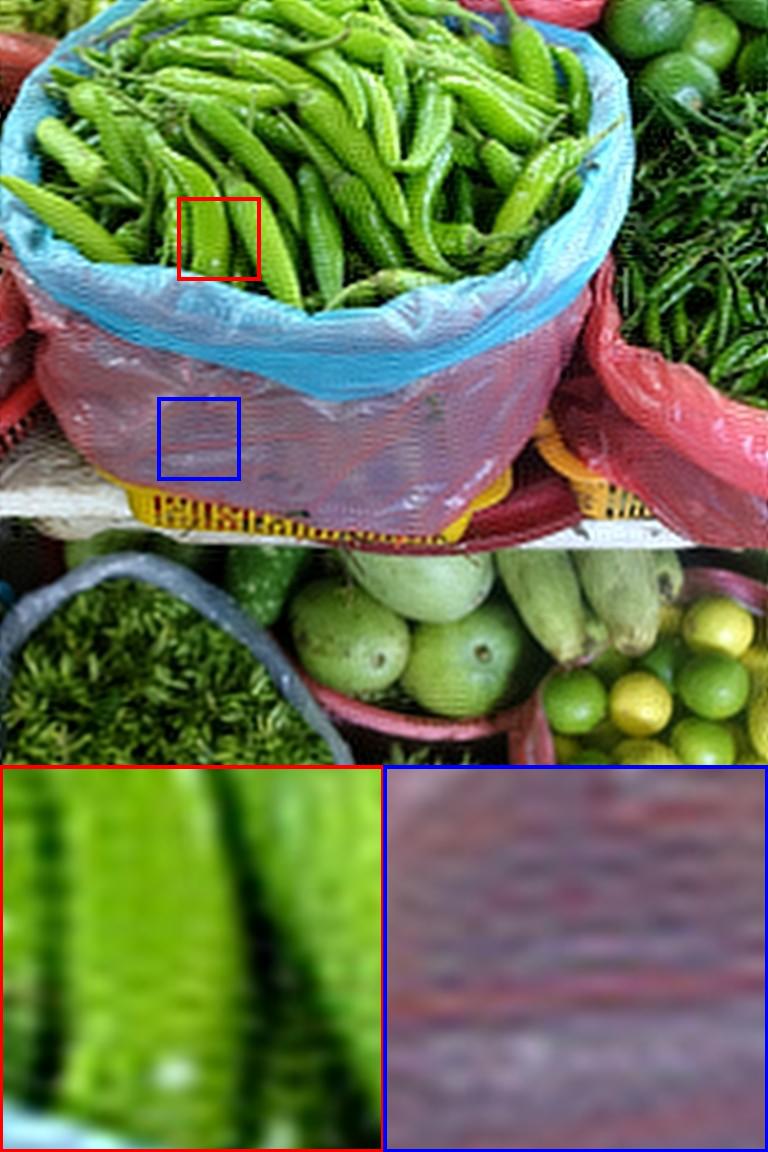} &
					\includegraphics[width=18.5mm]{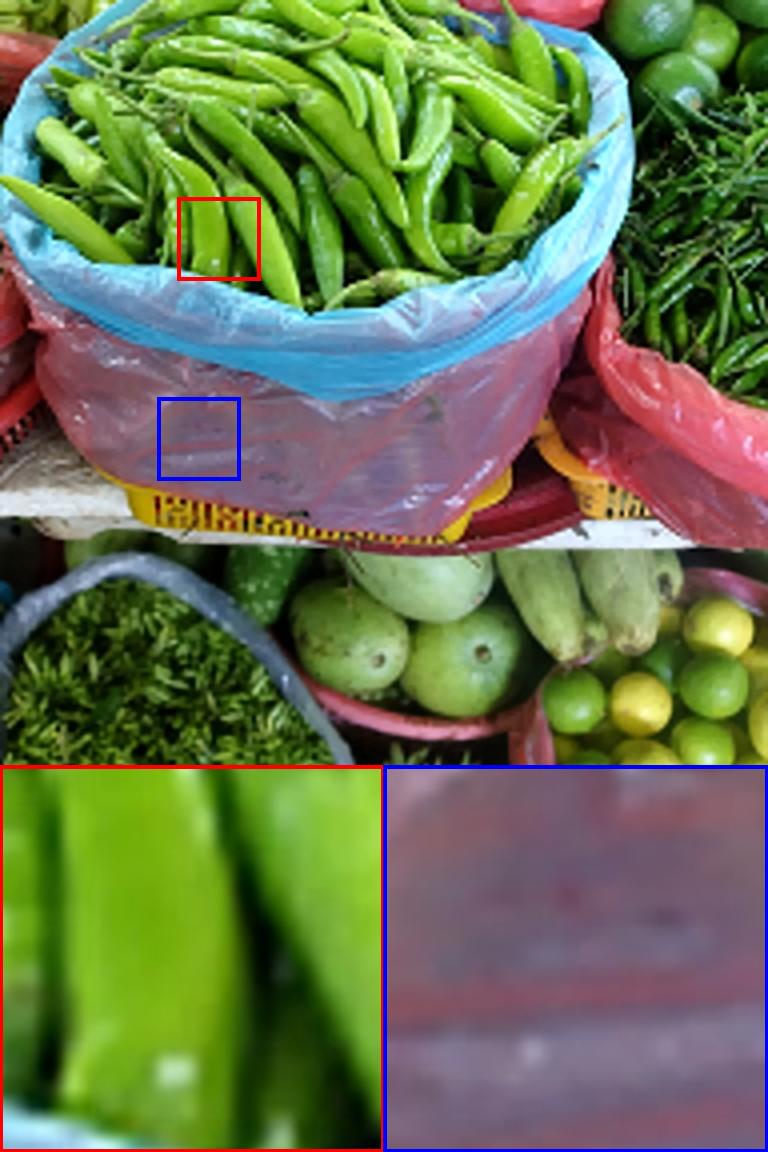} &
					\includegraphics[width=18.5mm]{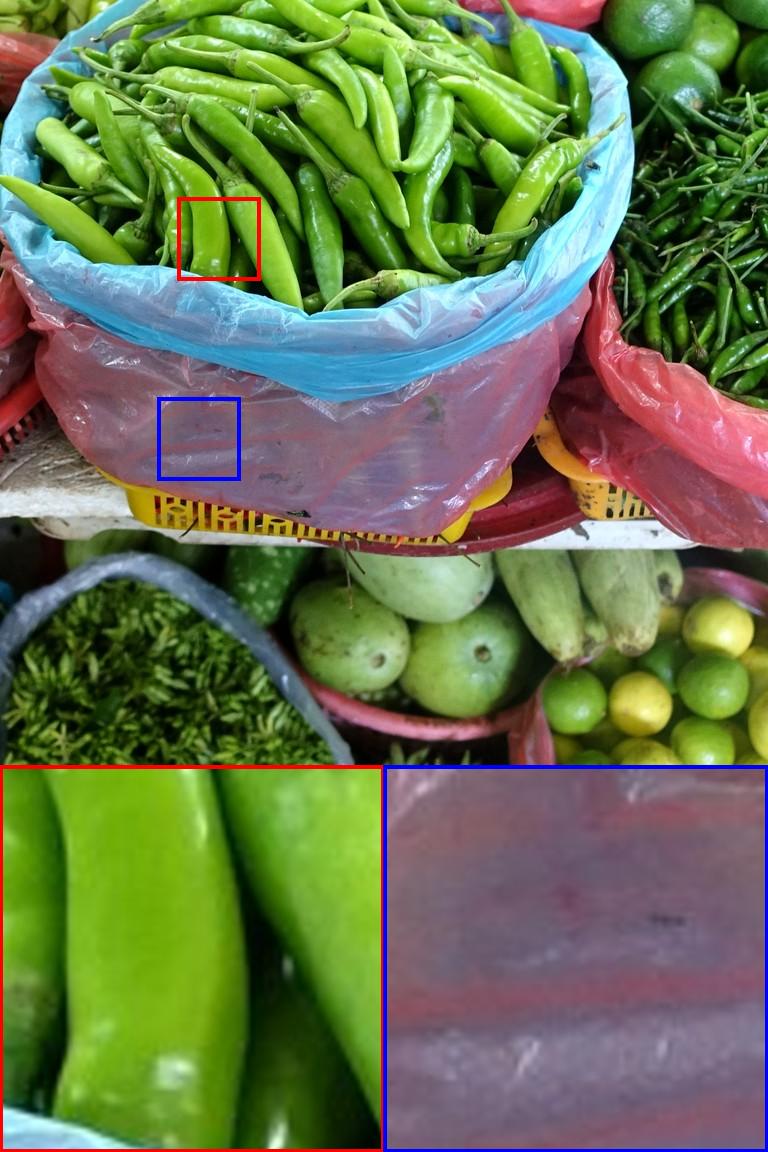} \\
					\scriptsize{24.78/0.575} & \scriptsize{25.79/0.609} & \scriptsize{25.40/0.591} & \scriptsize{26.59/0.643} & \scriptsize{26.35/0.616} &  \scriptsize{25.54/0.600} & \scriptsize{\textbf{26.87}/\textbf{0.673}} & \scriptsize{PSNR/SSIM} \\
					\includegraphics[width=18.5mm]{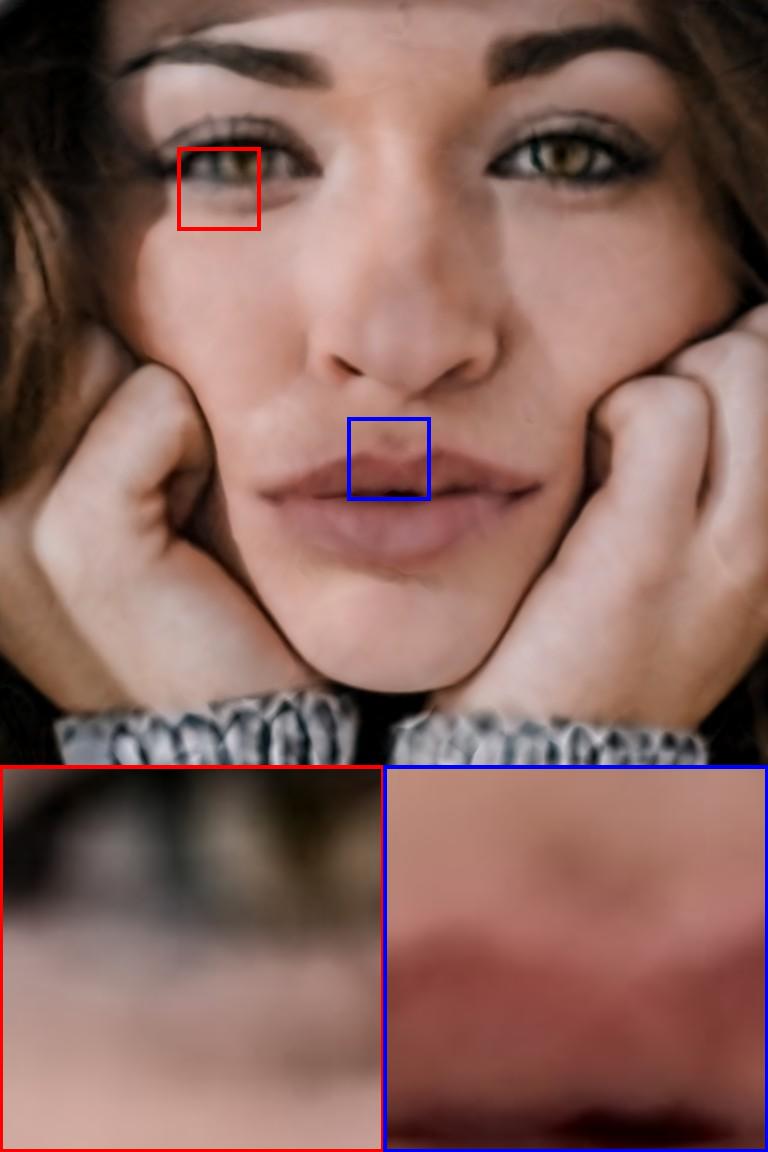} &
					\includegraphics[width=18.5mm]{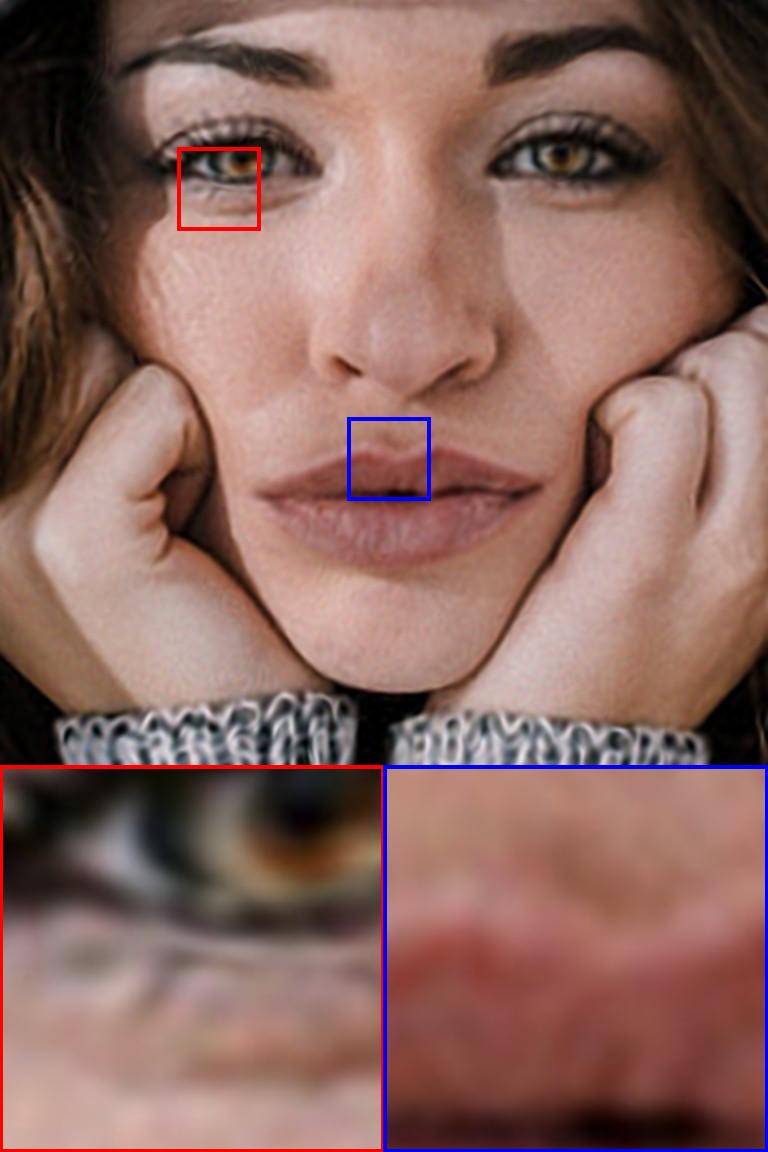} &
					\includegraphics[width=18.5mm]{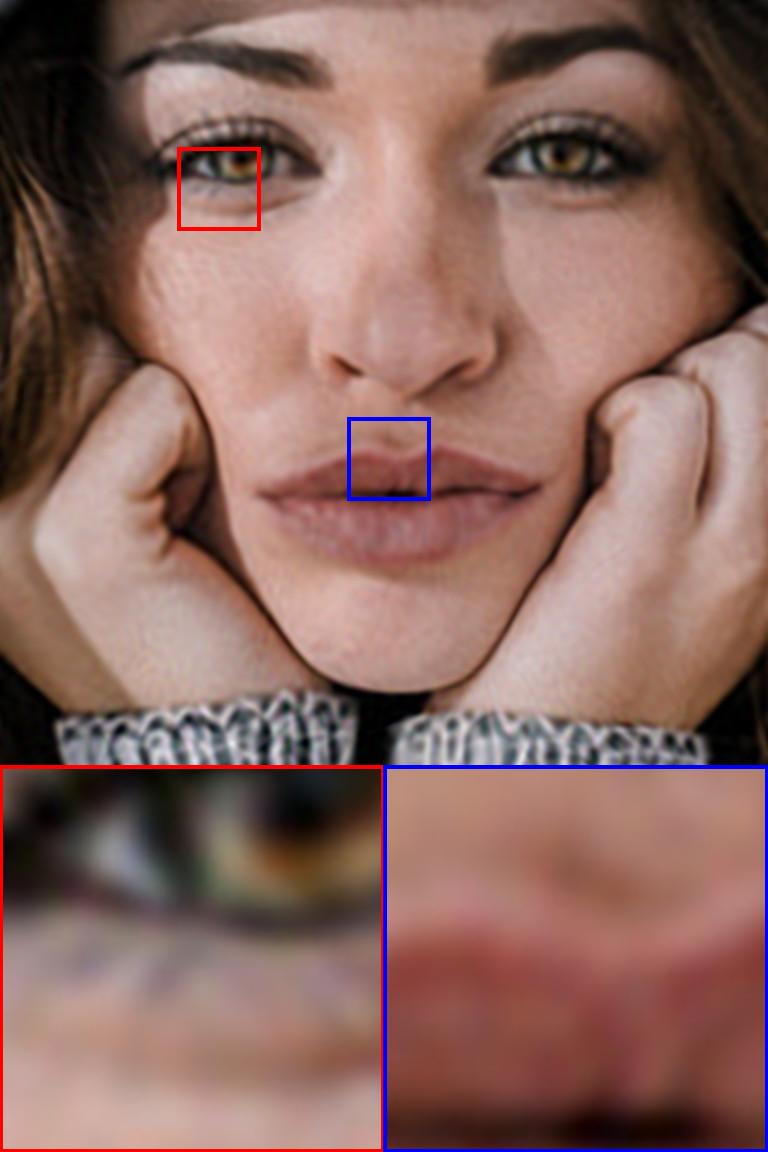} &
					\includegraphics[width=18.5mm]{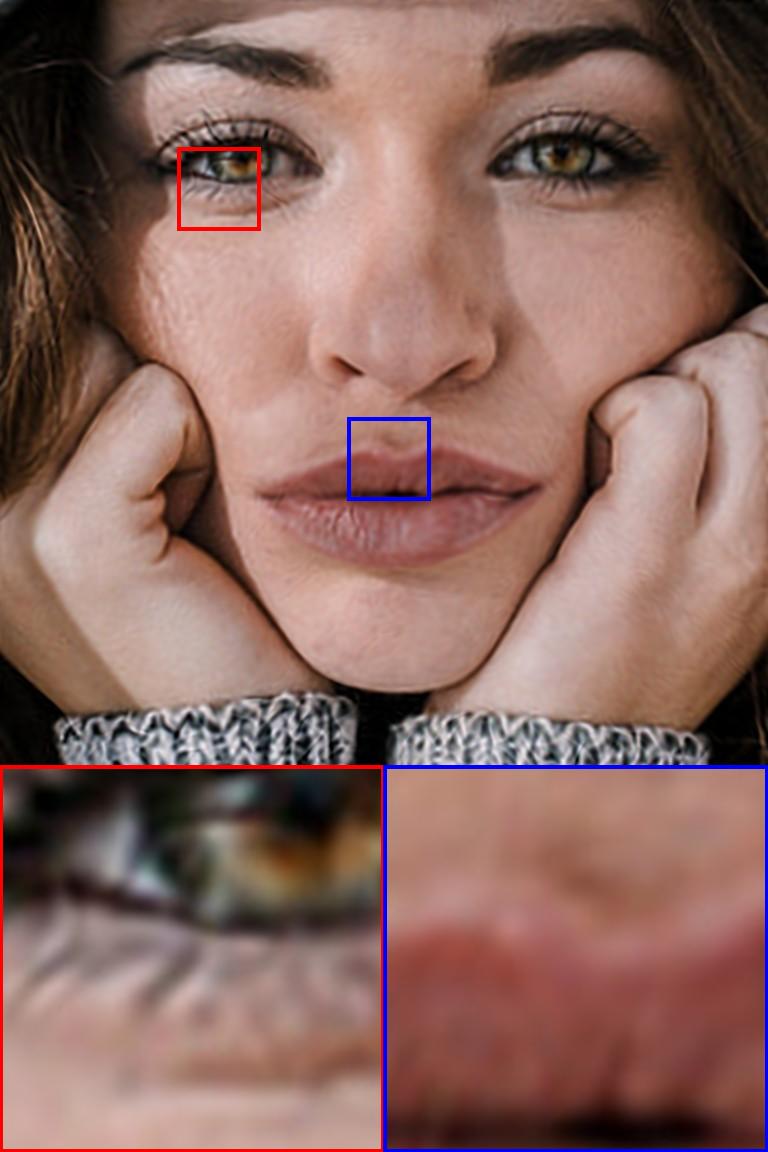} &
					\includegraphics[width=18.5mm]{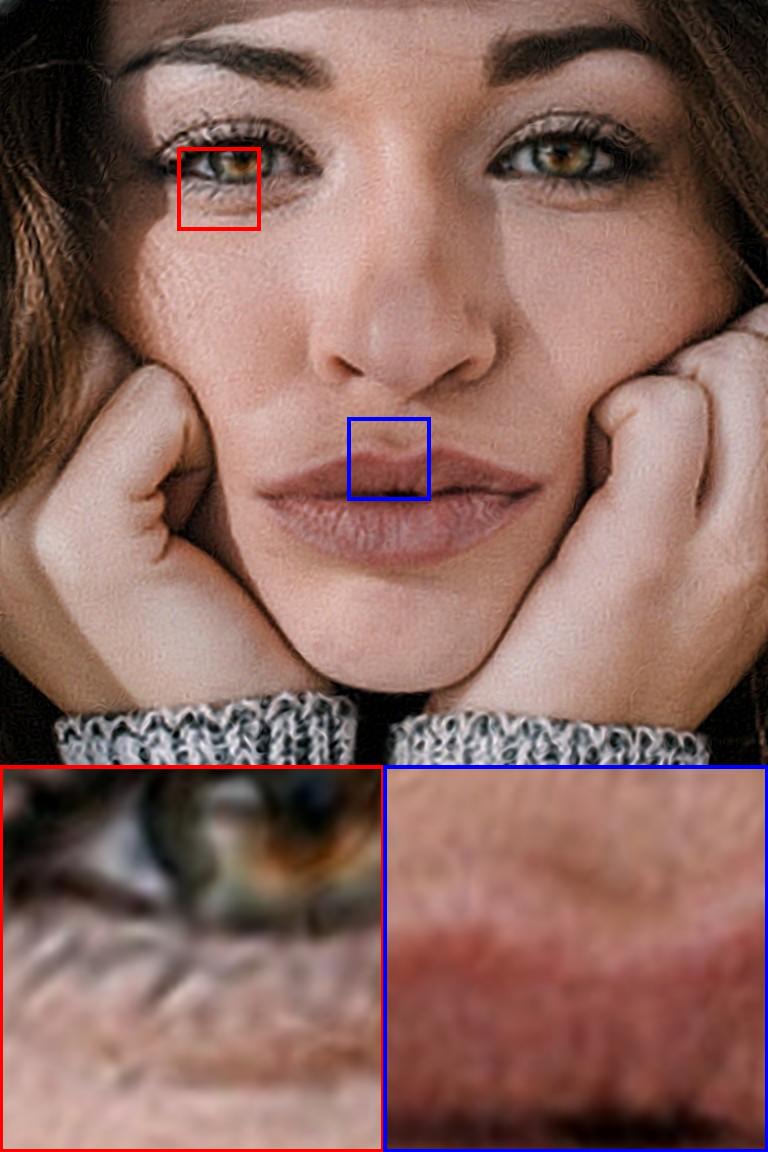} &
					\includegraphics[width=18.5mm]{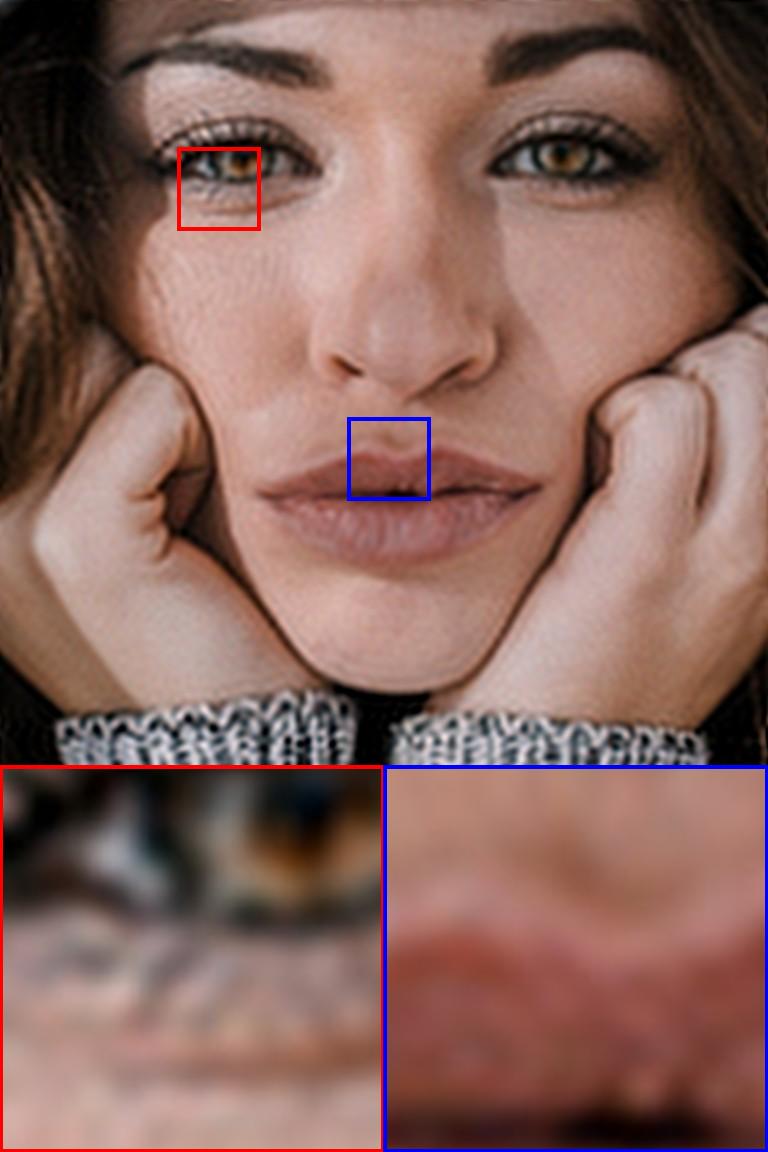} &
					\includegraphics[width=18.5mm]{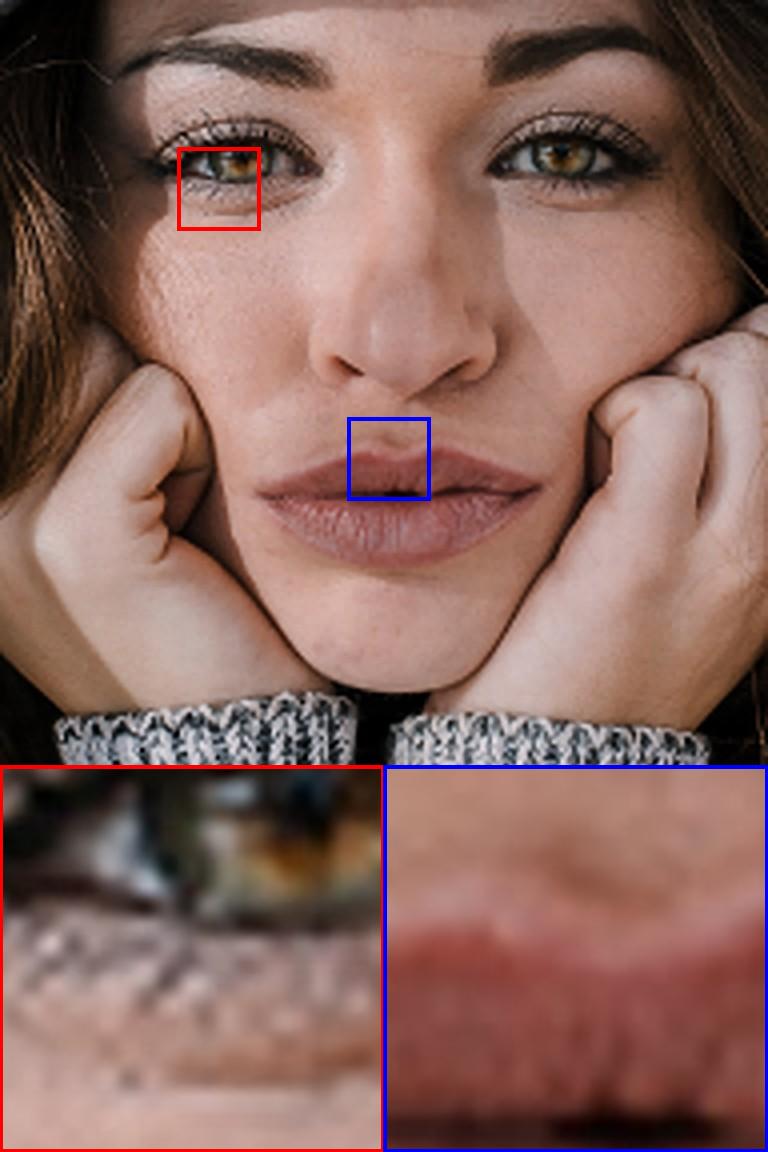} &
					\includegraphics[width=18.5mm]{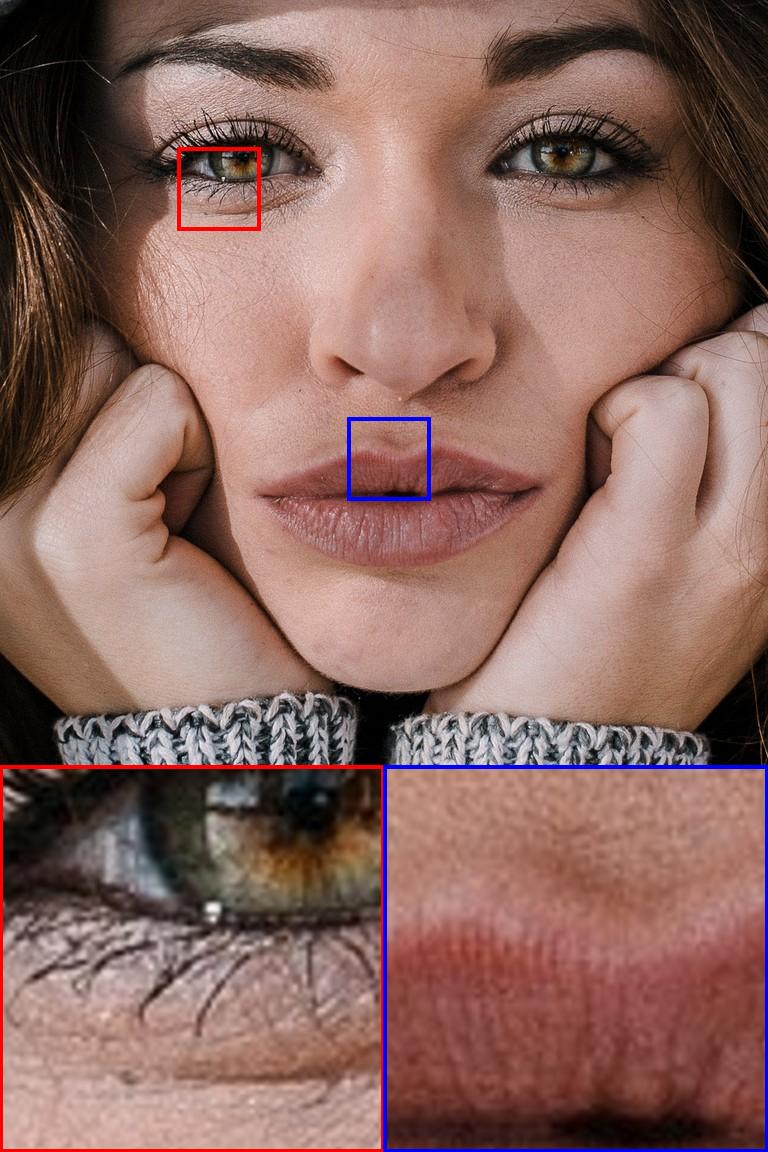} \\
					\scriptsize{23.19/0.633} & \scriptsize{24.85/0.740} & \scriptsize{23.88/0.726} & \scriptsize{25.75/0.790} & \scriptsize{25.29/0.749} &  \scriptsize{24.54/0.743} & \scriptsize{\textbf{26.15}/\textbf{0.843}} & \scriptsize{PSNR/SSIM} \\
					\includegraphics[width=18.5mm]{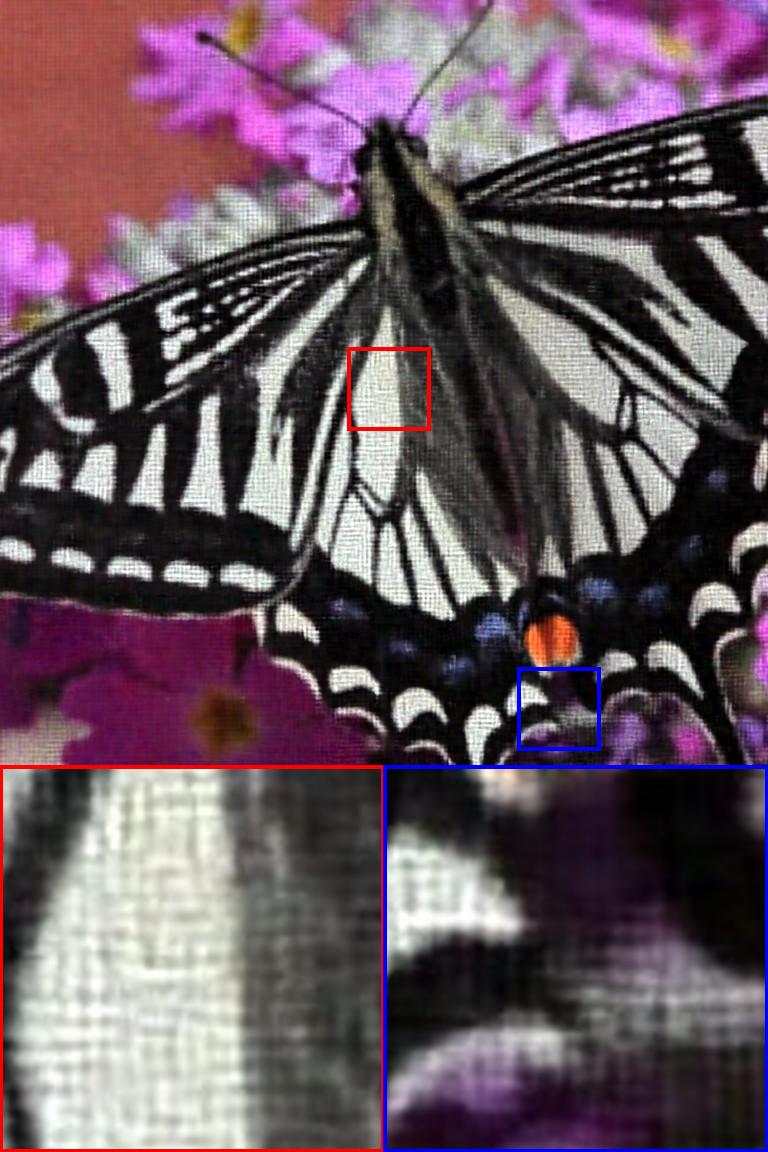} &
					\includegraphics[width=18.5mm]{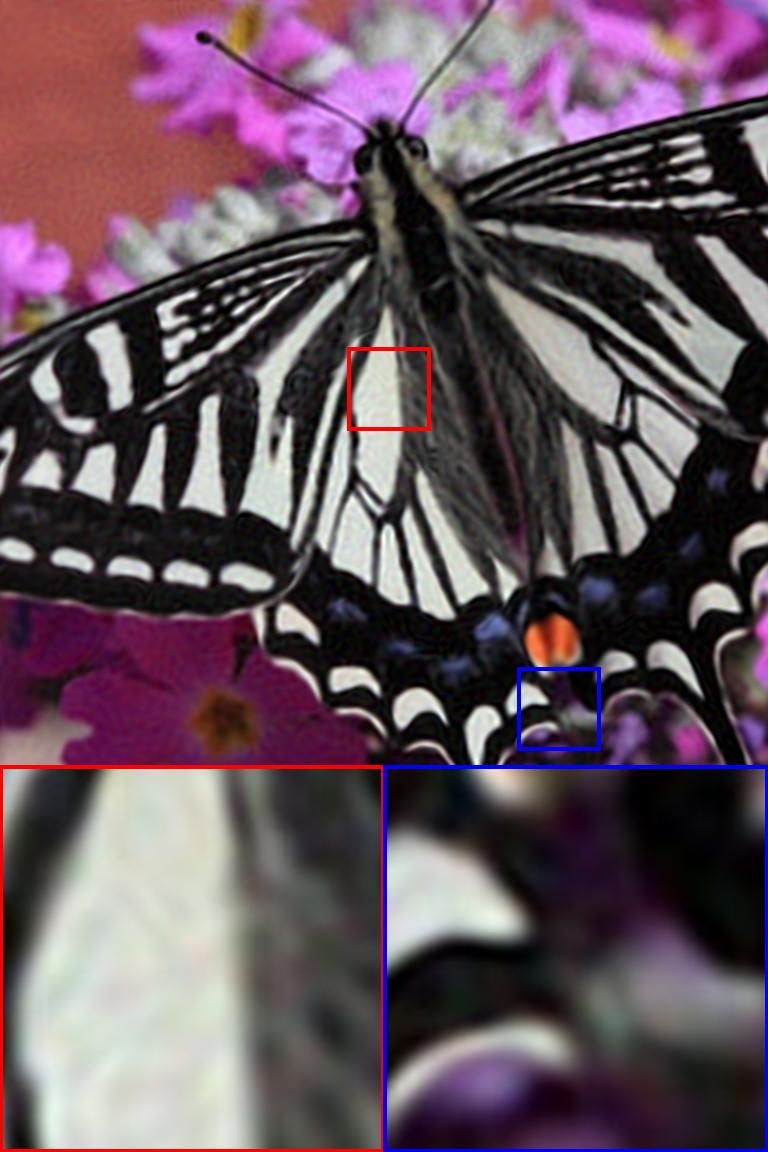} &
					\includegraphics[width=18.5mm]{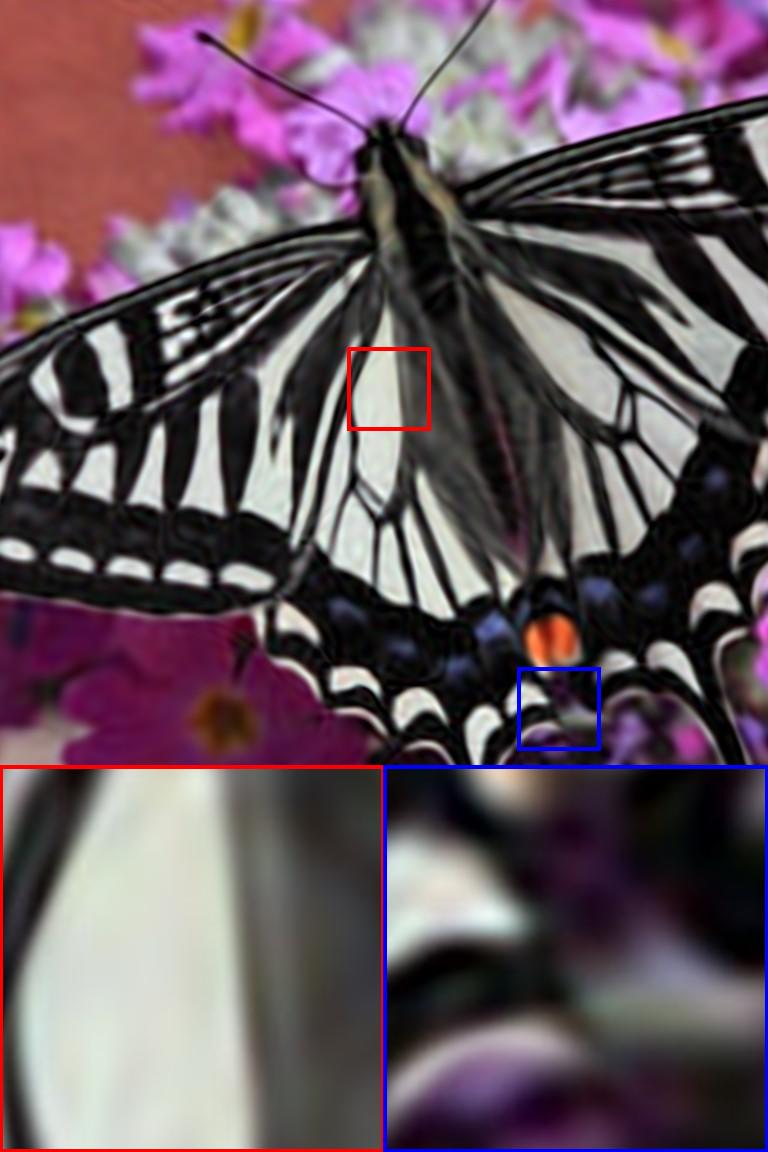} &
					\includegraphics[width=18.5mm]{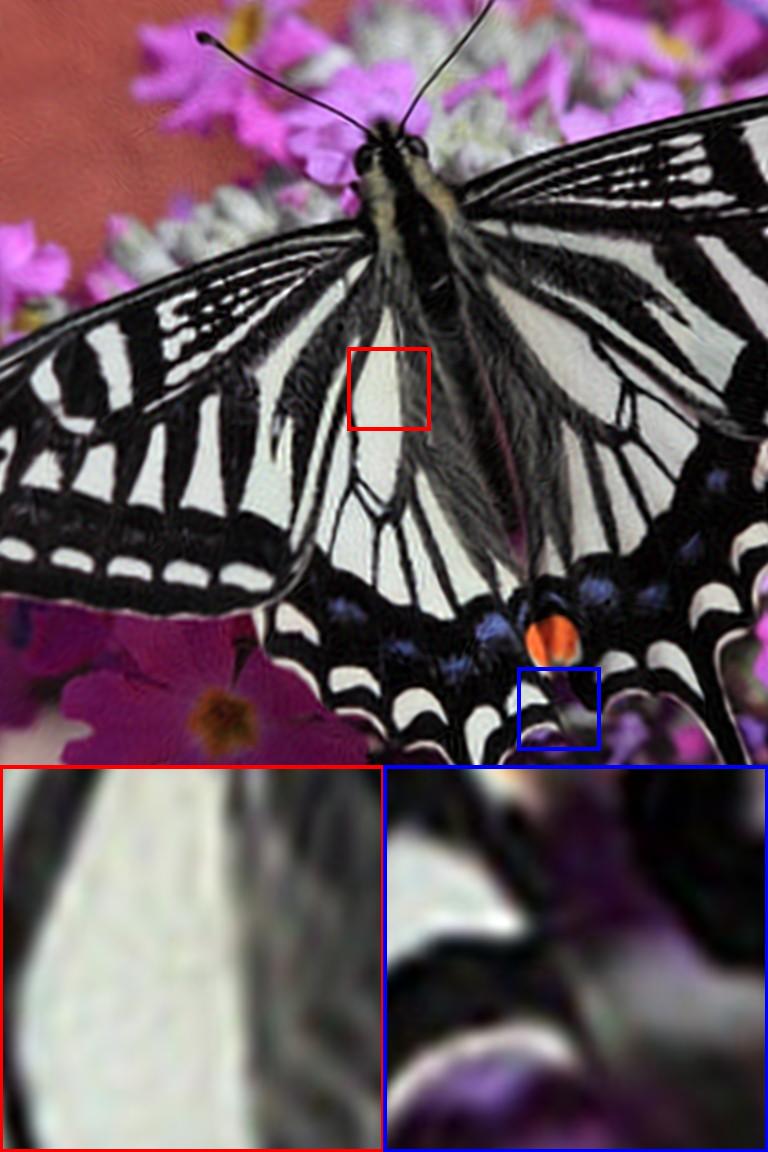} &
					\includegraphics[width=18.5mm]{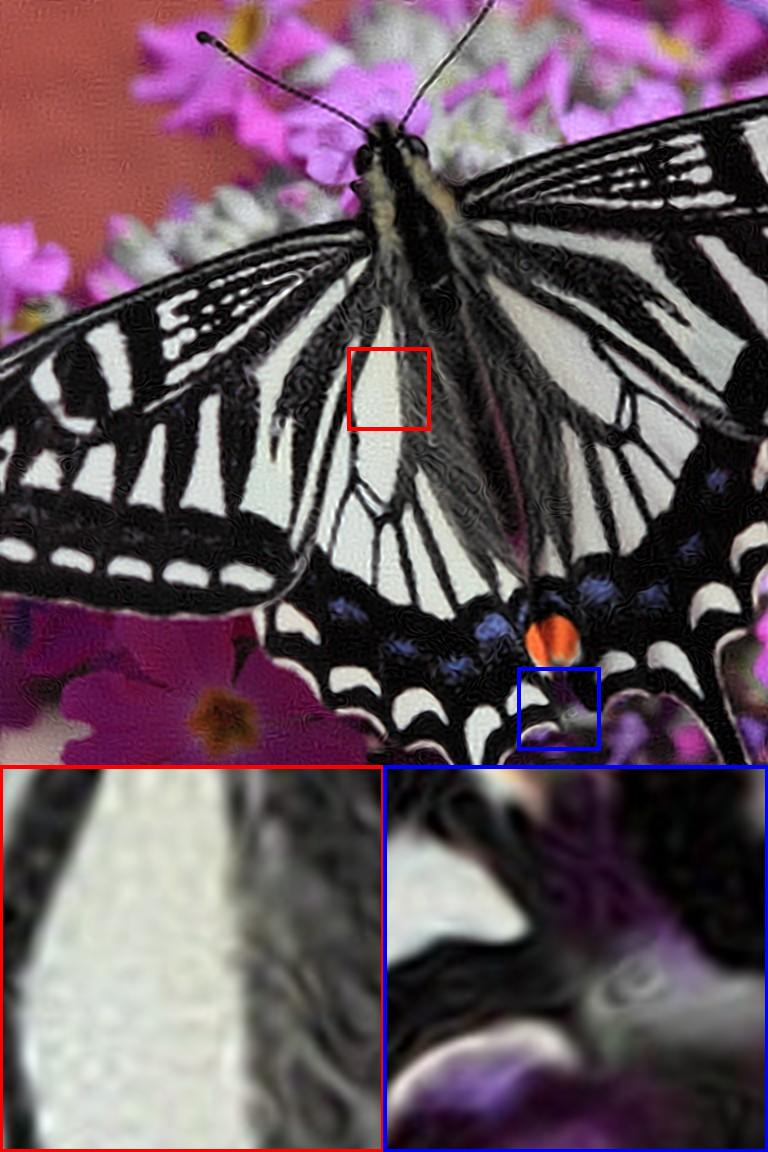} &
					\includegraphics[width=18.5mm]{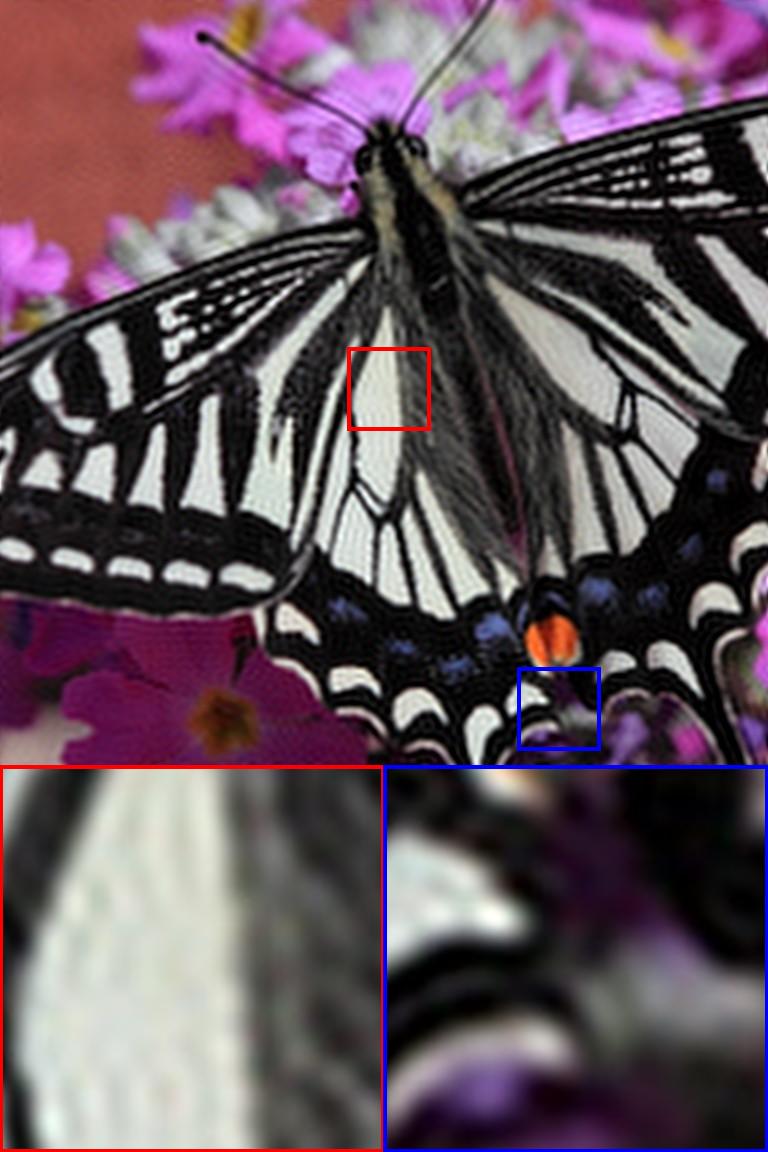} &
					\includegraphics[width=18.5mm]{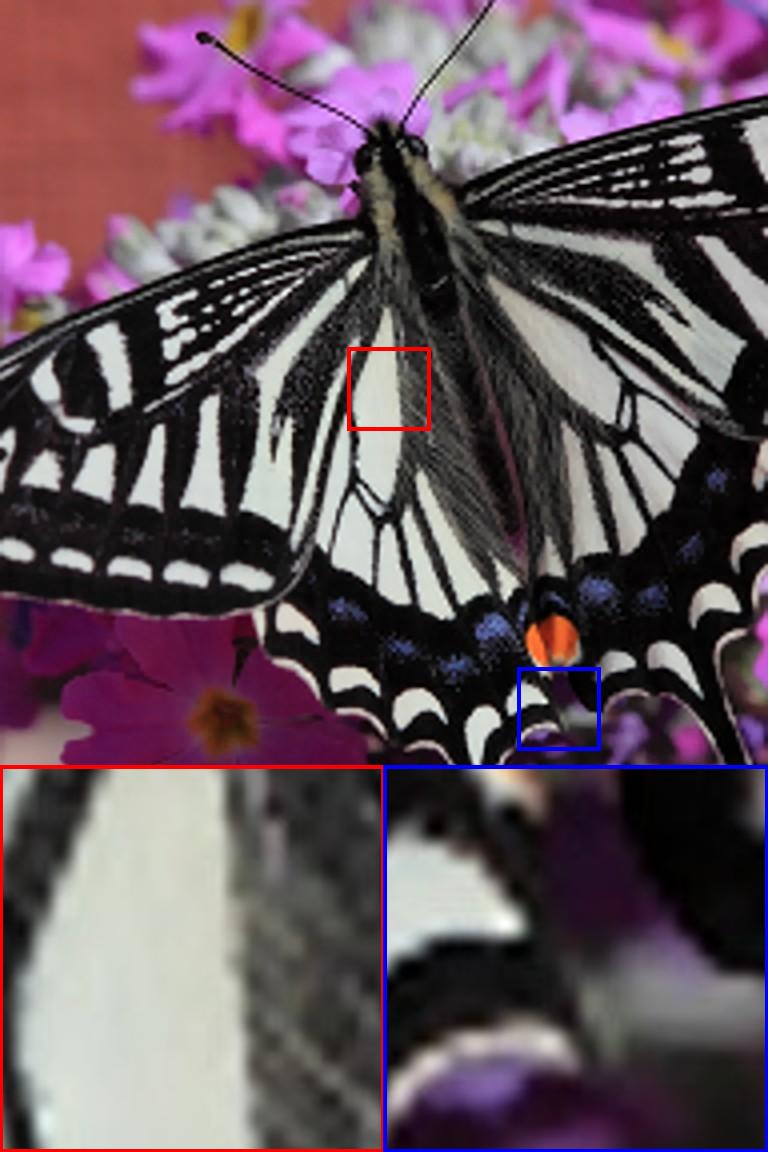} &
					\includegraphics[width=18.5mm]{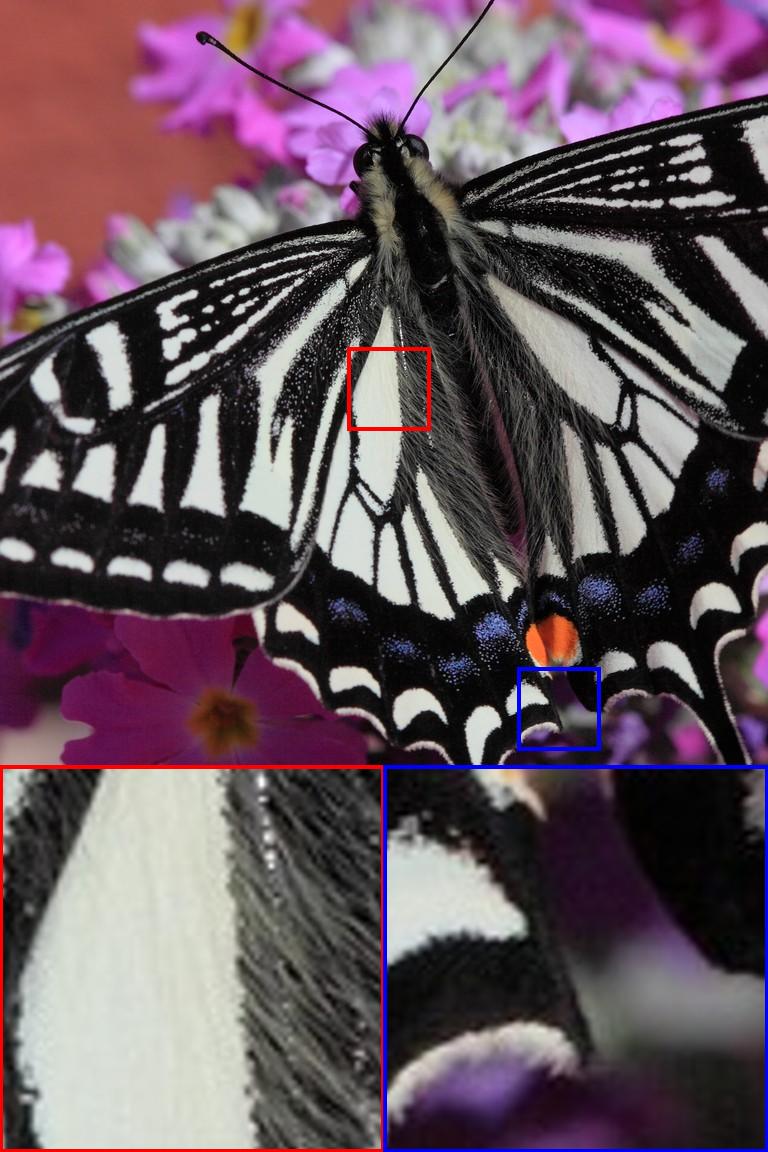} \\
					\scriptsize PEMLP & \scriptsize SIREN & \scriptsize Gauss & \scriptsize WIRE & \scriptsize FINER & \scriptsize LRTFR & \scriptsize Ours & \scriptsize{Ground truth}
			\end{tabular}}
			\caption{Visual super-resolution results at $\times4$ scaling on the DIV2K dataset.}
			\label{fig:super_resolution_supp}
		\end{figure*}
		
		\noindent{\bf Extended Point Cloud Recovery Results. }Fig.~\ref{fig:point_cloud_supp} presents extended visual results for point cloud recovery under different SRs. We compare the proposed RepTRFD with several INR (PEMLP~\cite{mildenhall2021nerf}, SIREN~\cite{sitzmann2020implicit}, Gauss~\cite{ramasinghe2022beyond}, WIRE~\cite{saragadam2023wire}, and FINER~\cite{liu2024finer}) and tensor functional representation methods (LRTFR~\cite{luo2023low}). Across all tested SRs, our method consistently reconstructs more accurate surface geometry and fine structures, as reflected by the lower NRMSE values shown above each reconstruction. These results demonstrate the superior performance of RepTRFD in point cloud recovery.
		
		\begin{figure*}[tb]
			\renewcommand{\arraystretch}{1}
			\setlength\tabcolsep{0.5pt}
			\centering
			\resizebox{\linewidth}{!}{
				\begin{tabular}{ccccccccc}
					\scriptsize{NRMSE} & \scriptsize{NRMSE 0.079} & \scriptsize{NRMSE 0.077} & \scriptsize{NRMSE 0.080} & \scriptsize{NRMSE 0.076} & \scriptsize{NRMSE 0.076} &  \scriptsize{NRMSE 0.074} & \scriptsize{NRMSE \textbf{0.066}} & \scriptsize{NRMSE \textbf{0.000}} \\
					\includegraphics[width=18.5mm]{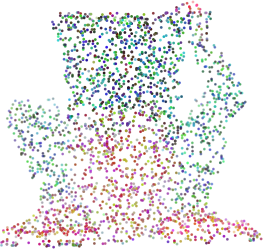} & \includegraphics[width=18.5mm]{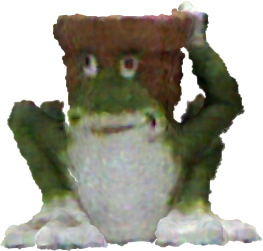} &
					\includegraphics[width=18.5mm]{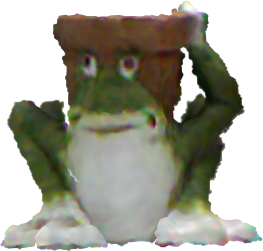} &
					\includegraphics[width=18.5mm]{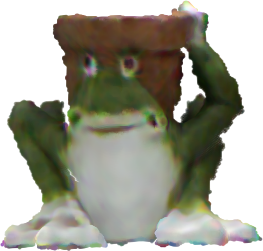} &
					\includegraphics[width=18.5mm]{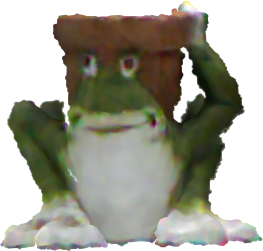} &
					\includegraphics[width=18.5mm]{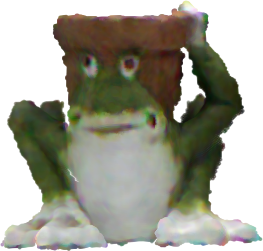} &
					\includegraphics[width=18.5mm]{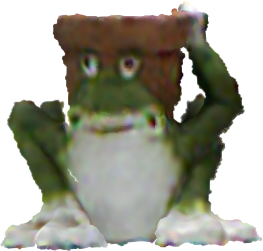} &
					\includegraphics[width=18.5mm]{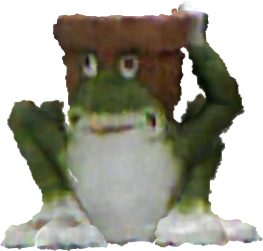} &
					\includegraphics[width=18.5mm]{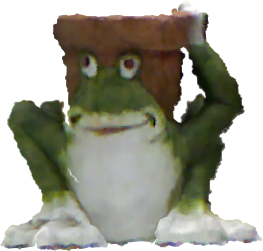} \\
					\scriptsize{NRMSE} & \scriptsize{NRMSE 0.086} & \scriptsize{NRMSE 0.085} & \scriptsize{NRMSE 0.086} & \scriptsize{NRMSE 0.084} & \scriptsize{NRMSE 0.085} &  \scriptsize{NRMSE 0.088} & \scriptsize{NRMSE \textbf{0.072}} & \scriptsize{NRMSE \textbf{0.000}} \\
					\includegraphics[width=18.5mm]{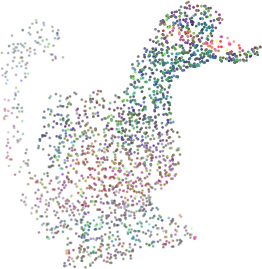} & \includegraphics[width=18.5mm]{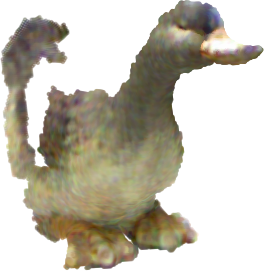} &
					\includegraphics[width=18.5mm]{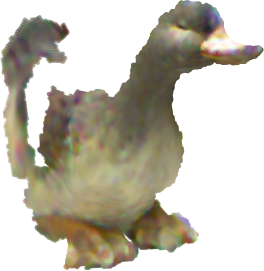} &
					\includegraphics[width=18.5mm]{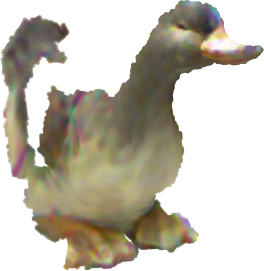} &
					\includegraphics[width=18.5mm]{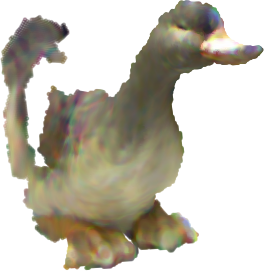} &
					\includegraphics[width=18.5mm]{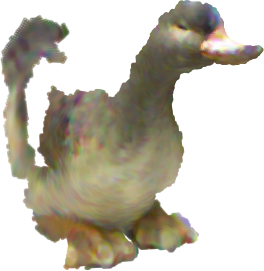} &
					\includegraphics[width=18.5mm]{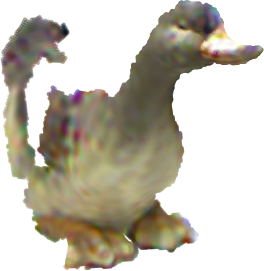} &
					\includegraphics[width=18.5mm]{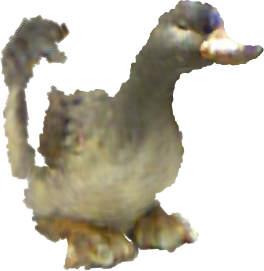} &
					\includegraphics[width=18.5mm]{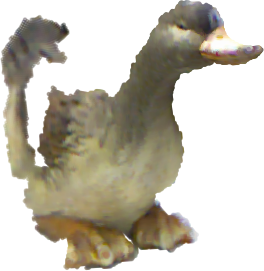} \\
					\scriptsize{NRMSE} & \scriptsize{NRMSE 0.094} & \scriptsize{NRMSE 0.085} & \scriptsize{NRMSE 0.083} & \scriptsize{NRMSE 0.080} & \scriptsize{NRMSE 0.082} &  \scriptsize{NRMSE 0.093} & \scriptsize{NRMSE \textbf{0.079}} & \scriptsize{NRMSE \textbf{0.000}} \\
					\includegraphics[width=18.5mm]{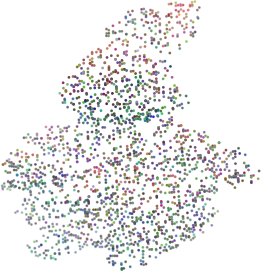} & \includegraphics[width=18.5mm]{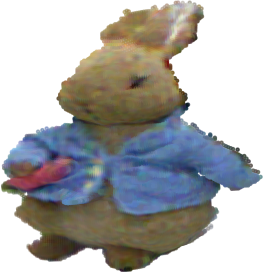} &
					\includegraphics[width=18.5mm]{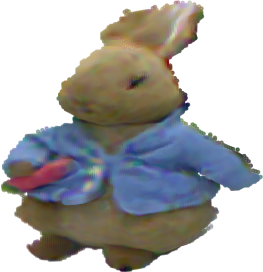} &
					\includegraphics[width=18.5mm]{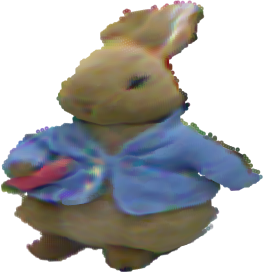} &
					\includegraphics[width=18.5mm]{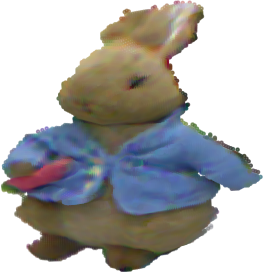} &
					\includegraphics[width=18.5mm]{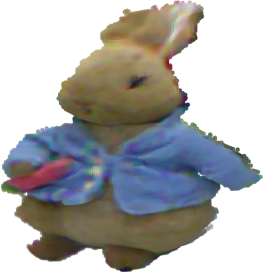} &
					\includegraphics[width=18.5mm]{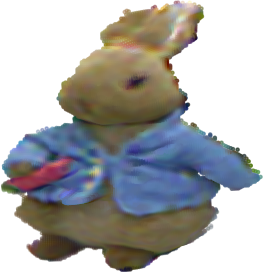} &
					\includegraphics[width=18.5mm]{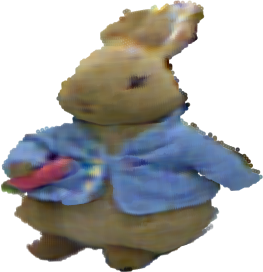} &
					\includegraphics[width=18.5mm]{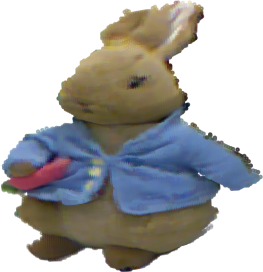} \\
					\scriptsize{NRMSE} & \scriptsize{NRMSE 0.114} & \scriptsize{NRMSE 0.104} & \scriptsize{NRMSE 0.109} & \scriptsize{NRMSE 0.100} & \scriptsize{NRMSE 0.102} &  \scriptsize{NRMSE 0.112} & \scriptsize{NRMSE \textbf{0.092}} & \scriptsize{NRMSE \textbf{0.000}} \\
					\includegraphics[width=18.5mm]{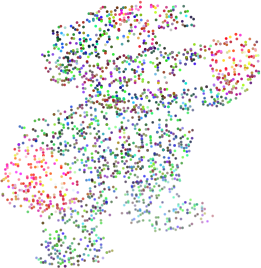} & \includegraphics[width=18.5mm]{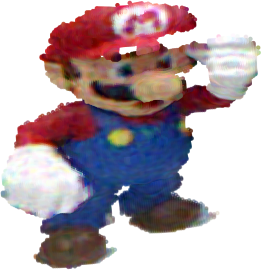} &
					\includegraphics[width=18.5mm]{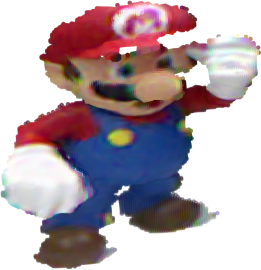} &
					\includegraphics[width=18.5mm]{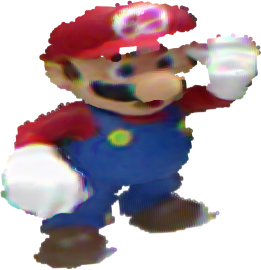} &
					\includegraphics[width=18.5mm]{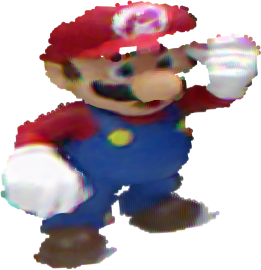} &
					\includegraphics[width=18.5mm]{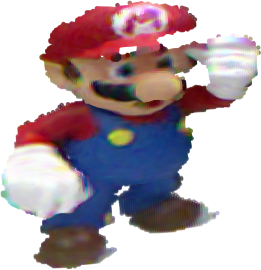} &
					\includegraphics[width=18.5mm]{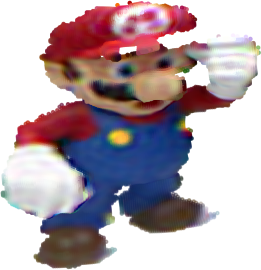} &
					\includegraphics[width=18.5mm]{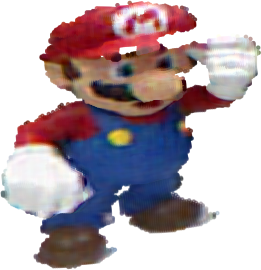} &
					\includegraphics[width=18.5mm]{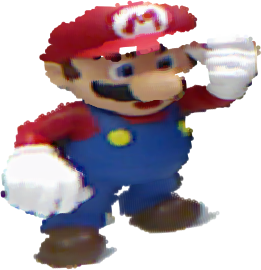} \\
					\scriptsize{NRMSE} & \scriptsize{NRMSE 0.149} & \scriptsize{NRMSE 0.142} & \scriptsize{NRMSE 0.149} & \scriptsize{NRMSE 0.127} & \scriptsize{NRMSE 0.147} &  \scriptsize{NRMSE 0.149} & \scriptsize{NRMSE \textbf{0.112}} & \scriptsize{NRMSE \textbf{0.000}} \\
					\includegraphics[width=18.5mm]{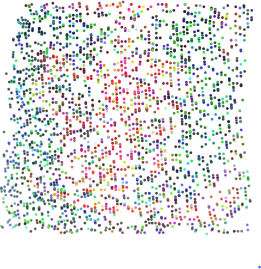} & \includegraphics[width=18.5mm]{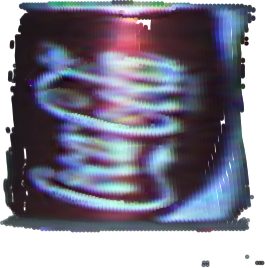} &
					\includegraphics[width=18.5mm]{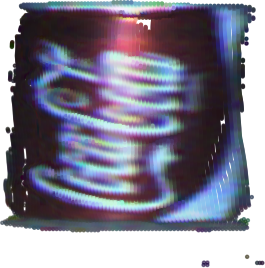} &
					\includegraphics[width=18.5mm]{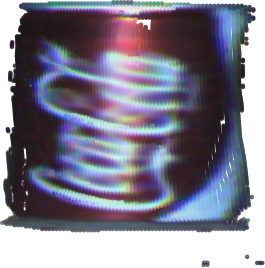} &
					\includegraphics[width=18.5mm]{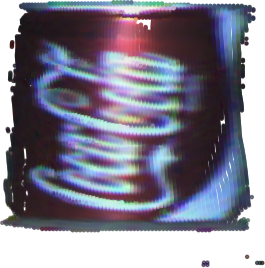} &
					\includegraphics[width=18.5mm]{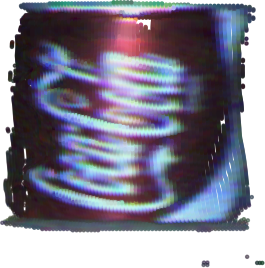} &
					\includegraphics[width=18.5mm]{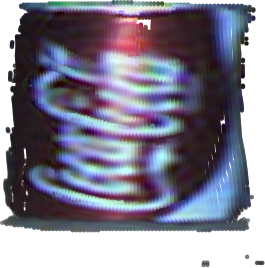} &
					\includegraphics[width=18.5mm]{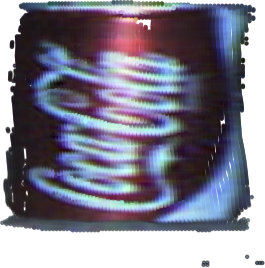} &
					\includegraphics[width=18.5mm]{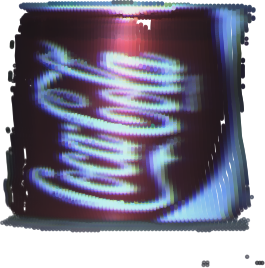} \\
					\scriptsize{NRMSE} & \scriptsize{NRMSE 0.112} & \scriptsize{NRMSE 0.111} & \scriptsize{NRMSE 0.123} & \scriptsize{NRMSE 0.106} & \scriptsize{NRMSE 0.110} &  \scriptsize{NRMSE 0.114} & \scriptsize{NRMSE \textbf{0.093}} & \scriptsize{NRMSE \textbf{0.000}} \\
					\includegraphics[width=18.5mm]{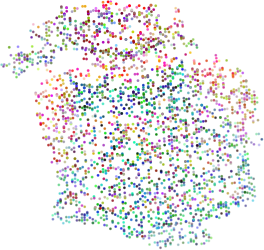} & \includegraphics[width=18.5mm]{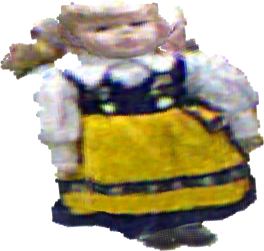} &
					\includegraphics[width=18.5mm]{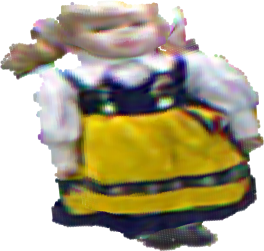} &
					\includegraphics[width=18.5mm]{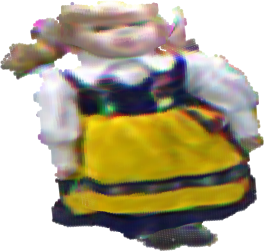} &
					\includegraphics[width=18.5mm]{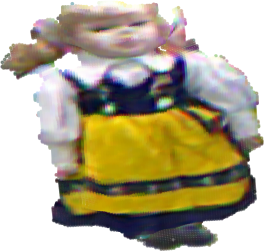} &
					\includegraphics[width=18.5mm]{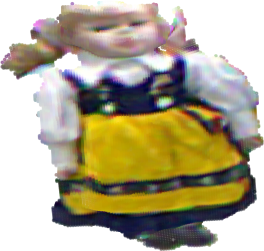} &
					\includegraphics[width=18.5mm]{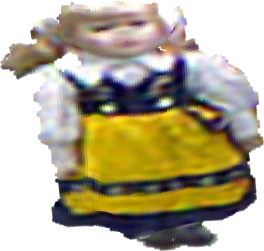} &
					\includegraphics[width=18.5mm]{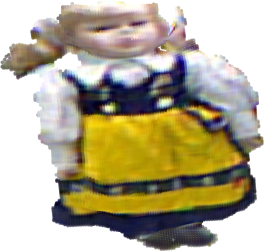} &
					\includegraphics[width=18.5mm]{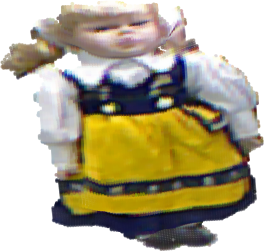} \\
					\scriptsize{Observed} & \scriptsize PEMLP & \scriptsize SIREN & \scriptsize Gauss & \scriptsize WIRE & \scriptsize FINER & \scriptsize LRTFR & \scriptsize Ours & \scriptsize{Ground truth}
			\end{tabular}}
			\caption{Visual comparisons on point cloud recovery over six objects at three SRs: $\text{SR}=0.1$ (\textit{Frog}, \textit{Duck}), $\text{SR}=0.15$ (\textit{PetterRabbit}, \textit{Mario}), and $\text{SR}=0.2$ (\textit{Doll}, \textit{Cola}).}
			\label{fig:point_cloud_supp}
		\end{figure*}


	\end{document}